\pdfoutput=1

\documentclass{article}

\PassOptionsToPackage{numbers, square}{natbib}

\usepackage[final]{neurips_2022}
\usepackage{graphicx}
\usepackage{amsmath}
\usepackage{bbold}
\usepackage{subfig}
\usepackage{color}
\usepackage{enumitem}
\usepackage{comment}

\usepackage[utf8]{inputenc} % allow utf-8 input
\usepackage[T1]{fontenc}    % use 8-bit T1 fonts
\usepackage{hyperref}       % hyperlinks
\usepackage{url}            % simple URL typesetting
\usepackage{booktabs}       % professional-quality tables
\usepackage{amsfonts}       % blackboard math symbols
\usepackage{nicefrac}       % compact symbols for 1/2, etc.
\usepackage{microtype}      % microtypography
\usepackage{xcolor}         % colors
\usepackage{amsmath,amssymb} % define this before the line numbering.
\usepackage{dsfont}         % for \IR, etc.
\usepackage{xspace}

\usepackage{amsthm}

\setcitestyle{numbers}

\newtheorem{theorem}{Theorem}
\newtheorem*{T1}{Theorem~\ref{thm:volumetric_approximation_of_coverage}}
\newtheorem{lemma}[theorem]{Lemma}

\title{\method: Surface Coverage Optimization in Unknown Environments by Volumetric Integration}

\author{%
  Antoine Guédon \qquad Pascal Monasse \qquad Vincent Lepetit %\thanks{Use footnote for providing further information about author (webpage, alternative address)---\emph{not} for acknowledging funding agencies.} 
  \\
  LIGM, Ecole des Ponts, Univ Gustave Eiffel, CNRS, France \\
  \texttt{\{antoine.guedon,pascal.monasse,vincent.lepetit\}@enpc.fr} \\
}

\usepackage{etoolbox}

\newif\ifshowedits

\newcommand{\addeditor}[3]{%
  \definecolor{#1color}{rgb}{#3}
  \expandafter\newcommand\csname #1\endcsname[1]{%
  \ifshowedits
    {\color{#1color} ##1}%
  \else
    {##1}%
  \fi
  }%
  \expandafter\newcommand\csname #1rmk\endcsname[1]{%
  \ifshowedits
    {\color{#1color} {\bf [#2: ##1]}}
  \fi
  }%
  \expandafter\newcommand\csname #1rpl\endcsname[2]{%
  \ifshowedits
    {\color{#1color} ##1 \sout{##2}}
  \else
    {##1}
  \fi
  }%
}

\newcommand{\mycomment}[1]{}

\newcommand{\createtextvar}[1]{
  \expandafter\newcommand\csname #1\endcsname{%
  {\text{#1}}
}%
}
\newcommand{\textvars}[1]{\forcsvlist{\createtextvar}{#1}}

\newcommand{\calC}{{\cal C}}

\newcommand{\calL}{{\cal L}}

\newcommand{\calP}{{\cal P}}

\newcommand{\IR}{{\mathds{R}}}

\newcommand{\ie}{\emph{i.e.}}
\newcommand{\eg}{\emph{e.g.}}

\addeditor{vincent}{VL}{0.0, 0.5, 0.0}
\addeditor{antoine}{AG}{0.0, 0.0, 0.5}
\showeditsfalse

\newcommand{\method}{SCONE\xspace}
\newcommand{\occfield}{\sigma}
\newcommand{\probfield}{\hat{\sigma}} % Is the notation \sigma OK for the occupancy probability field?
\newcommand{\absvis}{v} % \absvis is the "absolute" visibility
\newcommand{\vis}{\nu} % and \vis is actually the visibility gain
\newcommand{\newvis}{g}  % {\overline{\nu}}

\newcommand{\camhist}{H}
\newcommand{\proxy}{\hat{\chi}}

\newcommand{\softmax}{\text{softmax}}

\textvars{pos,rot}

\begin{document}

\maketitle

\begin{abstract}
  % Next Best View computation (NBV) is a long-standing problem in robotics, and consists in identifying the next most informative sensor position(s) for reconstructing a 3D object or scene efficiently and accurately. Like most current methods, we consider NBV prediction from a depth sensor. Learning-based methods relying on a volumetric representation of the scene are suitable for path planning, but do not scale well with the size of the scene and have lower accuracy than methods using a surface-based representation. However, the latter constrain the camera to a small number of poses. To obtain the advantages of both representations, we show that we can maximize surface metrics by Monte Carlo integration over a volumetric representation. Our method scales to large scenes and handles free camera motion: It takes as input an arbitrarily large point cloud gathered by a depth sensor like Lidar systems as well as camera poses to predict NBV. We demonstrate our approach on a novel dataset made of large and complex 3D scenes.
  Next Best View computation (NBV) is a long-standing problem in robotics, and consists in identifying the next most informative sensor position(s) for reconstructing a 3D object or scene efficiently and accurately. Like most current methods, we consider NBV prediction from a depth sensor like Lidar systems. Learning-based methods relying on a volumetric representation of the scene are suitable for path planning, but have lower accuracy than methods using a surface-based representation. However, the latter do not scale well with the size of the scene and constrain the camera to a small number of poses. To obtain the advantages of both representations, we show that we can maximize surface metrics by Monte Carlo integration over a volumetric representation. In particular, we propose an approach, \method, that relies on two neural modules: The first module predicts occupancy probability in the entire volume of the scene. Given any new camera pose, the second module samples points in the scene based on their occupancy probability and leverages a self-attention mechanism to predict the visibility of the samples. Finally, we integrate the visibility to evaluate the gain in surface coverage for the new camera pose. NBV is selected as the pose that maximizes the gain in total surface coverage. Our method scales to large scenes and handles free camera motion: It takes as input an arbitrarily large point cloud gathered by a depth sensor as well as camera poses to predict NBV. We demonstrate our approach on a novel dataset made of large and complex 3D scenes.
\end{abstract}

\section{Introduction}

Next Best View computation~(NBV) is a long-standing problem in robotics~\cite{connolly-father-paper, yamauchi-97}, which consists in identifying the next most informative sensor position(s) for reconstructing a 3D object or scene efficiently and accurately. Typically, a position is evaluated on how much it can increase the total coverage of the scene surface. 
Few methods have relied on Deep Learning (DL) for the NBV problem, even though DL can provide useful geometric prior to obtain a better prediction of the surface coverage~\cite{zeng-icirs20-pcnbv, Mendoza-2019, vasquez-21-3DCNN}. 
Like most current methods, we consider NBV prediction from a depth sensor. Existing methods based on a depth sensor rely either on a volumetric or on a surface-based representation of the scene geometry. Volumetric mapping-based methods can compute collision efficiently, which is practical for path planning in real case scenarios~\cite{potthast-14-probabilistic-framework, vasquez-14-volumetric-nbv, vasquez-17-view-state-planning, kriegel-12-next-best-scan, bissmarck-15-efficient-nbv, daudelin-17-adaptable-probabilistic}. However, they typically rely on voxels or a global embedding~\cite{hepp-corr18-learntoscore, charrow-15-information, song-17-online-inspection, song-20-online-coverage, wang-20-efficient-autonomous, cieslewski-17-ICIRS} for the scene, which results in poor accuracy in reconstruction and poor performance in NBV selection for complex 3D objects. On the contrary, surface mapping-based methods that process directly a dense point cloud of the surface as gathered by the depth sensor are efficient for NBV prediction with high-detailed geometry. They are however limited to very specific cases, generally a single, small-scale, isolated object with the camera constrained to stay on a sphere centered on the object~\cite{lee-tra20-mechanical-parts,vasquez-17-view-state-planning, delmerico-18-acomparison, chen-smc05-visionsensorplanning,kriegel-11-surface-based, kriegel-12-next-best-scan,kriegel-12-next-best-scan,zeng-icirs20-pcnbv, Mendoza-2019, vasquez-21-3DCNN}. Thus, they cannot be applied to the exploration of 3D scenes.

% They are however limited to very specific cases~\cite{lee-tra20-mechanical-parts}, generally a single, small-scale~\cite{vasquez-17-view-state-planning, delmerico-18-acomparison, chen-05-visionsensorplanning, kriegel-11-surface-based, kriegel-12-next-best-scan}, isolated object~\cite{kriegel-12-next-best-scan} floating in the air with the camera constrained to stay on a sphere centered on the object~\cite{zeng-icirs20-pcnbv, Mendoza-2019, vasquez-21-3DCNN}. Thus, they cannot be applied to the exploration of 3D scenes. 

\begin{figure}%
    \centering
    \begin{tabular}{cccc}
      \includegraphics[width=3.15cm, trim={0 0 3cm 0}, clip]{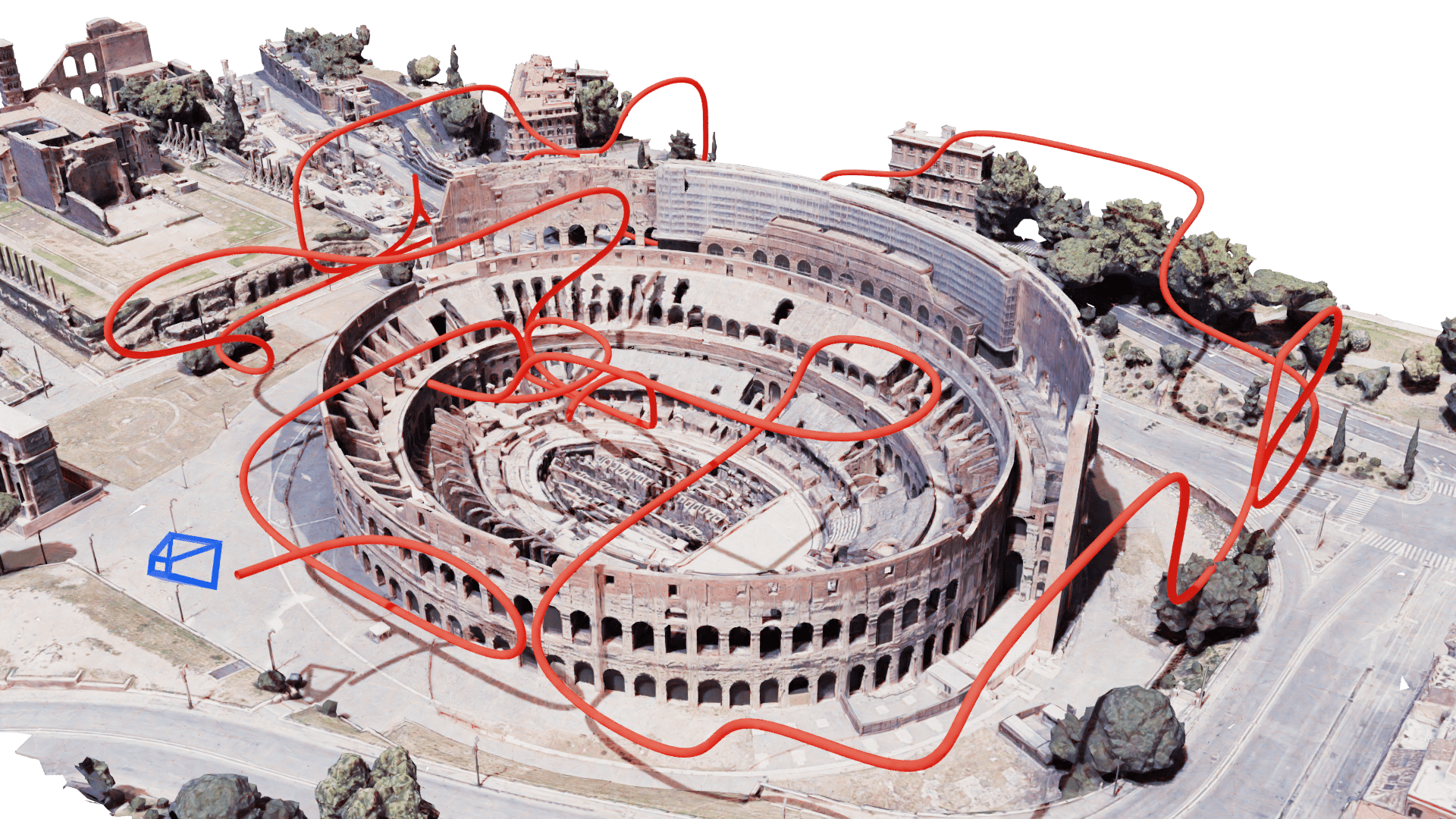} & % width=0.45\linewidth
      \includegraphics[width=3.15cm, trim={3cm 0 3cm 0}, clip]{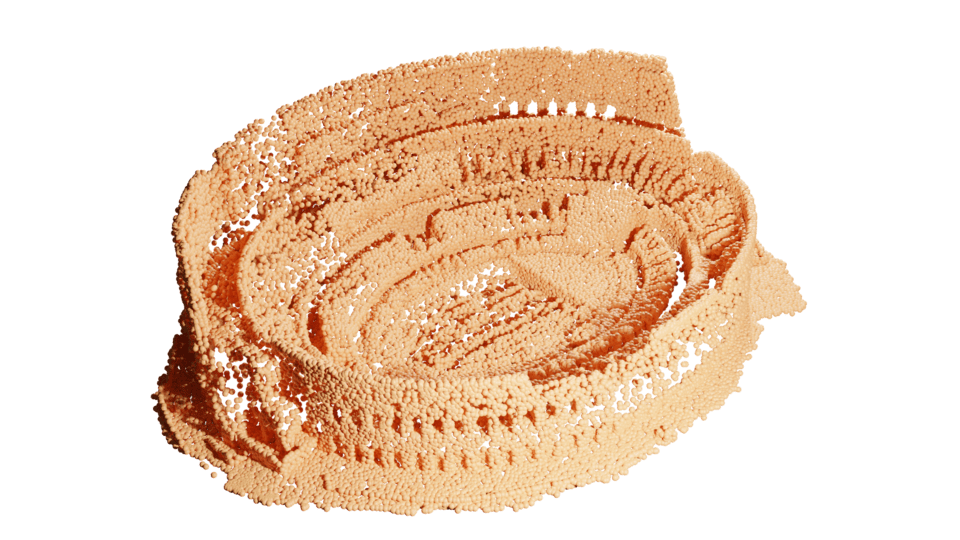} &
      \includegraphics[width=3.15cm, trim={0 0 0 0}, clip]{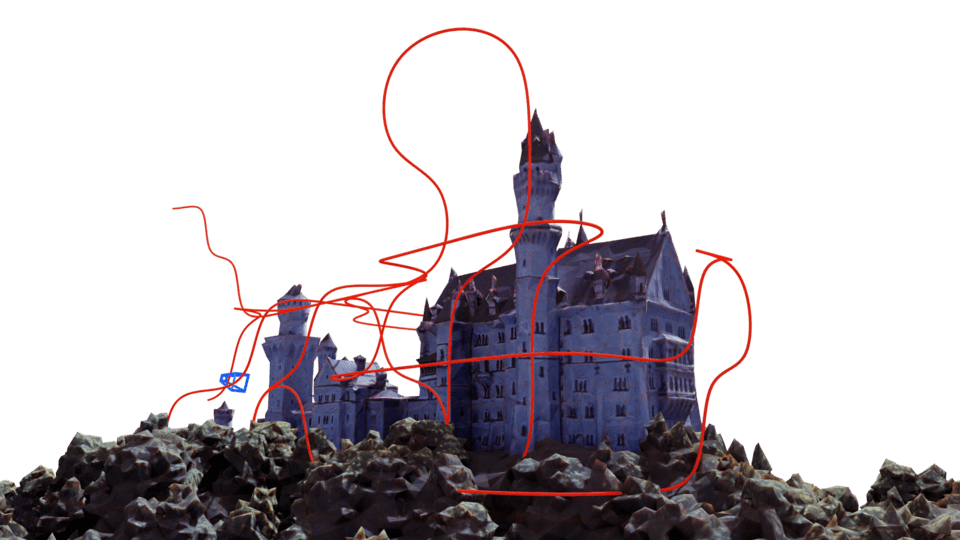} & % width=0.45\linewidth
      \includegraphics[width=2.8cm, trim={6cm 0 3cm 3cm}, clip]{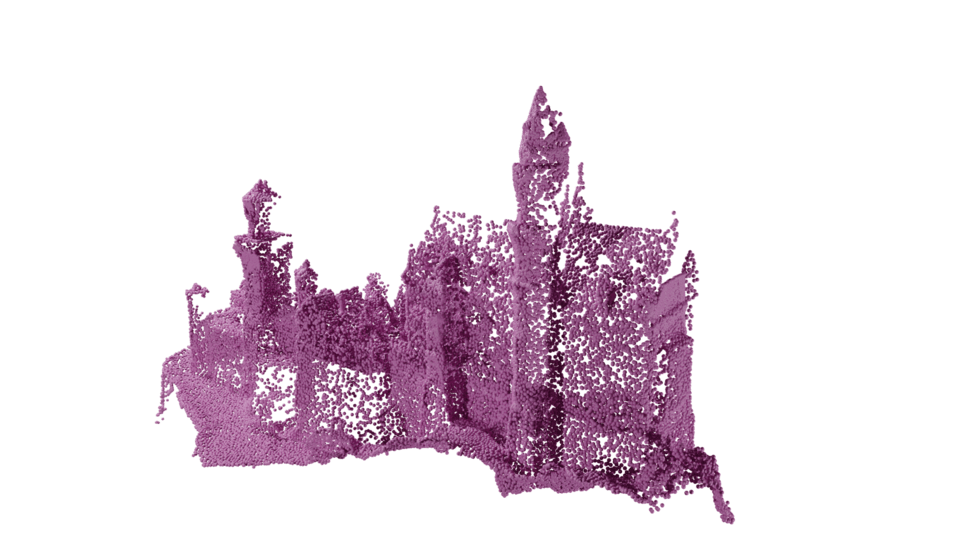}\\
    \end{tabular}
    \caption{{\bf Our Next Best View~(NBV) method \method  can handle unknown large-scale 3D scenes} to produce accurate 3D reconstruction inside a given 3D bounding box. Here, we call \method iteratively within a naive path planning algorithm to compute a complete camera trajectory avoiding collisions and obtain a complete 3D model. Despite being trained only on ShapeNet 3D models, it generalizes to complex scenes as shown above. (input 3D model courtesy of Brian Trepanier, under CC License. Downloaded from Sketchfab.)}%
    \label{fig:teaser}%
\end{figure}

As shown in Figure~\ref{fig:teaser}, we introduce a volumetric DL method that efficiently identifies NBVs for unknown large-scale 3D scenes in which the camera can move freely.  Instead of representing the scene with a single global embedding, we choose to use a grid of local point clouds, which scales much better to large and complex 3D scenes. We show how to learn to predict the visibility of unseen 3D points in all directions given an estimate of the 3D scene geometry. We can then integrate these visibilities in the direction of any camera by using a Monte Carlo integration approach, which allows us to optimize the camera pose to find the next most informative views.  We call our method \method, for Surface Coverage Optimization in uNknown Environments.

In this respect, we introduce a theoretical framework to translate the optimization of surface coverage gain, a surface metric on manifolds that represents the ability of a camera to increase the visible area of the surface, into an optimization problem on volumetric integrals. Such a formalism allows us to use a volumetric mapping of geometry, which is convenient not only to scale the model to exploration tasks and scene reconstruction, but also to make probabilistic predictions about geometry.

In particular, given a partial point cloud gathered by a depth sensor, our model learns to make predictions with virtually infinite resolution about the occupancy probability in the scene volume by learning a deep implicit function~\cite{lars-18-occupancy_network, xu-19-disn, mildenhall-20-nerf, yariv-2020-idr, yariv-2021-volsdf, oechsle-21-unisurf}. Such predictions scale to very large point clouds since they depend only on neighborhood geometry. Then, our model leverages a self-attention mechanism~\cite{ashish-2017-attention} to predict occlusions and compute informative functions mapped on a continuous sphere that represent visibility likelihood of points in all directions. The occupancy probability field is finally used as a basis to sample points and compute Monte Carlo integrals of visibility scores.

Since NBV learning-based methods are mostly limited to single, small-scale, centered object reconstruction in literature, we first compare the performance of our model to the state of the art on the ShapeNet dataset~\cite{shapenet2015}, following the protocol introduced in \cite{zeng-icirs20-pcnbv}. While our method was designed to handle more general frameworks such as 3D scene reconstruction and continuous cameras poses in the scene, it outperforms the state of the art for dense reconstruction of objects when the camera is constrained to stay on a sphere centered on the object. We then conduct experiments in large 3D environments using a simple planning algorithm that builds a camera trajectory online by iteratively selecting NBVs with \method.  Since, to the best of our knowledge, we propose the first supervised Deep Learning method for such free 6D motion of the camera, we created a dataset made of several large-scale scenes under the CC License for quantitative evaluation. We will make our code and this dataset available for allowing comparison of future methods with \method on our project webpage: \url{https://github.com/Anttwo/SCONE}.

\vincentrmk{What about the code?}

\antoinermk{Indeed we are the first, for NBV selection of structures with detailed geometry (with free 6D motion, it's true). Maybe we should insist on this point; there are some works on path planning with coarse geometry, but their focus is not on high quality of reconstruction. There are also some RL approaches too but it's really different and in this particular field they all look quite the same in my opinion ahah. I should check once again the SOTA to be sure}

\section{Approach}
% I reworked the intro to this part. A bit longer, but I think it is much easier to understand now.
% Partial surface point cloud VS Proxy point cloud

\antoine{Let us consider a depth sensor exploring a 3D scene, at time step $t \geq 0$. Using its observations at discrete time steps $j$ with $0 \leq j \leq t$, the sensor has gathered a cloud of points distributed on the surface of the scene. We refer to this cloud as the \emph{partial surface point cloud}, as it describes the part of the surface seen --or \emph{covered}-- by the sensor in the scene. To solve the NBV problem, we want to identify a camera pose that maximizes the coverage of previously unseen surface.

To this end, our method takes as input the partial surface point cloud as well as the history of 6D camera poses at time steps $j \leq t$ (\ie all previous positions and orientations of the sensor). Our approach is built around two successive steps, each relying on a dedicated neural module as shown in figure~\ref{fig:entire_method}: First, we make a prediction about the geometry of the scene, to estimate where the uncovered points could be. Then, we predict the visibility gain of uncovered points from any new camera pose; The NBV is finally selected as the camera with the most new visible points in its field of view.

\begin{figure}
  \centering
  %\fbox{\rule[-.5cm]{0cm}{4cm} \rule[-.5cm]{4cm}{0cm}}
  \includegraphics[width=13.5cm, trim={0 0cm 0 0}, clip]{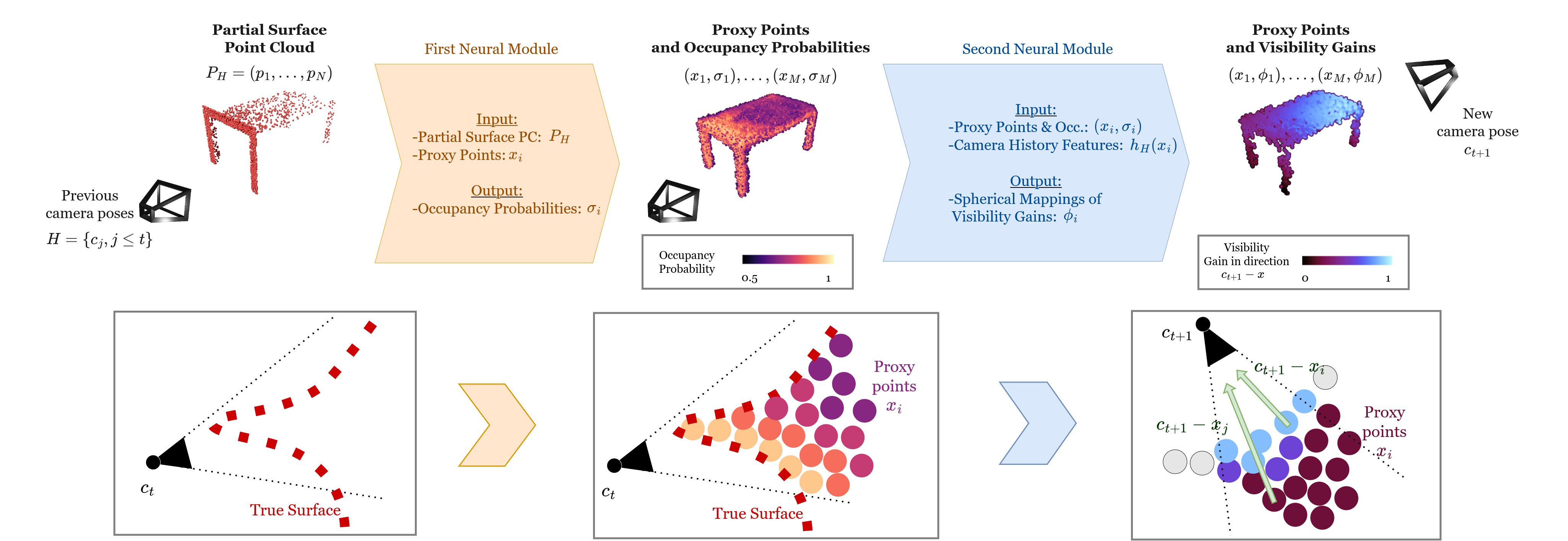}
  \caption{\label{fig:entire_method} {\bf The main steps of our method \method}. At time $t$, the depth sensor has visited camera poses $H=\{c_j, j\leq t\}$ and gathered a partial surface point cloud $P_H$ on the true surface, shown in red in the left image. Using its first neural module, our model predicts from the real surface points $P_H$ an occupancy probability distribution over proxy points $(x_1, ..., x_N)$ shown in the middle. The points $(x_1, ..., x_N)$ are sampled uniformly in the scene; we refer to them as \emph{proxy points} because we use them to encode the volume. For readability, the figure does not show the proxy points with an occupancy value under $0.5$. To compute the coverage gain of any next camera pose $c_{t+1}$, the model samples a subset of proxy points $x_i$ in the field of view of $c_{t+1}$ and uses its second module to predict the visibility gain in direction $c_{t+1}-x_i$ for each point $x_i$, as shown on the right. The proxy points are sampled with probabilities proportional to their occupancy value. Moreover, for each proxy point $x_i$, the second module encodes the relative positions of previous cameras with specific features $h_H(x_i)$. We finally integrate visibility gains over proxy points in the field of view of $c_{t+1}$ to approximate the volumetric coverage gain integral appearing in Equation~\ref{eqn:volumetric_approximation_of_coverage}.}
\end{figure}

Although we seek to maximize a surface metric such as surface coverage gain, our method relies on a volumetric representation of the object or scene to reconstruct. In this regard, we show that we can maximize a surface metric by integrating over a volumetric representation with virtually infinite resolution. As we argue below, such a representation is not only useful for collision-free robot navigation but is also much more efficient for optimizing surface coverage gain than the alternative of identifying the 3D points lying on the surface, which is difficult in an unknown and occluded environment. More exactly, we derive a volumetric integral that is asymptotically proportional to the surface coverage gain metric, which is enough for our maximization problem. 

In the following subsection, we first present this derivation considering volume represented as a perfect binary occupancy map. We then present the two neural modules of SCONE and explain how we use them to predict all terms involved in the volumetric integral, by leveraging neural models and self-attention mechanisms to predict occupancy and occlusions.}

\subsection{Maximizing Surface Coverage Gain on a Binary Occupancy Map}

Here, we consider a binary occupancy map $\occfield:\IR^3 \rightarrow \{0,1\}$ representing the volume of the target object or scene. We will relax our derivations to a probabilistic occupancy map when looking for the next best view in the next subsections. From the binary map $\occfield$, we can define the set $\chi$ of occupied points, \ie, the set of points $x$ verifying $\occfield(x) = 1$, its surface as the boundary $\partial\chi$, and the surface coverage $C(c)$ achieved by a camera pose $c=(c_\pos, c_\rot)\in\calC:=\IR^3\times SO(3)$ as the following surface integral:
\begin{equation}
    \label{eqn:surface_coverage_definition}
    C(c) = \frac{1}{|\partial\chi|_S} \int_{\partial\chi} \absvis_c(x)\,\text{d}x,
\end{equation}
where $|\partial \chi|_S := \int_{\partial\chi}\text{d}x$ is the area of surface $\partial\chi$. $\chi_c \subset \chi$ is the subset of occupied points contained in the field of view of camera $c$, and $\absvis_c(x)$ is the visibility of point $x$ from camera $c$, \ie, $\absvis_c(x) = \mathbb{1}_{\chi_c}(x) \cdot \mathbb{1}\left(\occfield\left(\{(1-\lambda) c_\pos + \lambda x \text{ such that } \lambda\in[0,1)\}\right)=\{0\}\right)$. 

Since we want to maximize the total coverage of the surface by all cameras during reconstruction, we are actually interested in maximizing the coverage of previously unobserved points rather than the absolute coverage.  Given a set of previous camera poses, which we call the \textit{camera history} $\camhist\subset \calC$, and a 3D point $x$, we introduce the \textit{knowledge indicator} $\gamma_H:\IR^3 \rightarrow \{0, 1\}$ such that $\gamma_H(x)=\max\{\absvis_c(x):c\in H\}$. We then define the \textit{coverage gain} $G_H(c)$ of camera pose $c$ as:
\begin{equation}
    \label{eqn:coverage_gain_definition}
    G_H(c) = \frac{1}{|\partial\chi|_S} \int_{\partial\chi} \vis^H_c(x)\,\text{d}x,
\end{equation}
where $\vis^H_c(x) = (1-\gamma_H(x)) \cdot \absvis_c(x)$ is the \textit{visibility gain} of $x$ in $\chi_c$, for camera history $H$. This function is equal to 1 iff $x$ is visible at pose $c$ but was not observed by any camera pose in $H$. Given a camera history $H$, our goal is to identify a pose $c$ that maximizes $G_H(c)$.

Given an occupancy map $\occfield$, we could evaluate the integral in Eq.~\eqref{eqn:coverage_gain_definition} by simply sampling points $p$ on surface $\partial\chi$. \antoine{However, in practice we will estimate the occupancy map iteratively in an unknown environment, and we will only have access to an occupancy probability distribution. Extracting surface points from such a probabilistic occupancy map gives results that can differ a lot from the true surface: Indeed, in 3D, a surface acts as a very concentrated set with zero-measure, and requires high confidence to give meaningful results. Instead of extracting surface points, we extend the properties of such points to a small spherical neighborhood of the surface. This will allow us to replace the maximization of a surface metric by the maximization of a volumetric integral, which is much easier to compute from our volumetric representation.}

More exactly, we assume there exists a quantity $\mu_0 > 0$ such that any volume point in the spherical neighborhood $T(\partial\chi, \mu_0) := \left\{ p \in \IR^3 \ | \ \exists x \in \partial\chi, \|x-p\|_2 < \mu_0 \right\}$ keeps the same visibility property as its neighboring surface points. With such a hypothesis, we give a thickness to the surface, which makes sense when working with discrete points sampled in space to approximate a volume.

To this end, we introduce a new visibility gain function $\newvis_c^H$ to adapt the definition of the former visibility gain $\vis_c^H$. For any $0<\mu<\mu_0$:
\begin{equation}
        \newvis_{c}^H(\mu; x) = 
        \begin{cases}
            1 & \text{if } \exists x_0 \in \partial\chi, \lambda <\mu \text{ such that } x = x_0 + \lambda N(x_0) \text{ and } \vis_c^H(x_0) = 1,\\
            0 & \text{otherwise} \> ,
        \end{cases}
\end{equation}
where $N$ is the inward normal vector field. With further regularity assumptions about the surface that are detailed in the appendix, such quantities are well defined. Assuming $\mu_0$ is small enough, the following explicit formula translates the surface approach into a volume integral for any camera pose $c \in \calC$ and $\mu < \mu_0$:
\begin{equation}
    \label{eqn:link_volume_surface}
    \int_{T(\partial\chi, \mu)}  g_c^H(\mu; x) \text{d}x
    = \int_{\partial\chi} \int_{-\mu}^\mu g_c^H(\mu; x_0 + \lambda N(x_0))\,\det(I-\lambda W_{x_0})\,\text{d}\lambda\,\text{d}x_0,
\end{equation}
%
%with $g_c^{H}(\mu,\cdot):x \mapsto \newvis_c^H(\mu;x)$ and 
with $W_{x_0}$ the Weingarten map at $x_0$, that is, the Hessian of the signed distance function on the boundary of $\chi$, which is continuous on the scene surface, assumed to be compact~\cite{gilbarg}.

By developing the determinant, we find that $\det(I - \lambda W_{x_0}) = 1 + \lambda b(\lambda, x_0)$ where $b$ is a bounded function on the compact space $[-\mu, \mu] \times \partial \chi$. Moreover, for all $x_0\in \partial\chi$, we have by definition $g_c^H(\mu;x_0 + \lambda N(x_0)) = g_c^H(\mu;x_0) = \vis_c^H(x_0)$ when $0 \leq \lambda < \mu$, and $g_c^H(\mu;x_0 + \lambda N(x_0))= 0$ when $-\mu < \lambda < 0$. It follows that, for every $0<\mu<\mu_0$:
\begin{equation}
\begin{split}
    \int_{T(\partial\chi, \mu)}  g_c^H(\mu;x) \text{d}x & = \int_{\partial\chi} \int_{0}^\mu g_c^H(\mu;x_0)(1 + \lambda b(\lambda, x_0)) \,\text{d}\lambda \,\text{d}x_0\\
    & = \mu \int_{\partial\chi} g_c^H(\mu;x_0) \,\text{d}x_0 + \int_{\partial\chi} \int_{0}^\mu \lambda g_c^H(\mu;x_0) b(\lambda, x_0)\,\text{d}\lambda \,\text{d}x_0\\
    & = \mu |\partial\chi|_S G_H(c) + \int_{\partial\chi} \int_{0}^\mu \lambda g_c^H(\mu;x_0) b(\lambda, x_0)\,\text{d}\lambda \,\text{d}x_0.
\end{split}
\end{equation}
The complete derivations are given in the appendix.

Function $g_c^H(\mu;\cdot)$ is naturally equal to 0 for every point outside $T(\partial\chi, \mu)$. Moreover, considering the regularity assumptions we made on the compact surface, if $\mu_0$ is chosen small enough then for all $x_0\in \partial \chi, \mu < \mu_0$, the point $x_0 + \mu N(x_0)$ is located inside the volume, such that $\int_{T(\partial\chi, \mu)}  g_c^H(\mu;x)\,\text{d}x = \int_{\chi}  g_c^H(\mu;x)\,\text{d}x$. Since $|g_c^H(\mu;\cdot)|\leq 1$ for all $c\in\mathcal{C}$ and $\mu>0$, we deduce the following theorem by bounding $|b|$ on $[-\mu, \mu] \times \partial \chi$:
\begin{theorem}
    \label{thm:volumetric_approximation_of_coverage}
    Under the previous regularity assumptions on the volume $\chi$ of the scene and its surface $\partial\chi$, there exist $\mu_0 > 0$ and $M>0$ such that for all $\mu < \mu_0$, and any camera $c\in \calC$:
    \begin{equation}
    \label{eqn:volumetric_approximation_of_coverage}
         \left| \frac{1}{|\chi|_V}\int_{\chi}  g_c^H(\mu;x) \text{d}x -  \mu \frac{|\partial\chi|_S}{|\chi|_V} G_H(c) \right| \leq M \mu^2 \> ,
    \end{equation}
    where $|\chi|_V$ is the volume of $\chi$.% and $|\partial\chi|_S$ is the area of its surface $\partial\chi$.
\end{theorem}

This theorem states that, asymptotically for small values of $\mu$, the volume integral $\int_\chi  g_c^H(\mu;x)\,\text{d}x$ gets proportional to the surface coverage gain values $G_H(c)$ that we want to maximize. This result is convenient since a volume integral can be easily approximated with Monte-Carlo integration on the volume and a uniform dense sampling based on the occupancy function $\occfield$. Consequently, the more points we sample in the volume, the smaller $\mu$ we can choose, and the closer maximizing the volume integral of spherical neighborhood visibility gain gets to maximizing the surface coverage gain.

\subsection{Architecture}

\antoine{To approximate the volumetric integral in Equation~\ref{eqn:volumetric_approximation_of_coverage} for any camera pose $c$, we need to compute $\chi_c$ as well as function $g_c^H$. In this regard, we need to compute both the occupancy map and the visibility gains of points for any camera pose. Since the environment is not perfectly known, we predict each one of these functions with a dedicated neural module. The first module takes as input the partial point cloud gathered by the depth sensor to predict the occupancy probability distribution. The second module takes as input a camera pose, a feature representing camera history as well as a predicted sequence of occupied points located in the camera field of view to predict visibility gains.} %The visibility gains are aggregated using Monte-Carlo integration to compute the resulting coverage gain.

\paragraph{Predicting the occupancy probability field $\probfield$. } The occupancy function $\occfield$ is not known perfectly in practice. \antoine{To represent occupancy, most volumetric NBV methods rely on memory-heavy representations (like an occupancy 3D-grid or a volumetric voxelization), that are generally less efficient for encoding fine details and optimizing dense reconstructions, and will necessarily downgrade the resolution compared to a point cloud directly sampled on the surface. To address this issue while still working with a volumetric representation of the scene, we use a deep implicit function to encode the 3D mapping of occupancy efficiently. Such a function has a virtually infinite resolution, and prevents us from saving a large 3D grid in memory.} We thus approximate $\occfield$ with the first module of our model, which consists of a neural network $\probfield : \calP(\IR^3) \times \IR^3 \rightarrow [0,1]$ that takes as inputs a partial surface point cloud $P_H \subset \IR^3$ and query points $x \in \IR^3$, and outputs the occupancy probability $\probfield(P_H; x)$ of $x$. $P_H$ is obtained by merging together all depth maps previously captured from cameras in $H$.

\begin{figure}
  \centering
  %\fbox{\rule[-.5cm]{0cm}{4cm} \rule[-.5cm]{4cm}{0cm}}
  \includegraphics[width=8cm, trim={0 0cm 0 0}, clip]{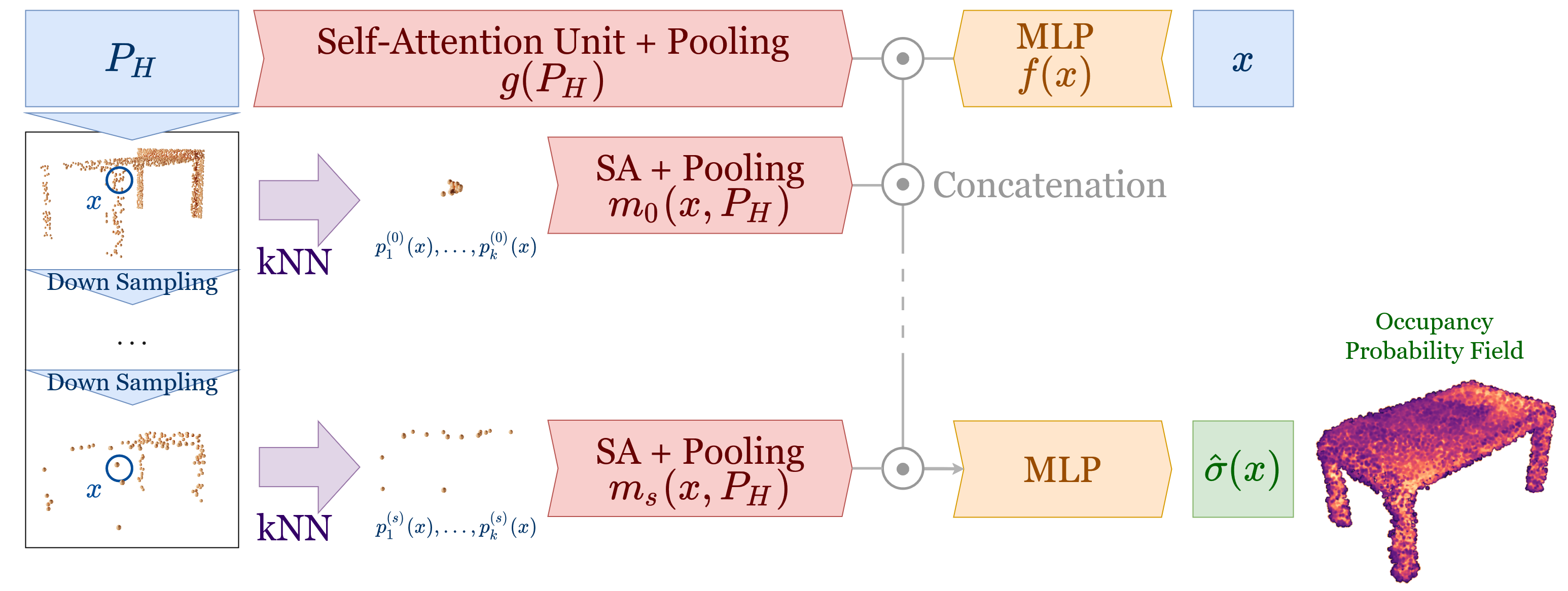}
  \caption{\label{fig:occupancy_probability_prediction} {\bf Architecture of the first module of \method}, which predicts the occupancy probability field $\probfield$. This module predicts the occupancy probability of a point $x$ from several inputs: The point $x$ after transformation by an MLP $f(x)$; A coarse global encoding $g(P_H)$ of the point cloud obtained by applying self-attention units on the sequence of points followed by a pooling operation; Multi-scale neighborhood features $m_i(x, P_H), i=0,...,N$ computed by down-sampling multiple times the point cloud, and encoding the $k$ nearest neighbors of $x$ with self-attention units after each sampling.}
\end{figure}

As shown in Figure~\ref{fig:occupancy_probability_prediction}, rather than using a direct encoding of the global shape of $P_H$ as input to $\probfield$, we take inspiration from \cite{philipp-20-points2surf} to achieve scalability and encode $P_H$ using features computed from the points' neighborhoods. The difference with \cite{philipp-20-points2surf} is that we rely on a multiscale approach: For a given query 3D point $x$, these features are computed from the $k$ nearest neighbors of $x$ computed at different scales. For each scale $s$, we downsample point cloud $P_H$ around $x$ into a sparser point cloud $P_H^{(s)}$ before recomputing the nearest neighbors $p^{(s)}_i(x), i=1,...,k$ of $x$: In this way, the size of the neighborhood increases with scale.

Next, for each value of $s$, we use small attention units \cite{ashish-2017-attention, meng-20-pct} on the sequence of centered neighborhood points ($p^{(s)}_1(x) - x, ..., p^{(s)}_k(x) - x)$ and apply pooling operations to encode each sequence of $k$ neighbors into a single feature that describes the local geometry for the corresponding scale. We finally concatenate these different scale features with another uncentered global feature as well as the query point $x$, and feed them to an MLP to predict the occupancy probability. The last global feature aims to provide really coarse information about the geometry and the location of $x$ in the scene.

This model scales well to large scenes: Adding points from distant views to the current partial point cloud does not change the local state of the point cloud. 
 To avoid computing neighborhoods on the entire point cloud when reconstructing large scenes, we partition the space into cells in which we store the points in $P_H$. Given a query point $x$, we only use the neighboring cells to compute $p^{(s)}_i(x)$. %, i=1,...,k$. 

\paragraph{Predicting the visibility gain $\newvis^H_{c}$. } 
To maximize surface coverage gain, we need to compute the volumetric integral of visibility gain functions $\newvis^H_{c}$. We do this again by Monte Carlo sampling, however, in unknown environments we cannot compute explicitly occlusions to derive visibility gain functions $\newvis^H_{c}$ since the geometry, represented as a point cloud, is partially unknown and sparser than a true surface. We thus train the second module of our model to predict visibility gain functions by leveraging a self-attention mechanism that helps to estimate occlusion effects in the point cloud $P_H$.

In particular, for any camera pose $c\in \calC$ and 3D point $x\in \chi_c$, the second module derives its prediction of visibility gains from three core features: (i)~The predicted probability $\probfield(P; x)$ of $x$ to be occupied, (ii)~the occlusions on $x$ by the subvolume $\chi_c$ and (iii)~the camera history $H$. To feed all this information to our model in an efficient way, we follow the pipeline presented in Figure~\ref{fig:visibility_gain_computation}. The model starts by using the predicted occupancy probability function $\probfield$ to sample 3D points in the volume $\chi$. These samples will be used for Monte Carlo integration. We refer to these points as \emph{proxy points} as we use them to encode the  volume in the camera field of view, \ie, in a pyramidal frustum view. We write $\proxy$ as the discrete set of sampled proxy points, and $\proxy_c$ as the set of proxy points located in the field of view of the camera $c$.

We first encode these proxy points individually by applying a small MLP on their 3D coordinates and their occupancy probability value concatenated together. Then, our model processes the sequence of these encodings with a self-attention unit to handle occlusion effects of subvolume $\chi_c$ on every individual point. Note there is no pooling operation on the output of this unit: The network predicts per-point features and does not aggregate predictions, since we do it ourselves with Monte Carlo integration. Next, for each proxy point $x\in \proxy$ , we compute an additional feature $h_H(x)$ that encodes the history of camera positions $H$ with respect to this point as a spherical mapping: It consists in the projection on a sphere centered on $x$ of all camera positions for which $x$ was in the field of view. These features are concatenated to the outputs of the self-attention unit.

Our model finally uses an MLP on these features to predict the entire visibility gain functions of every point $x$ as a vector of coordinates in the orthonormal basis of spherical harmonics. With such a formalism, the model is able to compute visibility gains for points inside a subvolume in all directions with a single forward pass. In this regard, we avoid unnecessary computation and are able to process a large number of cameras in the same time when they share the same proxy points in their field of view (\eg, reconstruction of a single object centered in the scene, where $\proxy_c=\proxy$ for all $c$, or when several cameras observe the same part of the 3D scene, \ie, $\proxy_c=\proxy'\subset \proxy$ for several $c$).

Formally, if we denote by $Y_l^m : S^2 \rightarrow \IR$ the real spherical harmonic of rank $(l, m)$ and $\phi_l^m(\proxy_c;x, h_H(x))$ the predicted coordinate of rank $(l,m)$ for proxy point $x\in \proxy_c$ with attention to subset $\proxy_c$ and camera history feature $h_H(x)$, the visibility gain of point $x$ in direction $d\in S^2$ is defined as
\begin{equation}
    \label{eqn:predicted_visibility_gain}
    \sum_{l, m} \phi_l^m(\proxy_c; x, h_H(x)) \cdot Y_l^m \left(d\right)
\end{equation}
so that the coverage gain $G_H(c)$ of a camera pose $c\in \calC$ is proportional to
\begin{equation}
    \label{eqn:predicted_coverage_gain}
    I_H(c) := \frac{1}{|\proxy|} \sum_{x \in \proxy} \left[
    \mathbb{1_{\proxy_c}}(x)
    \sum_{l, m} \phi_l^m(\proxy_c; x, h_H(x)) \cdot Y_l^m \left(\frac{x-c_{pos}}{\|x-c_{pos}\|_2}\right) 
    \right].
\end{equation}
The next best view among several camera positions is finally defined as the camera pose $c^*$ with the highest value for $I_H(c^*)$. \antoine{Equation~\ref{eqn:predicted_coverage_gain} is a Monte Carlo approximation of the volumetric integral in Equation~\ref{eqn:volumetric_approximation_of_coverage}}, where the occupancy map and the visibility gains are predicted with neural networks.
We choose to use a Monte-Carlo integral rather than a neural aggregator because this approach is simple, fast, makes training more stable, has good performance, better interpretability, and can handle sequences of arbitrary size. In particular, it implicitly encourages our model to compute meaningful visibility gains for each point since there is no asymmetry between the points.

We also use spherical harmonics to encode the camera history encoding $h_H(x)$ of each point $x$, which makes it a homogeneous input to the predicted output. Consequently, this input comes at the end of the architecture, and aims to adapt the visibility gain according to previous camera positions. This convenient representation, inspired by \cite{yu-21-plenoxels}, allows us to handle free camera motion; On the contrary, several models in the literature encode the camera directions on a discrete sphere \cite{zeng-icirs20-pcnbv, Mendoza-2019, vasquez-21-3DCNN}.

\begin{figure}
  \centering
  % \fbox{\rule[-.5cm]{0cm}{4cm} \rule[-.5cm]{4cm}{0cm}}
  \includegraphics[width=13.5cm]{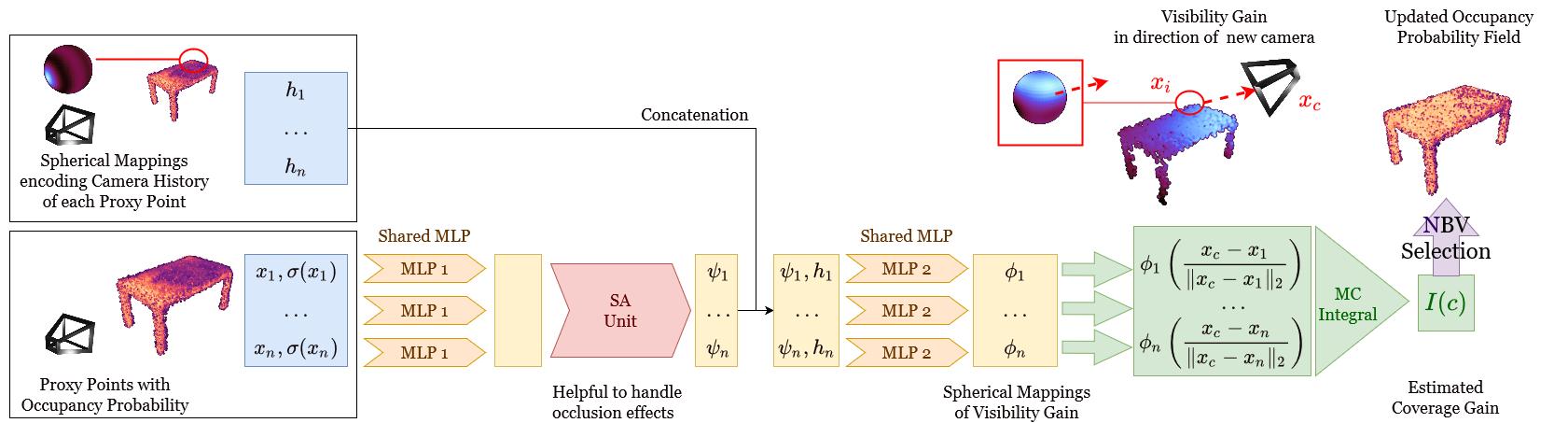}
  \caption{\label{fig:visibility_gain_computation}  {\bf Architecture of the second module of \method}, which predicts a visibility gain for each proxy point. To make this prediction, the model encodes the proxy points $x$ concatenated with their occupancy probability $\probfield(x)$. We use an attention mechanism to take into account occlusion effects in the volume between the proxy points and their consequences on the visibility gains.}
\end{figure}

\paragraph{Training.} We train the occupancy probability module alone with a Mean Squared Error loss with the ground truth occupancy map. We do not compute ground-truth visibility gains to train the second module since it would make computation more difficult and require further assumptions: We supervise directly on $I_H(c)$ by comparing it to the ground-truth surface coverage gain for multiple cameras, with softmax normalization and Kullback-Leibler divergence loss. Extensive details about the training of our modules and the choices we made are given in the appendix.

\section{Experiments}

As discussed in the introduction, deep learning-based NBV models for dense reconstruction are currently limited to single, small-scale, centered object reconstruction. To compare \method to these previous methods in this context, we first constrain the camera pose to lie on a sphere centered on an object. We then introduce our dataset made of 13 large-scale 3D models that we created to evaluate \method on free camera motions (3D models courtesy of Brian Trepanier, Andrea Spognetta, and 3D Interiors, under CC License; all models were downloaded from the website Sketchfab).

\begin{table}
  \caption{\label{tab:auc_shapenet}  {\bf AUCs of surface coverage for several NBV selection methods for dense object reconstruction}, as computed on the ShapeNet test dataset following the protocol of \cite{zeng-icirs20-pcnbv}\vincent{, and after averaging over multiple seeds in the case of our method}\vincentrmk{Correct?}. For this experiment, we constrain the camera to stay on a sphere centered on the objects in order to compare with previous methods. Even if our model is designed to scale to entire scene reconstructions with free camera motion, it is still able to beat other methods trained for the specific case of dense object reconstruction with constrained camera motion.}
  \centering
  \scalebox{0.93}{ % previously at 0.96...
  {\scriptsize
  \begin{tabular}{@{}lccccccccc@{}}
    \toprule
    \multicolumn{1}{c}{} & \multicolumn{8}{c}{Categories seen during training} \\
    \cmidrule(r){2-9}
     Method & Airplane & Cabinet & Car & Chair & Lamp & Sofa & Table & Vessel & Mean \\
    \midrule
    Random & 0.745 & 0.545 & 0.542 & 0.724 & 0.770 & 0.589 & 0.710 & 0.674 & 0.662 \\
    Proximity Count~\cite{delmerico-18-acomparison} & 0.800 & 0.596 & 0.591 & 0.772 & 0.803 & 0.629 & 0.753 & 0.706 & 0.706\\
    Area Factor~\cite{vasquez-14-volumetric-nbv} & 0.797 & 0.585 & 0.587 & 0.751 & 0.801 & 0.627 & 0.725 & 0.714 & 0.698\\
    NBV-Net~\cite{Mendoza-2019} & 0.778 & 0.576 & 0.596 & 0.743 & 0.791 & 0.599 & 0.693 & 0.667 & 0.680\\
    PC-NBV~\cite{zeng-icirs20-pcnbv} & 0.799 & 0.612 & \textbf{0.612} & \textbf{0.782} & 0.800 & 0.640 & 0.760 & 0.719 & 0.716 \\
    \method (Ours) & \textbf{0.827} & \textbf{0.625} & 0.591 & \textbf{0.782} & \textbf{0.819} & \textbf{0.662} & \textbf{0.792} & \textbf{0.734} & \textbf{0.729} \\
    % \bottomrule
  \end{tabular}}
  }
 \scalebox{0.88}{ 
 {\scriptsize
  \begin{tabular}{@{}lccccccccc@{}}
    \toprule
    \multicolumn{1}{c}{} & \multicolumn{8}{c}{Categories not seen during training} \\
    \cmidrule(r){2-9}
     Method & Bus & Bed & Bookshelf & Bench & Guitar & Motorbike & Skateboard & Pistol & Mean \\
    \midrule
    Random & 0.609 & 0.619 & 0.695 & 0.795 & 0.795 & 0.672 & 0.768 & 0.614 & 0.694 \\
    Proximity Count & 0.646 & 0.645 & \textbf{0.749} & 0.829 & 0.854 & 0.705 & 0.828 & 0.660 & 0.740\\
    Area Factor & 0.629 & 0.631 & 0.742 & 0.827 & 0.852 & 0.718 & 0.799 & 0.660 & 0.732\\
    NBV-Net & 0.654 & 0.628 & 0.729 & 0.824 & 0.834 & 0.710 & 0.825 & 0.645 & 0.731\\
    PC-NBV & 0.667 & 0.662 & 0.740 & \textbf{0.845} & 0.849 & \textbf{0.728} & 0.840 & 0.672 & 0.750 \\
    \method (Ours) & \textbf{0.694} & \textbf{0.689} & 0.746 & 0.832 & \textbf{0.860} & \textbf{0.728} & \textbf{0.845} & \textbf{0.717} & \textbf{0.764} \\
    \bottomrule
  \end{tabular}}
  }
\end{table}

\subsection{Next Best View for Single Object Reconstruction}
\label{sec:single_object}

We first compare the performance of our model to the state of the art on a subset of the ShapeNet dataset~\cite{shapenet2015} introduced in \cite{zeng-icirs20-pcnbv} and following the protocol of \cite{zeng-icirs20-pcnbv}: We sample 4,000 training meshes from 8 specific categories of objects, 400 validation meshes and 400 test meshes from the same categories, and 400 additional test meshes from 8 categories unseen during training.

The evaluation on the test datasets consists of  10-view reconstructions of single objects. Given a mesh in the dataset, camera positions are discretized on a sphere. We start the reconstruction process by selecting a random camera pose, then we iterate NBV selection 9 times in order to maximize coverage with a sequence of 10 views in total. The evaluation metric is the area under the curve~(AUC) of surface coverage throughout the reconstruction. This criterion not only evaluates the quality of final surface coverage, but also the convergence speed toward a satisfying coverage. Results are presented in Table~\ref{tab:auc_shapenet}. \vincentrmk{Still supp mat? :}Further details about the evaluation, the estimation of ground truth surface coverage gains, and the metric computation are available in the appendix.

\subsection{Active View Planning in a 3D Scene}

\begin{figure}%
    \centering
    \subfloat[\centering Dunnottar Castle]{{\includegraphics[width=3.4cm, trim={0 0 0 0}, clip]{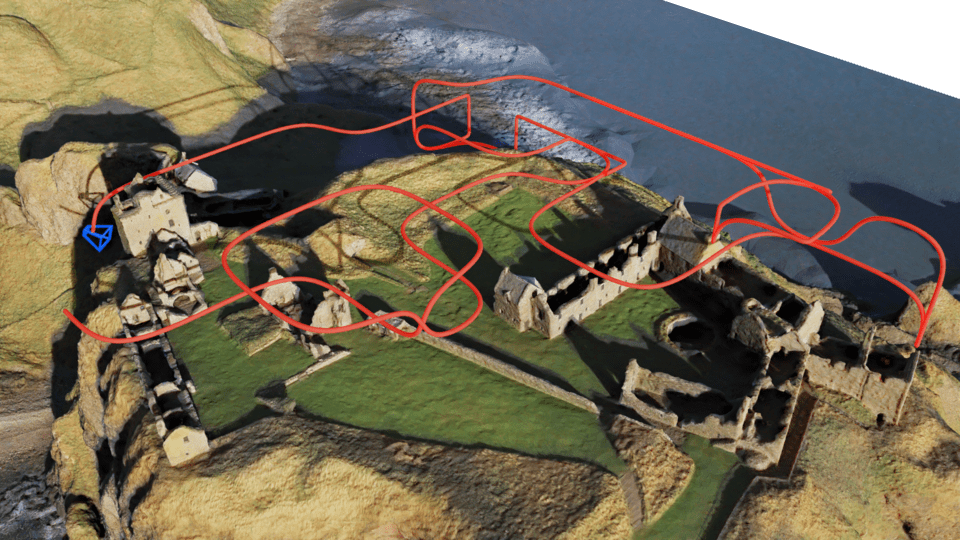} }}%
    \subfloat[\centering Pantheon]{{\includegraphics[width=3.4cm, trim={0 0 0 0}, clip]{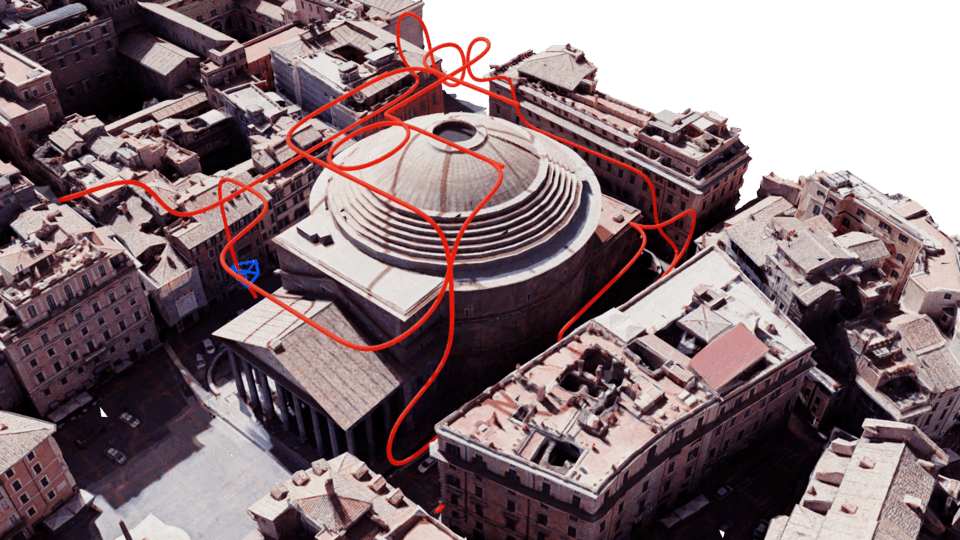} }}%
    \subfloat[\centering Statue of Liberty]{{\includegraphics[width=3.4cm, trim={0 0 0 0}, clip]{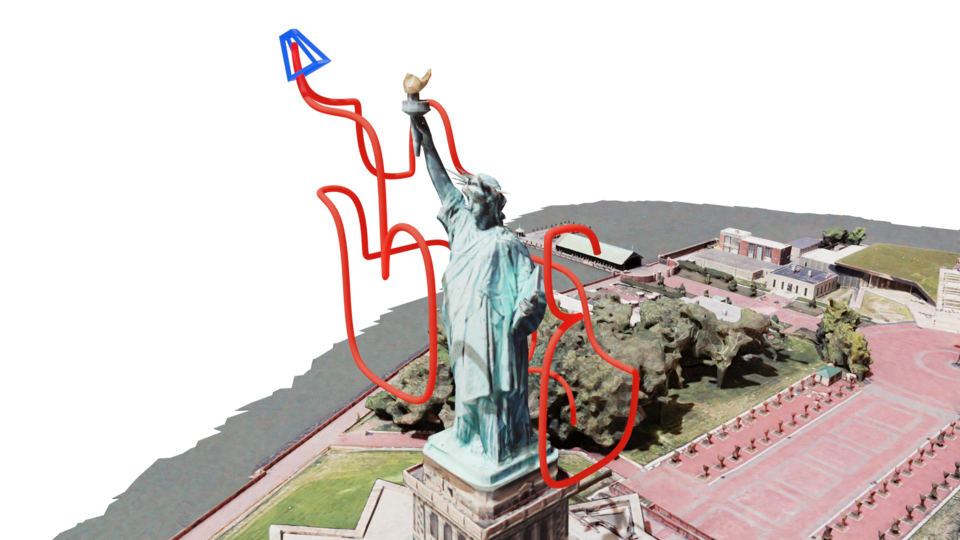} }}%
    \subfloat[\centering Leaning Tower, Pisa]{{\includegraphics[width=3.4cm, trim={0 0 0 0}, clip]{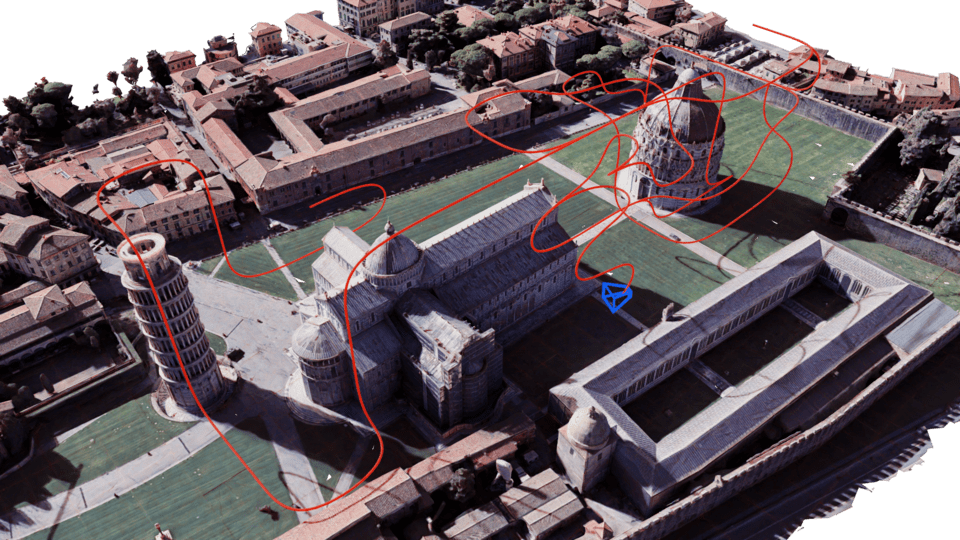} }}\\%
    
    \subfloat[\centering Fushimi Castle]{{\includegraphics[width=3.4cm, trim={0 0 0 0}, clip]{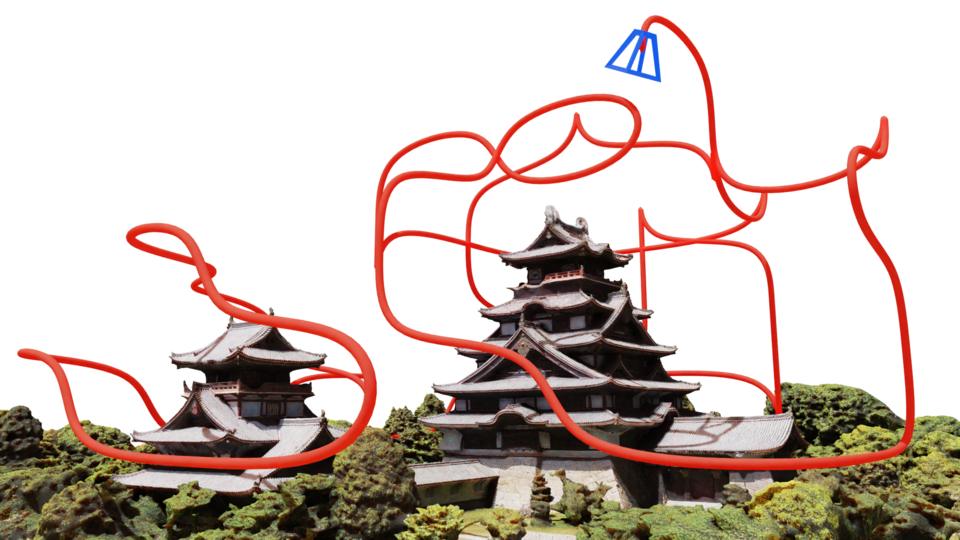} }}%
    \subfloat[\centering Alhambra Palace]{{\includegraphics[width=3.4cm, trim={0 0 0 0}, clip]{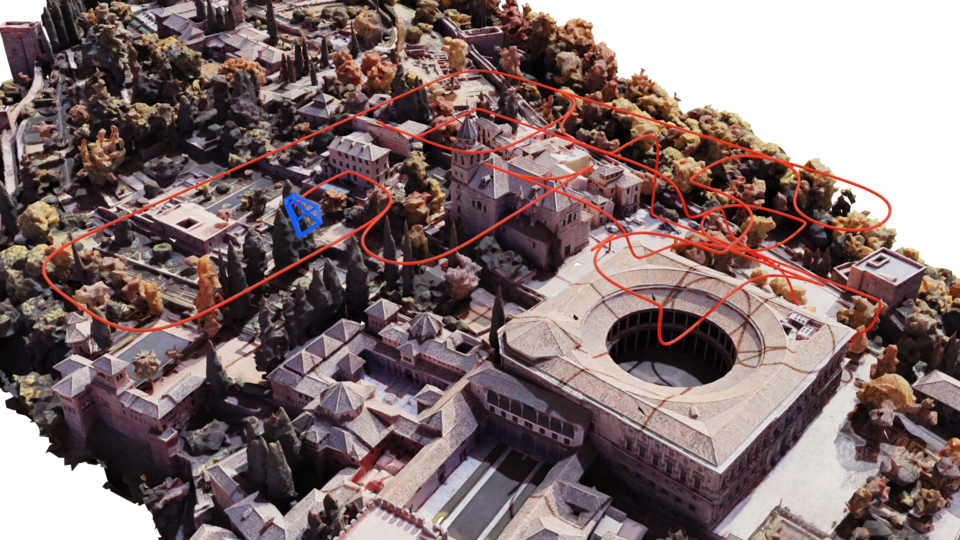} }}%
    \subfloat[\centering Eiffel Tower]{{\includegraphics[width=3.4cm, trim={0 0 0 0}, clip]{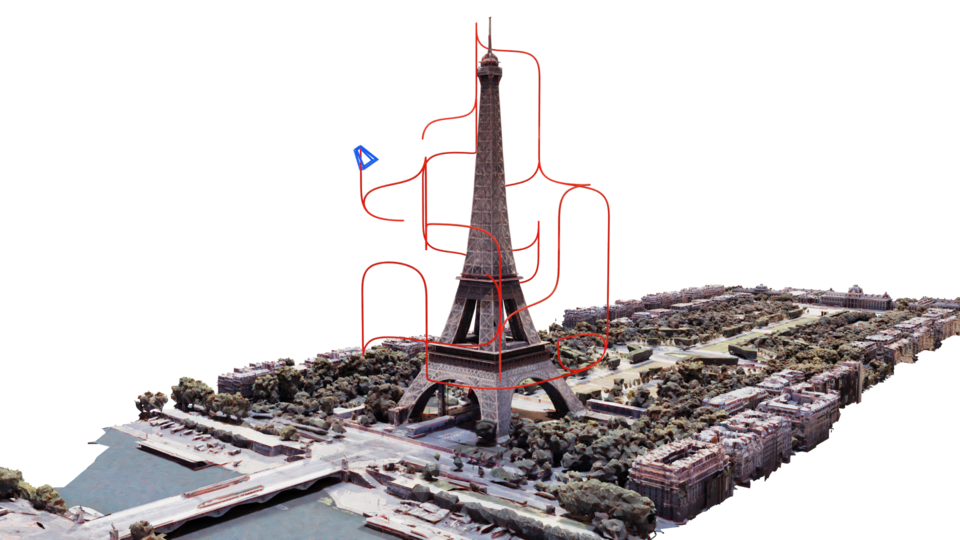} }}%
    \subfloat[\centering Natural History Museum]{{\includegraphics[width=3.4cm, trim={0 0 0 0}, clip]{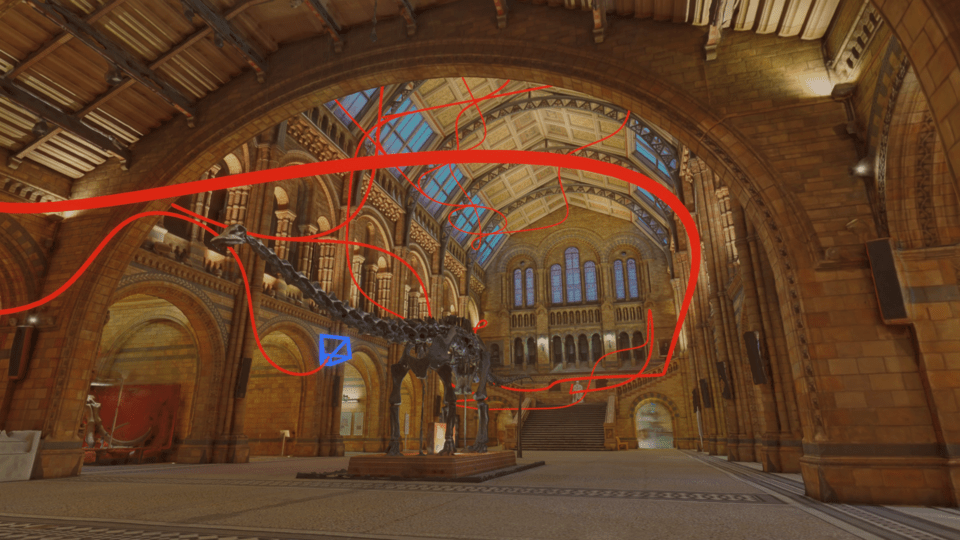} }}\\%
    % \subfloat[\centering Neuschwanstein Castle]{{\includegraphics[width=4.55cm, trim={0 0 0 0}, clip]{images/compressed/neuschwanstein.png} }}\\%
    
    \caption{\label{fig:trajectories} {\bf Application of our NBV approach to the reconstruction of large 3D structures}. The path, constructed iteratively in real time with our method, is shown in red (3D models courtesy of Brian Trepanier, Andrea Spognetta, and 3D Interiors, under CC License; all models were downloaded from the website Sketchfab).}
\end{figure}

\begin{figure}%
    \centering
    \subfloat[\centering Dunnottar Castle]{{\includegraphics[width=3.4cm, trim={0 0 0 0}, clip]{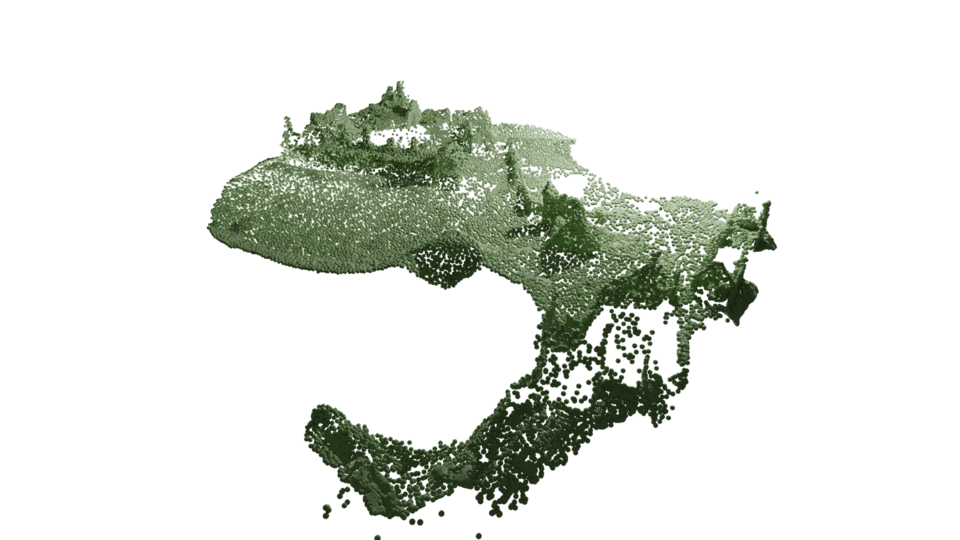} }}%
    \subfloat[\centering Pantheon]{{\includegraphics[width=3.4cm, trim={0 0 0 0}, clip]{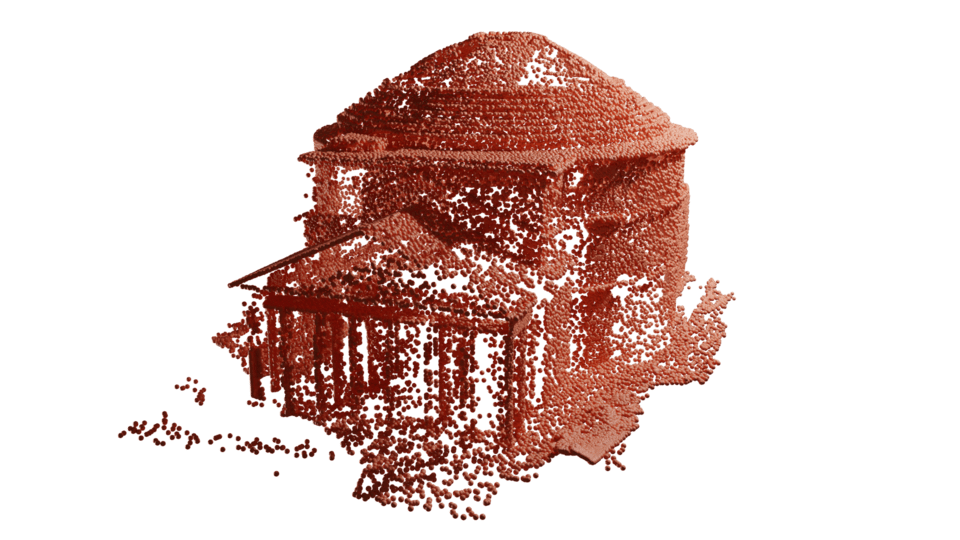} }}%
    \subfloat[\centering Statue of Liberty]{{\includegraphics[width=3.4cm, trim={0 0 0 0}, clip]{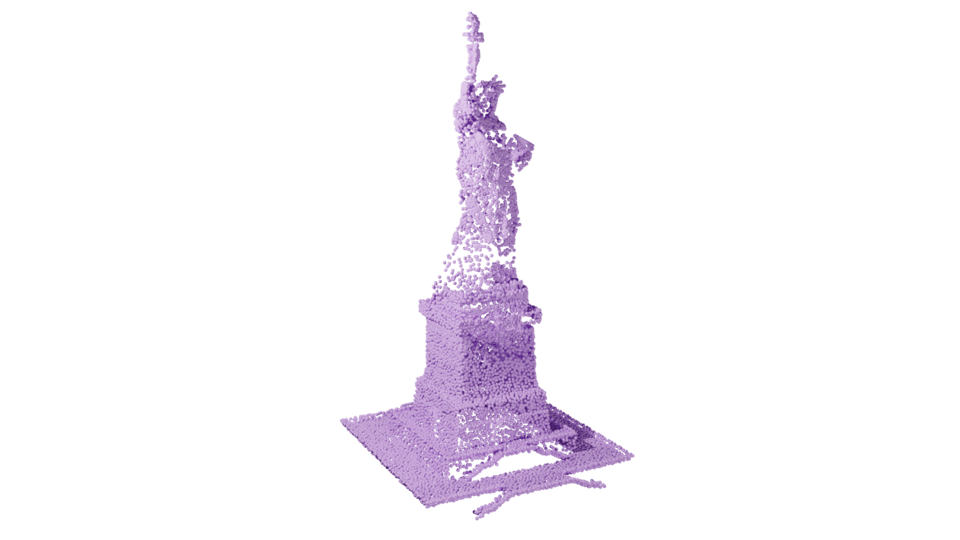} }}%
    \subfloat[\centering Leaning Tower, Pisa]{{\includegraphics[width=3.4cm, trim={0 0 0 0}, clip]{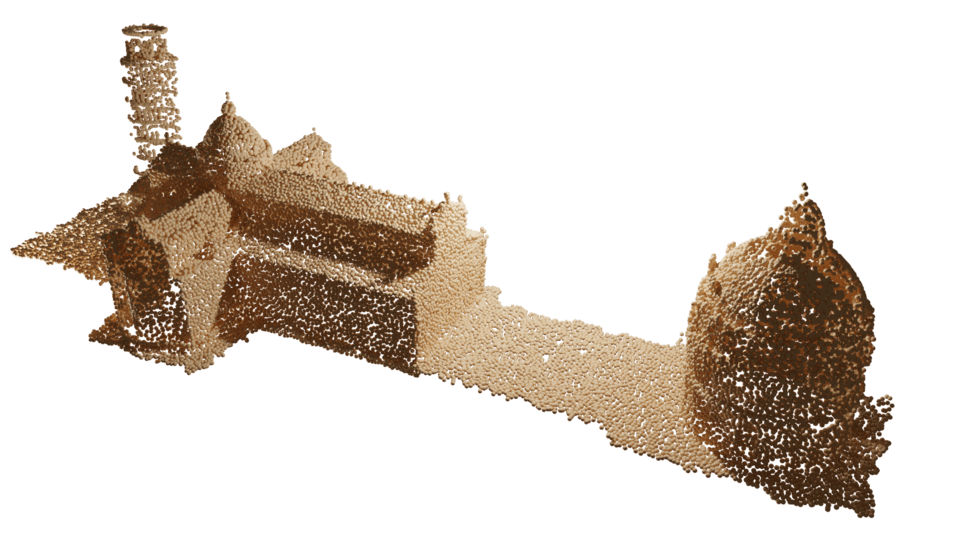} }}\\%
    
    \subfloat[\centering Fushimi Castle]{{\includegraphics[width=3.4cm, trim={0 0 0 0}, clip]{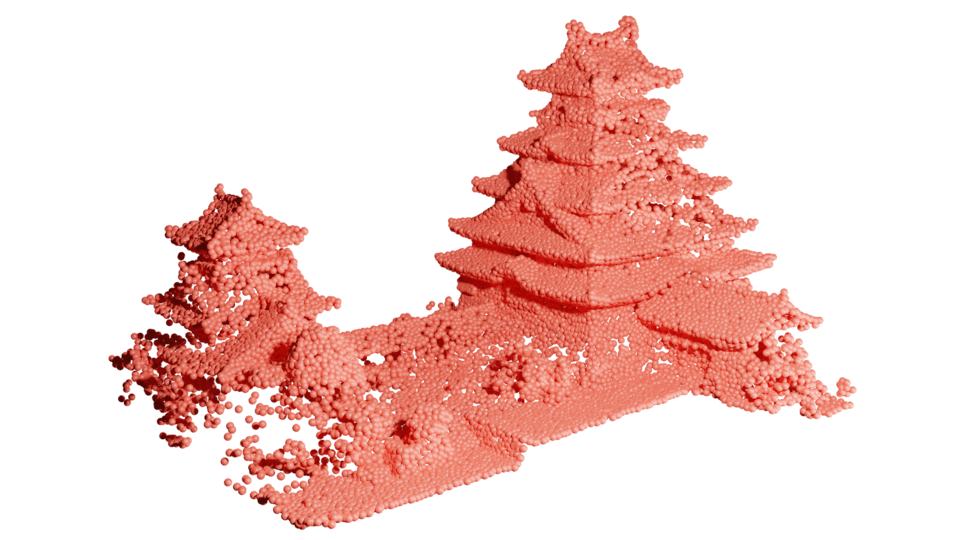} }}%
    \subfloat[\centering Alhambra Palace]{{\includegraphics[width=3.4cm, trim={0 0 0 0}, clip]{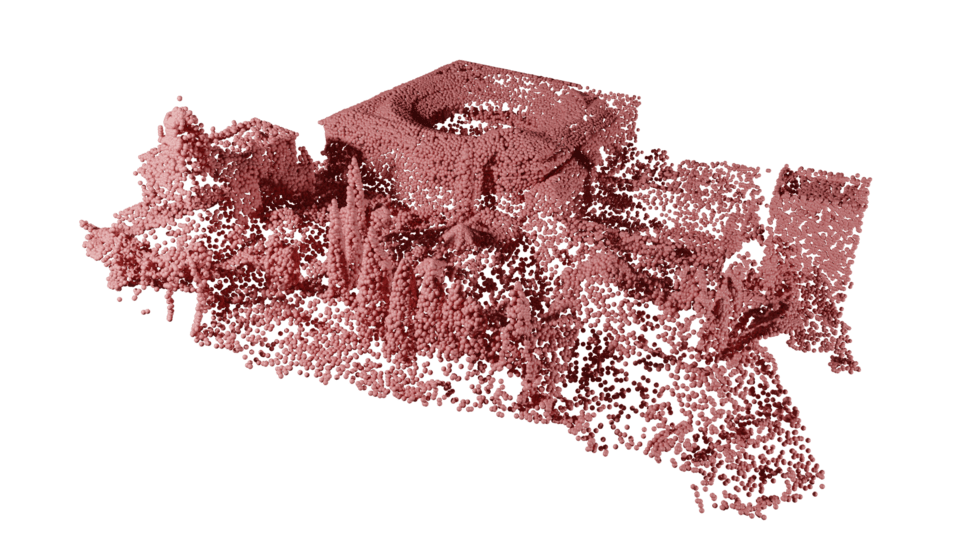} }}%
    \subfloat[\centering Eiffel Tower]{{\includegraphics[width=3.4cm, trim={0 0 0 0}, clip]{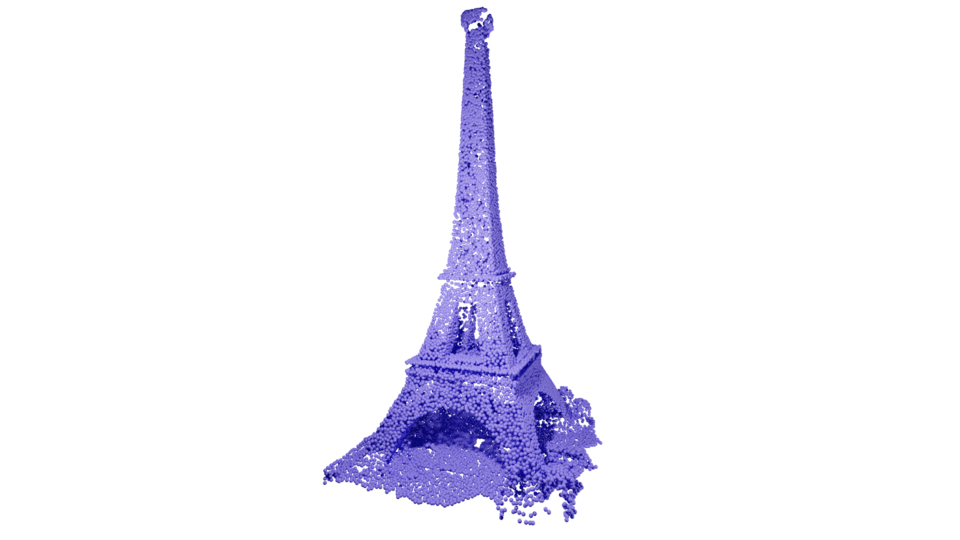} }}%
    \subfloat[\centering Natural History Museum]{{\includegraphics[width=3.4cm, trim={0 0 0 0}, clip]{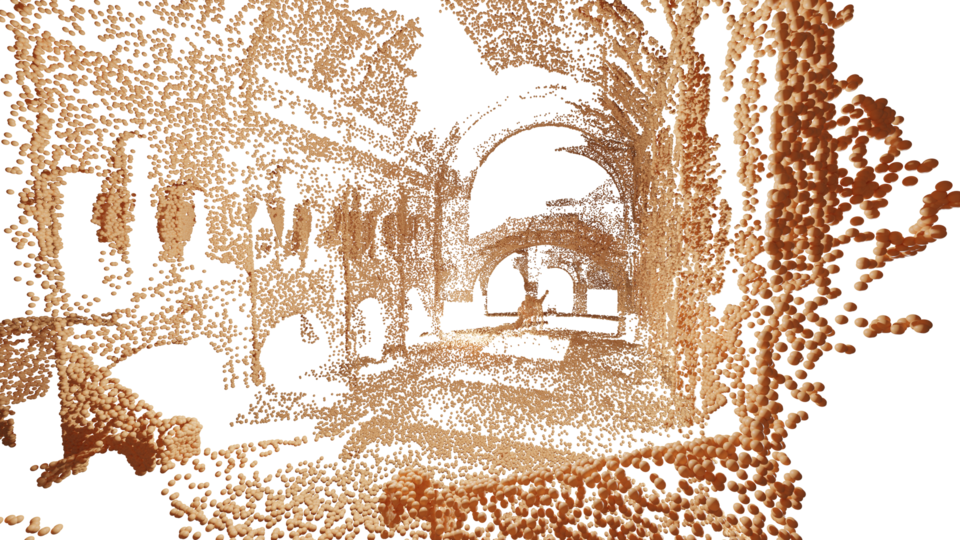} }}\\%
    % \subfloat[\centering Neuschwanstein Castle]{{\includegraphics[width=3.cm, trim={0 0 0 0}, clip]{images/pc_compressed/bavaria.png} }}\\%
    
    \caption{\label{fig:reconstruction} {\bf Reconstruction of large 3D scenes with our NBV approach \method}. We show the partial point cloud gathered by the depth sensor for each scene.}
\end{figure}

To evaluate the scalability of our model to large environments as well as free camera motion in 3D space, we also conducted experiments using a naive planning algorithm that incrementally builds a path in the scene: We first discretize the camera poses in the scene on a 5D grid, corresponding to coordinates $c_{pos} = (x_c, y_c, z_c)$ of the camera as well as the elevation and azimuth to encode rotation $c_{rot}$. The number of poses depends on the dimensions of the bounding box of the scene. This box is an input to the algorithm, as a way for the user to tell which part of the scene should be reconstructed. In our experiments, the number of different poses is around 10,000. Note we did not retrain our model for such scenes and use the previous model trained on ShapeNet as explained in Section~\ref{sec:single_object}.

The depth sensor starts from a random pose (from which, at least, a part of the main structure to reconstruct is visible). Then, at each iteration, our method estimates the coverage gain of direct neighboring poses in the 5D grid, and selects the one with the highest score. The sensor moves to this position, captures a new partial point cloud and concatenates it to the previous one. Since we focus on online path planning for dense reconstruction optimization and designed our model to be scalable with local neighborhood features, the size of the global point cloud can become very large. 
We iterate the process either 100 times to build a full path around the object in around 5 minutes on a single Nvidia GTX1080 GPU, and recovered a partial point cloud with up to 100,000 points.

Contrary to single, small-scale object reconstruction where the object is always entirely included in the field of view, we simulated a limited range in terms of field of view and depth for the depth sensor, so that the extent of the entire scene cannot be covered in a single view. Thus, taking pictures at long range and going around the scene randomly is not an efficient strategy; the sensor must find a path around the surface to optimize the full coverage. We use a ray-casting renderer to approximate a real Lidar. More details about the experiments can be found in the appendix.\vincentrmk{still supp mat?}

We compared the performance of \method with simple baselines: First, a random walk, which chooses neighboring poses randomly with uniform probabilities. Then, we evaluated an alternate version of our model, which we call \method-Entropy, which leverages the first module of our full model to compute occupancy probability in the scene, then selects the next best view as the position that maximizes the Shannon entropy in its field of view. The comparison is interesting, since \method-Entropy adapts classic NBV approaches based on information theory to a deep learning framework. Figures~\ref{fig:trajectories} and \ref{fig:reconstruction} provide qualitative results. Figure~\ref{tab:scene_reconstruction} shows the convergence speed of covered surface by \method and our two baselines, averaged on all scenes of the dataset. % Finally, Table~\ref{tab:scene_experiment_auc} reports surface coverage AUC in every scene for each method.

\begin{figure}%
    \centering
    \begin{tabular}{@{}cc@{}}
  \hspace{-0.0cm}\scalebox{0.75}{ 
  \begin{tabular}{@{}lccccccc@{}}
    \toprule
    \multicolumn{1}{c}{} & \multicolumn{7}{c}{3D scene} \\
    \cmidrule(r){2-8}
     & & {\scriptsize Statue of} & {\scriptsize Manhattan} & {\scriptsize Fushimi} & {\scriptsize Dunnottar} & {\scriptsize Neuschwanstein} &  \\[-1mm]
    Method & {\scriptsize Colosseum} & {\scriptsize Liberty} & {\scriptsize Bridge} & {\scriptsize Castle} & {\scriptsize Castle} & {\scriptsize Castle} & Mean \\
    \midrule
    Random Walk & 0.308 & 0.469 & 0.405 & 0.584 & 0.355 & 0.403 & 0.421\\
    SCONE-Entropy & 0.512 & 0.681 & 0.361 & 0.802 & 0.456 & 0.538 & 0.558\\
    SCONE & {\bf 0.571} & {\bf 0.693} & {\bf 0.685} & {\bf 0.841} & {\bf 0.739} & {\bf 0.653} & {\bf 0.697}\\
    \bottomrule
  \end{tabular}
 } 
 &
\raisebox{-.5\height}{\includegraphics[width=3.4cm]{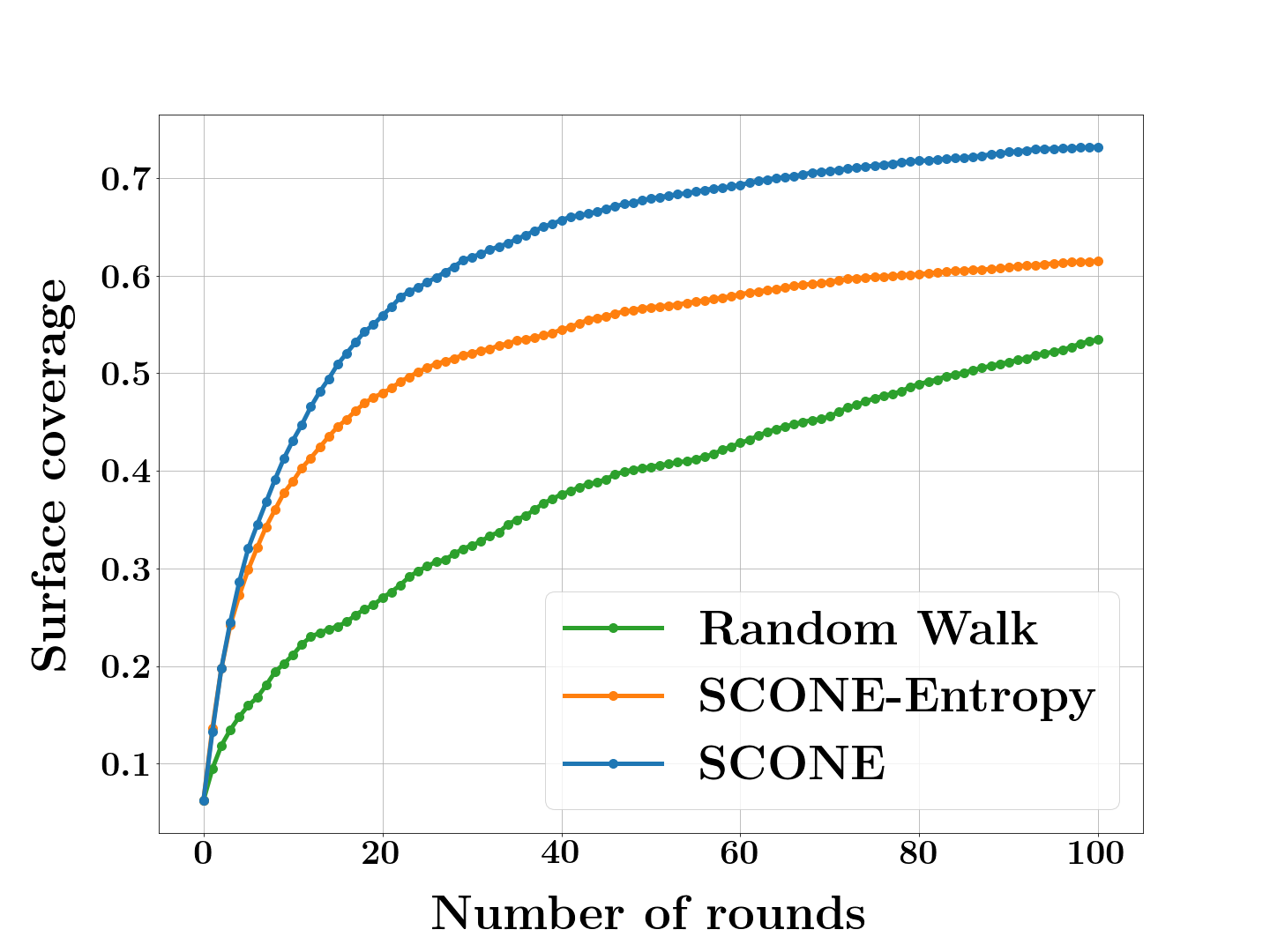}}\\[3mm]
(a) & (b)\\
    \end{tabular}
    \caption{\label{tab:scene_reconstruction} {\bf (a)~AUCs of surface coverage for several NBV selection methods during active view planning experiments on several 3D scenes from our dataset.} Results are averaged on several trajectories. {\bf (b)~Surface coverage curves for the different NBV selection methods, averaged on all 13 3D scenes.}
    Despite being trained only on centered ShapeNet 3D models, the second module of \method is able to generalize to complex scenes when predicting the visibility scores.}
\end{figure}

Despite being trained only on ShapeNet 3D models, our method is able to compute meaningful paths around the structures and consistently reach satisfying global coverage.
Since we focused on building a metric for NBV computation, this experiment is simple and does not implement any further prior from common path planning strategies: For instance, we do not compute distant NBV, nor optimal paths to move from a position to a distant NBV with respect to our volumetric mapping. There is no doubt the model could benefit from further strategies inspired by the path planning literature.

\subsection{Ablation Study}

We now provide an ablation study for both modules of our full model: The prediction of occupancy probability, and the computation of coverage gain from predictions about visibility gains. To evaluate both modules separately, we compared their performance on their training critera: MSE for occupancy probability, and softmax followed by KL-Div on a dense set of cameras for coverage gain. The values are reported in Figure~\ref{fig:ablation_study}. \vincentrmk{Do we still have a supp mat? :}We provide extensive details and analysis in the appendix.

\begin{figure}%
    \centering
    \begin{tabular}{@{}cc@{}}
  \hspace{-0.5cm}\scalebox{0.7}{ 
  \begin{tabular}{@{}lcc@{}}
    \toprule
    % \multicolumn{1}{c}{} & \multicolumn{5}{c}{3D scene} \\
    % \cmidrule(r){1-2}
    Architecture & Mean Squared Error & IoU\\
    % Method & Arena & Statue & Bridge & Castle & Monument & Mean \\
    \midrule
    Base Architecture & 0.0397 & 0.843\\
    No Neighborhood Feature & 0.0816 & 0.702\\
    With Camera History & \textbf{0.0386} & \textbf{0.844}\\
    \bottomrule
  \end{tabular}
 }
 &
\scalebox{0.7}{ 
  \begin{tabular}{@{}lc@{}}
    \toprule
    % \multicolumn{1}{c}{} & \multicolumn{5}{c}{3D scene} \\
    % \cmidrule(r){1-2}
    Architecture & Kullback-Leibler Divergence \\
    % Method & Arena & Statue & Bridge & Castle & Monument & Mean \\
    \midrule
    Base Architecture & \textbf{0.00039}\\
    No prediction $\probfield$ & 0.00051 \\
    No Camera History & 0.00045 \\
    \bottomrule
  \end{tabular}
 }\\[7mm]
(a) Occupancy probability & (b) Coverage gain \\
    \end{tabular}
    \caption{\label{fig:ablation_study} {\bf (a) Comparison of Mean Squared Error and IoU for variations of our occupancy probability prediction model. (b) Comparison of Kullback-Leibler divergence for variations of our coverage gain computation model.} Thanks to the volumetric prediction $\probfield$, the full model predicts a better distribution of coverage gains on the whole space as it is indicated by KL divergence, which is convenient for full path planning in a 3D scene.}%
\end{figure}

\textbf{Occupancy probability.}  Apart from the base architecture presented in Figure~\ref{fig:occupancy_probability_prediction}, we trained the first module under two additional settings. First, we removed the multi-scale neighborhood features computed by downsampling the partial point cloud $P_H$, and trained the module to predict an occupancy probability value $\probfield$ only from global feature $g(P_H)$ and encoding $f(x)$. As anticipated, the lack of neighborhood features not only prevents the model from an efficient scaling to large scenes, but also  causes a huge loss in performance. We also trained a slightly more complex version of the module by feeding it the camera history harmonics $h_H(x)$ as an additional feature. It appears this helps the model to increase its confidence and make better predictions, but the gains are quite marginal.

\textbf{Visibility gain.}  We trained two variations of our second module. First, we completely removed the geometric prediction: We use directly the surface points as proxy points, mapped with an occupancy probability equal to 1. As a consequence, the model suffers from a significant loss in performance to compute coverage gain from its predicted visibility gain functions. Thus, we can confirm that the volumetric framework improves the performance of \method. This result is remarkable, since surface representations usually make better NBV predictions for dense reconstruction of detailed objects in literature. On the contrary, we show that our formalism is able to achieve higher performance by leveraging a volumetric representation with a high resolution deep implicit function $\probfield$. We trained a second variation without the spherical mappings $h_H$ of camera history. We verified that such additional features increase performance, especially at the end of a 3D reconstruction.

\section{Limitations and Conclusion}

Beyond the prediction of the Next Best View, our method \method is able to evaluate the value of any camera pose given a current estimate of the scene volume. We demonstrated this ability with a simple path planning algorithm that looks for the next pose around the current pose iteratively. However:
\begin{itemize}[leftmargin=5mm]
    \item Like current methods, \method relies on a depth sensor and assumes that the captured depth maps are perfect. In practice, such a sensor is not necessarily available and can also be noisy. It would be very interesting to extend the approach to color cameras.
    \item Also like current methods, it is limited to the prediction of the camera pose for the next step. This is greedy, non-optimal as multiple poses are required anyway for a complete 3D reconstruction. Thanks to its scalability, we believe that our method could be extended to the prediction of the ``Next Best Path'', where future camera poses would be predicted jointly for their complementarity.
\end{itemize}

\section*{Acknowledgements}

This work was granted access to the HPC resources of IDRIS under the allocation 2022-AD011013387 made by GENCI. We thank Renaud Marlet, Elliot Vincent, Romain Loiseau, and Tom Monnier for inspiring discussions and valuable feedback.

\bibliographystyle{plainnat}
\bibliography{camera_ready}

%\begin{comment}
\section*{Checklist}

% %%% BEGIN INSTRUCTIONS %%%
% The checklist follows the references.  Please
% read the checklist guidelines carefully for information on how to answer these
% questions.  For each question, change the default \answerTODO{} to \answerYes{},
% \answerNo{}, or \answerNA{}.  You are strongly encouraged to include a {\bf
% justification to your answer}, either by referencing the appropriate section of
% your paper or providing a brief inline description.  For example:
% \begin{itemize}
%   \item Did you include the license to the code and datasets? \answerYes{See Section~\ref{gen_inst}.}
%   \item Did you include the license to the code and datasets? \answerNo{The code and the data are proprietary.}
%   \item Did you include the license to the code and datasets? \answerNA{}
% \end{itemize}
% Please do not modify the questions and only use the provided macros for your
% answers.  Note that the Checklist section does not count towards the page
% limit.  In your paper, please delete this instructions block and only keep the
% Checklist section heading above along with the questions/answers below.
% %%% END INSTRUCTIONS %%%

\begin{enumerate}

\item For all authors...
\begin{enumerate}
  \item Do the main claims made in the abstract and introduction accurately reflect the paper's contributions and scope?
    \answerYes{}{}
  \item Did you describe the limitations of your work?
    \answerYes{} See the last section.
  \item Did you discuss any potential negative societal impacts of your work?
    \answerYes{} A discussion can be found in the supplementary material.
  \item Have you read the ethics review guidelines and ensured that your paper conforms to them?
    \answerYes{}
\end{enumerate}

\item If you are including theoretical results...
\begin{enumerate}
  \item Did you state the full set of assumptions of all theoretical results?
    \answerYes{}{}
        \item Did you include complete proofs of all theoretical results?
    \answerYes{}
\end{enumerate}

\item If you ran experiments...
\begin{enumerate}
  \item Did you include the code, data, and instructions needed to reproduce the main experimental results (either in the supplemental material or as a URL)?
    \answerNo{} Not yet, but we will release our source code and dataset.
  \item Did you specify all the training details (e.g., data splits, hyperparameters, how they were chosen)?
    \answerYes{} In the supplementary material.
        \item Did you report error bars (e.g., with respect to the random seed after running experiments multiple times)?
    \answerNo{}
        \item Did you include the total amount of compute and the type of resources used (e.g., type of GPUs, internal cluster, or cloud provider)?
    \answerYes{} See supplementary material.
\end{enumerate}

\item If you are using existing assets (e.g., code, data, models) or curating/releasing new assets...
\begin{enumerate}
  \item If your work uses existing assets, did you cite the creators?
    \answerYes{}
  \item Did you mention the license of the assets?
    \answerYes{}
  \item Did you include any new assets either in the supplemental material or as a URL?
    \answerNo{}
  \item Did you discuss whether and how consent was obtained from people whose data you're using/curating?
    \answerYes{}
  \item Did you discuss whether the data you are using/curating contains personally identifiable information or offensive content?
    \answerNo{} Our data do not have such issue.
\end{enumerate}

\item If you used crowdsourcing or conducted research with human subjects...
\begin{enumerate}
  \item Did you include the full text of instructions given to participants and screenshots, if applicable?
    \answerNA{}
  \item Did you describe any potential participant risks, with links to Institutional Review Board (IRB) approvals, if applicable?
    \answerNA{}
  \item Did you include the estimated hourly wage paid to participants and the total amount spent on participant compensation?
    \answerNA{}
\end{enumerate}

\end{enumerate}
%\end{comment}

\newpage
\appendix

{\LARGE 
\begin{center}
    \textbf{Appendix}
\end{center} 
\hrule
\vspace{3mm}}

In this appendix, we provide the following elements:
\begin{enumerate}
    \item A complete proof of our main theorem~\ref{thm:volumetric_approximation_of_coverage} in Section~\ref{sec:supp_mat_1_proof}.
    \item Further details about the training process, as well as the implementation of \method, in Section~\ref{sec:supp_mat_2_training}.
    \item Further details about the experiments we conducted in Section~\ref{sec:supp_mat_3_experiments}.
    % \item A video showing several examples of path planning with \method for reconstruction of large 3D scenes.
\end{enumerate}
Code and data will be available on our dedicated webpage: \url{https://github.com/Anttwo/SCONE}.

\section{Proof of Theorem~\ref{thm:volumetric_approximation_of_coverage}}
\label{sec:supp_mat_1_proof}

In this section, we keep all notations introduced in the main paper and give further details about the proof of Theorem \ref{thm:volumetric_approximation_of_coverage}.

In particular, we start by listing all assumptions we make about the scene's geometry and discuss whether they are appropriate or not to real-case scenarios. Then, we present the complete derivations leading to Theorem~\ref{thm:volumetric_approximation_of_coverage}.

\subsection{Main Assumptions}

%in practice walls and floors are scanned as open surfaces - not that important since we can just give them an arbitrary thickness; $\partial\chi$ needs to be $C^k$ for $k\geq 2$; the SDF needs to be twice continuously differentiable on the spherical neighborhood; the function to integrate should be absolutely integrable but this last point is trivial since we use a bounded measurable function on a bounded surface. reference: Gilbarg, 1983

\paragraph{Assumption 1: The surface of the scene consists in a finite collection of bounded watertight surfaces.} 
To derive Theorem~\ref{thm:volumetric_approximation_of_coverage}, we need the scene to be represented as a closed volume $\chi$, and its surface as the boundary $\partial\chi$ of $\chi$, which is trivially true for a finite collection of bounded watertight surfaces. Such an assumption is realistic since a real 3D object has a non-zero volume and is actually made of watertight surfaces. Note that, in practice, the floor is scanned as a plane surface by a depth sensor. The same observation applies to walls in indoor scenes, as well as all surfaces that delimit unreachable parts of 3D space (except for the inside of objects). However, these surfaces still delimit volumes in reality, and an arbitrary thickness can be predicted for them, as it won't change the result of the scan.

Following such assumptions, the Signed Distance Function~(SDF) $f_\chi:\IR^3 \rightarrow \IR$ of the volume $\chi$ is defined as:
\begin{equation}
    \label{sdf}
    f_\chi(x) =
    \begin{cases}
            d(x, \partial\chi) & \text{if } x\in \chi\\
            -d(x, \partial\chi) & \text{if } x \notin \chi\> ,
        \end{cases}
\end{equation}
where $d(x, \partial\chi) = \inf_{y\in\partial\chi}\|x-y\|_2$ for any $x\in \chi$.

As we never reconstruct infinitely large scenes in practice, we also consider the scene to be bounded. As a consequence, both $\chi$ and $\partial\chi$ are compact as they are bounded, closed subsets of $\IR^3$.

\paragraph{Assumption 2: The boundary $\partial \chi$ is $C^2$.}

We suppose the scene's surface is regular enough to be, locally, the image of a $C^2$ embedding of $\IR^2$ into $\IR^3$. This is a non-trivial assumption since in theory, the surface of every object with sharp edges or corners is not a $C^2$ boundary. However, in practice most of real objects with edges (like a door, a table, etc.) are smoother than they appear, and can be accurately represented by a $C^2$ boundary. In general, we believe that a $C^2$ approximation of most realistic non-smooth surface is enough to compute a meaningful NBV anyway, as shown in Figure~\ref{fig:smooth_approximation}.

\begin{figure}%
    \centering
    \subfloat{{\includegraphics[width=5cm, trim={10cm 0 10cm 0}, clip]{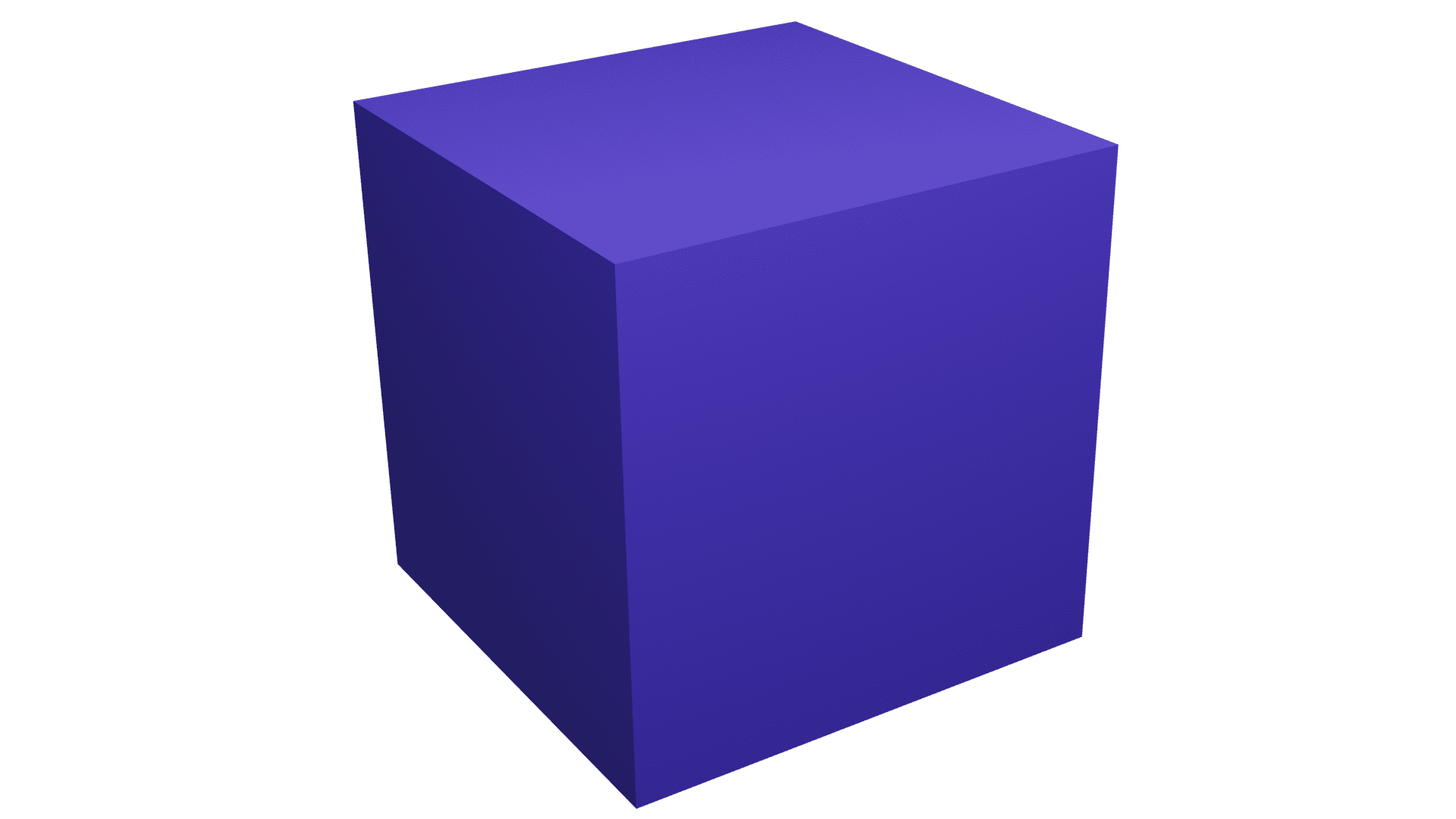}} }%
    \subfloat{{\includegraphics[width=5cm, trim={10cm 0 10cm 0}, clip]{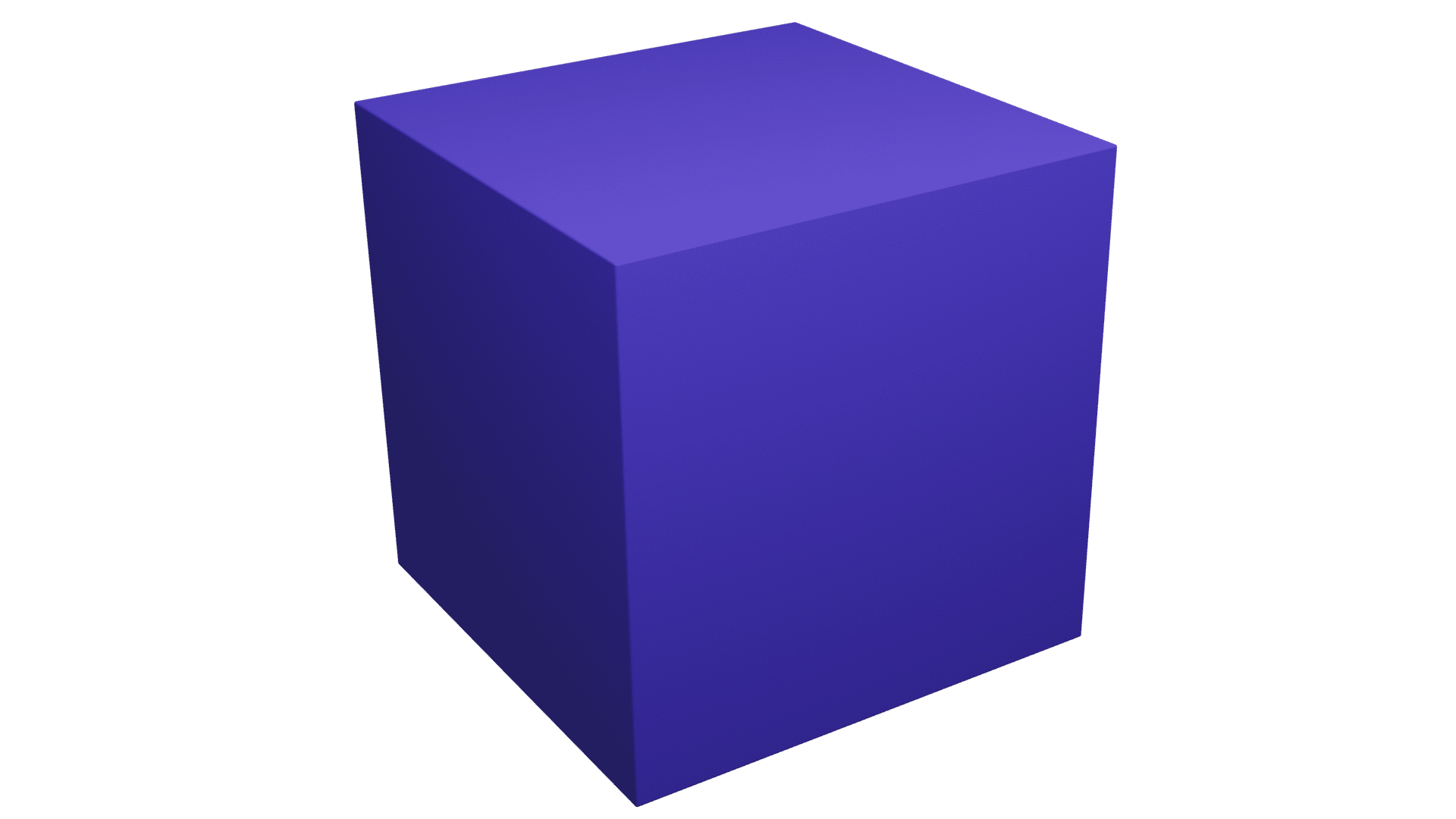} }}\\%
    \subfloat{{\includegraphics[width=5cm, trim={20cm 0 0cm 0}, clip]{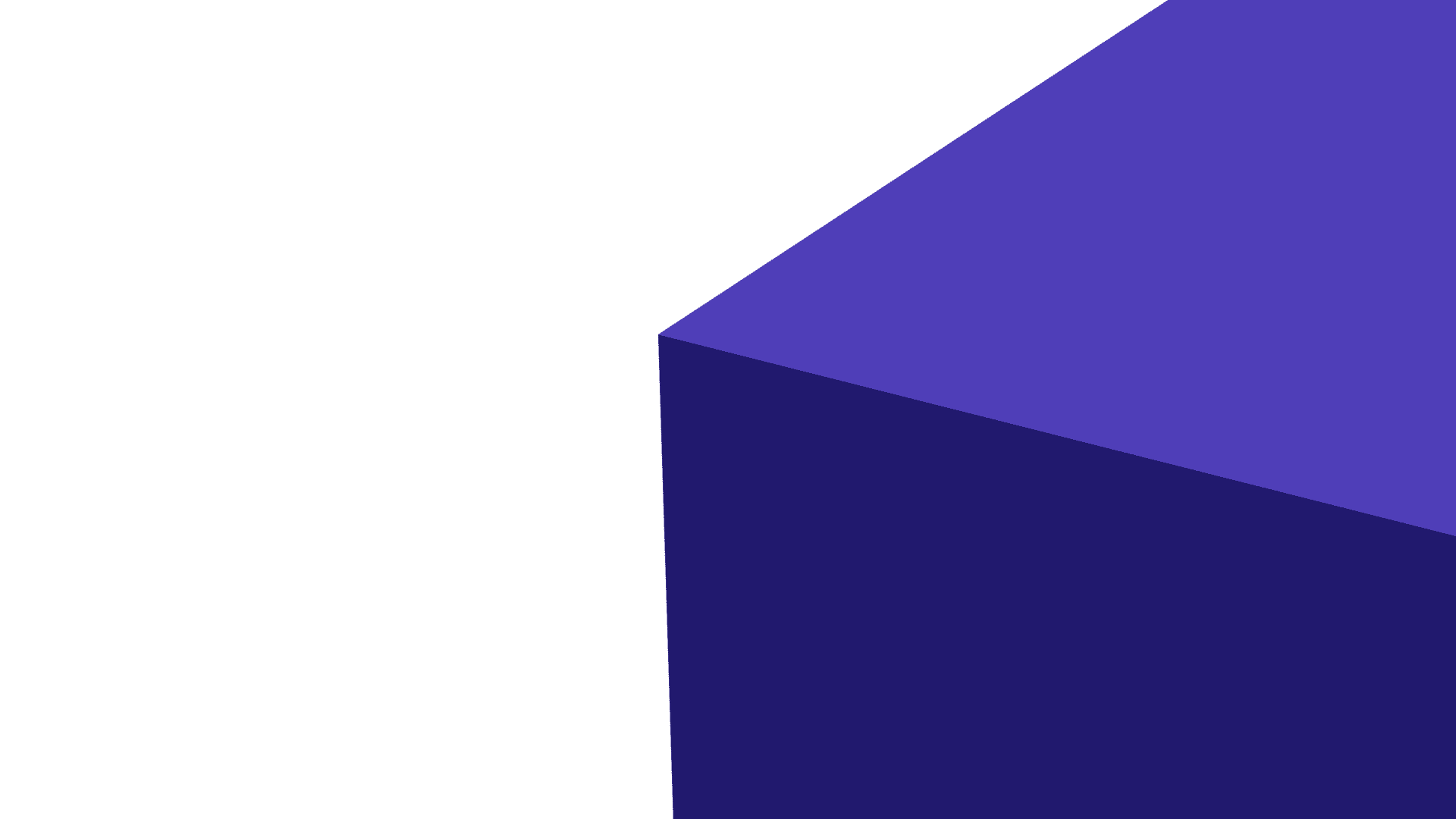} }}%
    \subfloat{{\includegraphics[width=5cm, trim={20cm 0 0cm 0}, clip]{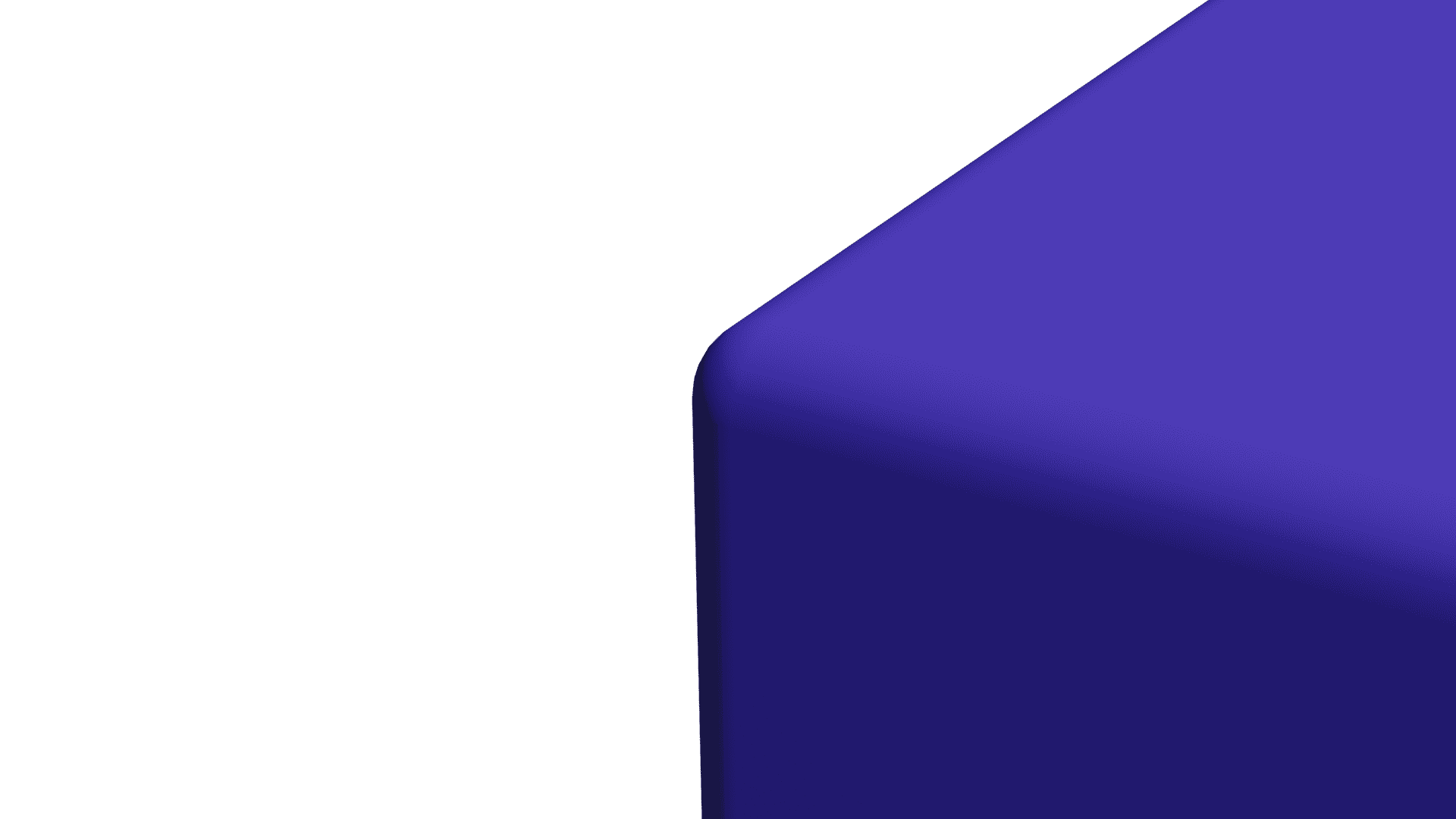} }}%
    \caption{\label{fig:smooth_approximation} {\bf Example of a $C^2$ approximation of a surface}. On the left, a cube mesh with 6 faces and 12 sharp edges. On the right, a smoother $C^2$ approximation of the same cube.}
\end{figure}

Following such assumptions, according to \cite{gilbarg}, there exists a quantity $\mu_0>0$ such that the SDF $f_\chi$ is twice continuously differentiable on the spherical neighborhood $T(\partial\chi, \mu_0) := \left\{ p \in \IR^3 \ | \ \exists x \in \partial\chi, \|x-p\|_2 < \mu_0 \right\}$.

In particular, for all $x_0\in \partial\chi$, the function $f_\chi$ satisfies $\nabla f_\chi(x_0) = N(x_0)$ where $N$ is the inward normal vector field \cite{gilbarg}. Moreover, for all absolutely integrable function $g:T(\partial\chi, \mu_0)\rightarrow \IR$, and for all $\mu<\mu_0$, the following identity is satisfied~\cite{gilbarg}:
\begin{equation}
    \label{eqn:gilbarg_link_surface_volume}
        \int_{T(\partial\chi, \mu)}  g(x) \text{d}x
    = \int_{\partial\chi} \int_{-\mu}^\mu g(x_0 + \lambda N(x_0))\,\det(I-\lambda W_{x_0})\,\text{d}\lambda\,\text{d}x_0,
\end{equation}
where $W_{x_0}$ is the Weingarten map at $x_0$, that is, the Hessian of $f_\chi$, which is continuous on $T(\partial\chi, \mu_0)$ since $f_\chi$ is $C^2$ on $T(\partial\chi, \mu_0)$. The integral on the left of Equation~\ref{eqn:gilbarg_link_surface_volume} is a volumetric integral on a spherical neighborhood, while the integral on the right is a surface integral on $\partial\chi$.

Please note that Equation~\ref{eqn:gilbarg_link_surface_volume} from \cite{gilbarg} actually only applies to tubular neighborhoods, which are specific neighborhoods of submanifolds resembling the normal bundle. However, since the spherical neighborhoods of $C^2$ watertight surfaces also are tubular neighborhoods, we prefer to use the simpler definition of spherical neighborhoods.

\subsection{Proof of Theorem \ref{thm:volumetric_approximation_of_coverage}}

To derive the theorem, we first prove the following lemma.

\begin{lemma}
\label{thm:lemma}
Under the previous assumptions, there exists $\lambda_0>0$ such that, for all $\lambda<\lambda_0$ and $x_0\in \partial\chi$, the point $x_0 + \lambda N(x_0)$ is located inside the volume, \ie, $f_\chi(x_0+\lambda N(x_0)) \geq 0$.
\end{lemma}

\begin{proof}
Following our main assumptions, $f_\chi$ is $C^2$ on the spherical neighborhood $T(\partial\chi, \mu_0)$. We define $\Gamma$ as the closure of $T(\partial\chi, \frac{\mu_0}{2})$. As a bounded, closed subset of $T(\partial\chi, \mu_0)$, $\Gamma$ is compact. Since the Hessian of $f_\chi$ is continuous on $\Gamma$, it is a bounded function on $\Gamma$. We denote by $B>0$ a bound verifying $\|W_{x_0}\|_{op} \leq B$ for all $x_0 \in \Gamma$, with $\|\cdot\|_{op}$ the operator norm.

Then, for a given surface point $x_0 \in \partial\chi$, we are interested in the sign of $f_\chi(x_0 + \lambda N(x_0))$ for $\lambda>0$. In this regard, we use the Lagrangian form of the Taylor expansion of $f_\chi$ at $x_0$: For all $\lambda < \frac{\mu_0}{2}$, there exists $w_{x_0, \lambda}$ that lies on the line connecting $x_0$ and $x_0 + \lambda N(x_0)$ such that
\begin{equation}
    f_\chi(x_0 + \lambda N(x_0)) = f_\chi(x_0) + \lambda \nabla f_\chi(x_0)^T N(x_0) + \frac{\lambda^2}{2} N(x_0)^T W_{w_{x_0, \lambda}} N(x_0).
\end{equation}
Since $x_0$ is located on the surface, $f(x_0)=0$. As we mentioned in the previous subsection, $\nabla f_\chi(x_0) = N(x_0)$ so that $\nabla f_\chi(x_0)^T N(x_0) = \|N(x_0)\|_2^2 = 1$. 

Moreover, $w_{x_0, \lambda} \in \Gamma$. We apply Cauchy-Schwarz inequality and deduce, by definition of the operator norm: 
\begin{equation*}
    \begin{split}
        |N(x_0)^T W_{w_{x_0, \lambda}} N(x_0)| & \leq \|N(x_0)\|_2 \cdot \| W_{w_{x_0, \lambda}}\|_{op} \cdot \|N(x_0)\|_2 \\
        |N(x_0)^T W_{w_{x_0, \lambda}} N(x_0)| & \leq B.
    \end{split}
\end{equation*}
Consequently,
\begin{equation*}
    \begin{split}
        f_\chi(x_0 + \lambda N(x_0)) & = \lambda + \frac{\lambda^2}{2} N(x_0)^T W_{w_{x_0, \lambda}} N(x_0)\\
        & \geq \lambda - \frac{\lambda^2}{2} B.
    \end{split}
\end{equation*}
The polynomial function $\lambda\mapsto \lambda - \frac{\lambda^2}{2}B$ being positive on $(0,\frac{2}{B})$, we define $\lambda_0=\min(\frac{2}{B}, \frac{\mu_0}{2})$ and conclude that $f_\chi(x_0 + \lambda N(x_0)) > 0$ for all positive $\lambda<\lambda_0$.
\end{proof}

Finally, we prove the main theorem.
\begin{T1}
    Under the previous regularity assumptions on the volume $\chi$ of the scene and its surface $\partial\chi$, there exist $\mu_0 > 0$ and $M>0$ such that for all $\mu < \mu_0$, and any camera $c\in \calC$:
    \begin{equation}
         \left| \frac{1}{|\chi|_V}\int_{\chi}  g_c^H(\mu;x) \text{d}x -  \mu \frac{|\partial\chi|_S}{|\chi|_V} G_H(c) \right| \leq M \mu^2 \> ,
    \end{equation}
    where $|\chi|_V$ is the volume of $\chi$.% and $|\partial\chi|_S$ is the area of its surface $\partial\chi$.
\end{T1}
\begin{proof}
We keep the notations introduced in the previous lemma and, if needed, we update the value of $\mu_0$ so that $\mu_0 \leq \lambda_0$. Then, we start back from equation (3) of the main paper, where we introduced a new visibility gain function $\newvis_c^H$ to adapt the definition of the former visibility gain $\vis_c^H$ on spherical neighborhoods. For any $0<\mu<\mu_0$, we define this function as:
\begin{equation}
        \newvis_{c}^H(\mu; x) = 
        \begin{cases}
            1 & \text{if } \exists x_0 \in \partial\chi, \lambda <\mu \text{ such that } x = x_0 + \lambda N(x_0) \text{ and } \vis_c^H(x_0) = 1,\\
            0 & \text{otherwise} \> ,
        \end{cases}
\end{equation}
where $N$ is the inward normal vector field, which is well defined according to our main regularity assumptions. Once again, for $c\in\calC, \mu < \mu_0$, $g_c^{H}$ is obviously bounded on the bounded subset $T(\partial\chi, \mu_0)$. Thus, $g_c^{H}(\mu,\cdot)$ is absolutely integrable on $T(\partial\chi, \mu_0)$. Consequently, we apply the formula \ref{eqn:gilbarg_link_surface_volume} from \cite{gilbarg} and find that, for all $c \in \calC$ and $\mu < \mu_0$:
\begin{equation}
    \label{eqn:link_volume_surface_appendix}
    \int_{T(\partial\chi, \mu)}  g_c^H(\mu; x) \text{d}x
    = \int_{\partial\chi} \int_{-\mu}^\mu g_c^H(\mu; x_0 + \lambda N(x_0))\,\det(I-\lambda W_{x_0})\,\text{d}\lambda\,\text{d}x_0.
\end{equation}
Now, we are going to show that $\det(I - \lambda W_{x_0}) = 1 + \lambda b(\lambda, x_0)$ where $b$ is a bounded function on the compact space $[-\mu_0, \mu_0] \times \partial \chi$. To this end, we could either develop directly the determinant or use results about characteristic polynomials to speed up the proof. 

We choose the second approach: since $W_{x_0}$ is a square $3\times 3$ matrix, there exist polynomial functions $f_i:\IR^{3\times 3} \rightarrow \IR, i=1,...,3$ such that, for all $x_0\in\partial\chi$, the characteristic polynomial of $W_{x_0}$ verifies 
$$\det(XI_3 - W_{x_0}) = X^3 + f_1(W_{x_0}) X^2 + f_2(W_{x_0}) X + f_3(W_{x_0}).$$
Consequently, for all $x_0\in\partial\chi$ and $0<\lambda<\mu_0$,
\begin{equation*}
    \begin{split}
        \det(I_3 - \lambda W_{x_0}) & = \lambda^3 \det(\frac{1}{\lambda}I_3 - W_{x_0})\\ 
        & = \lambda^3 \left(\frac{1}{\lambda^3} + f_1(W_{x_0}) \frac{1}{\lambda^2} + f_2(W_{x_0}) \frac{1}{\lambda} + f_3(W_{x_0}) \right)\\
        & = 1 + \lambda \left(f_1(W_{x_0}) + f_2(W_{x_0}) \lambda + f_3(W_{x_0}) \lambda^2\right).
    \end{split}
\end{equation*}
We define $b:(\lambda, x_0) \mapsto f_1(W_{x_0}) + f_2(W_{x_0}) \lambda + f_3(W_{x_0}) \lambda^2$. Since the functions $f_i$ are polynomial and $x_0\mapsto W_{x_0}$ is continuous on $\partial\chi$, we deduce that $b$ is bounded on $[-\mu_0, \mu_0] \times \partial \chi$ as a continuous function defined on a compact space. We denote by $M>0$ a bound verifying $|b(\lambda, x_0)| \leq M$ for all $(\lambda, x_0) \in [-\mu_0, \mu_0] \times \partial \chi$.

Subsequently, for all $x_0\in \partial\chi$, we have by definition $g_c^H(\mu;x_0 + \lambda N(x_0)) = g_c^H(\mu;x_0) = \vis_c(x_0)$ when $0 \leq \lambda < \mu$, and $g_c^H(\mu;x_0 + \lambda N(x_0))= 0$ when $-\mu < \lambda < 0$. By rewriting equation \ref{eqn:link_volume_surface_appendix}, it follows that, for every $0<\mu<\mu_0$:
\begin{equation}
\begin{split}
    \int_{T(\partial\chi, \mu)}  g_c^H(\mu;x) \text{d}x & = \int_{\partial\chi} \int_{0}^\mu g_c^H(\mu;x_0)(1 + \lambda b(\lambda, x_0)) \,\text{d}\lambda \,\text{d}x_0\\
    & = \mu \int_{\partial\chi} g_c^H(\mu;x_0) \,\text{d}x_0 + \int_{\partial\chi} \int_{0}^\mu \lambda g_c^H(\mu;x_0) b(\lambda, x_0)\,\text{d}\lambda \,\text{d}x_0\\
    & = \mu |\partial\chi|_S G_H(c) + \int_{\partial\chi} \int_{0}^\mu \lambda g_c^H(\mu;x_0) b(\lambda, x_0)\,\text{d}\lambda \,\text{d}x_0.
\end{split}
\end{equation}

Since the volume is supposed to be opaque, function $g_c^H(\mu;\cdot)$ is equal to 0 for every point outside $T(\partial\chi, \mu)$. Moreover, according to lemma \ref{thm:lemma}, for all $x_0\in \partial \chi, \mu < \mu_0$, the point $x_0 + \mu N(x_0)$ is located inside the volume $\chi$, such that $\int_{T(\partial\chi, \mu)}  g_c^H(\mu;x)\,\text{d}x = \int_{\chi}  g_c^H(\mu;x)\,\text{d}x$.

Since $|g_c^H(\mu;\cdot)|\leq 1$ for all $c\in\mathcal{C}$ and $\mu<\mu_0$:
\begin{equation}
    \begin{split}
        \left| \int_{\chi}  g_c^H(\mu;x) \text{d}x - \mu |\partial\chi|_S G_H(c) \right| & \leq \int_{\partial\chi} \int_{0}^\mu \lambda \left| g_c^H(\mu;x_0) b(\lambda, x_0)\right| \,\text{d}\lambda \,\text{d}x_0 \\
        & \leq \int_{\partial\chi} \int_{0}^\mu \lambda \left|b(\lambda, x_0)\right| \,\text{d}\lambda \,\text{d}x_0\\
        & \leq \int_{\partial\chi} \int_{0}^\mu \lambda M \,\text{d}\lambda \,\text{d}x_0 \\
        & \leq \frac{M\mu^2}{2} \int_{\partial\chi} \,\text{d}x_0 \\
        & \leq \frac{M\mu^2}{2} |\partial\chi|_S = |\chi|_V \cdot M'\mu^2 
    \end{split}
\end{equation}
with $M' = \frac{M}{2} \frac{|\partial\chi|_S}{|\chi|_V}$.
\end{proof}
\section{Training}
\label{sec:supp_mat_2_training}

\antoine{
Motivated by the literature about NBV reconstruction for objects, we trained our model on the simple case of single, centered object reconstruction (see \nameref{sec:supp_mat_single_object}), with meshes from ShapeNetCore v1~\cite{shapenet2015}, following the same training, validation and test distributions than \cite{zeng-icirs20-pcnbv}. Indeed, we wanted our model to learn general geometric prior and compute coverage gain predictions on various shapes. Since \method relies on neighborhood geometric features as well as proxy point-level information to compute various metrics, we made the assumption that training only on an object dataset should not prevent the model from scaling to 3D scenes. We confirmed this hypothesis with an experiment on large 3D structures with free motion on a 5D grid, detailed in \nameref{sec:supp_mat_view_planning_in_scene}.

In theory, the full model could be trained directly in an end-to-end fashion with a single loss. However, in practice we schedule the process, and train the two modules consecutively. In particular, we first train the geometric prediction module alone so that predicted occupancy mappings become meaningful; then, we use the geometric predictions as an input to train the second module, and make it learn to predict accurate visibility gains. Indeed, training the entire model from scratch makes convergence difficult since the second module of the model learns to predict visibility gains from meaningless geometric reconstructions at the beginning.

\subsection{Training the geometric prediction module}
\label{sec:training_occupancy}
Inspired by \cite{lars-18-occupancy_network, xu-19-disn}, we follow a simple pipeline to train the geometric prediction model: For each mesh $M_i$ in a batch $(M_1,...,M_{N_{\text{mesh}}})$, we sample a random number $n_i$ of camera poses, capture $n_i$ depth maps from these positions, and compute the corresponding partial point cloud $P_i$. Then, for each mesh we sample $N_X$ proxy points $X^{(i)} = (x_1^{(i)},...,x_{N_X}^{(i)})$ and compute their predicted occupancy probability $\probfield(P_i;\ x_j^{(i)})$. Finally, we use an MSE loss to compare the predictions with ground truth occupancy values. 

More exactly, the loss $\calL_{\text{occ}}$ used to train the geometric prediction module of \method is the following:
\begin{equation}
    \calL_{\text{occ}} = \frac{1}{N_{\text{mesh}}} 
    \sum_{i=1}^{N_{\text{mesh}}} \frac{1}{N_X}
    \sum_{j=1}^{N_X} \|\probfield(P_i;\ x_j^{(i)}) - \occfield(x_j^{(i)})\|^2.
\end{equation}

\subsection{Training the coverage prediction module}

Once the geometric prediction model converges, we start to train the visibility gain $I_H$ prediction module. To this end, we first select a batch of meshes. For each mesh, we capture a random number of initial depth map observations (up to 10) with a ray-casting renderer and apply \method to predict visibility gain function coordinates in spherical harmonics. Since all cameras are sampled on a sphere in this setup, they share the same proxy points \vincentrmk{proxy point?} in their field of view and we can directly and efficiently compute coverage gains for a dense set of cameras on the sphere from predicted visibility gains, in a single forward pass.

Since our predicted volumetric gains are not equal to the true coverage gains but only proportional, we cannot compare directly the predicted values to ground-truth coverage gains but have to normalize them first. Moreover, we want \method not only to select the right NBVs consistently but also predict an accurate distribution of coverage gains in the volume for further path planning applications. Consequently, we consider the predicted $I_H(c)$ as a distribution on $c$ and compare it to the ground truth coverage gains using the Kullback-Leibler divergence $D_{KL}$ after a softmax normalization. 

More exactly, the training process starts in the same way as the previous one: For each mesh $M_i$ in a batch $(M_1,...,M_{N_{\text{mesh}}})$, we sample a random number $n_i$ of camera poses (which correspond to the history $H_i$), capture $n_i$ depth maps from these positions, and compute the corresponding partial point cloud $P_i$. For each mesh we sample $N_X$ proxy points $X^{(i)} = (x_1^{(i)},...,x_{N_X}^{(i)})$ and compute their predicted occupancy probability $\probfield(P_i;\ x_j^{(i)})$. 

Then, for each mesh we sample a subset of proxy points according to their occupancy probability; we concatenate these proxy points with their occupancy values, compute the camera history features $(h_{H_i}(x_1^{(i)}), ..., h_{H_i}(x_{N_X}^{(i)}))$ and feed these inputs to the second module of \method to predict the spherical mappings $\phi_l^m$ of visibility gains with a single forward pass. 

Next, we sample a dense set of $N_{\text{cam}}$ camera poses $\calC^{(i)} = (c_1^{(i)}, ..., c_{N_{\text{cam}}}^{(i)})$ on a sphere around the object and compute the predicted coverage gains $(I_{H_i}(c_1^{(i)}), ..., I_{H_i}(c_{N_{\text{cam}}}^{(i)}))$ for all cameras using Monte-Carlo integration on the predicted spherical mappings of visibility gains. Finally, we compute the following loss $\calL_{\text{cov}}$ to supervise training for the second module of \method:
\begin{equation}
    \begin{split}
        \calL_{\text{cov}} & = \frac{1}{N_{\text{mesh}}} \sum_{i=1}^{N_{\text{mesh}}} D_{KL}(\softmax(G_{H_i})\ ||\ \softmax(I_{H_i}))\\
        & = \frac{1}{N_{\text{mesh}}} \sum_{i=1}^{N_{\text{mesh}}} 
        \sum_{j=1}^{N_{\text{cam}}} 
        s_j^{(i)} \log\left( \frac{s_j^{(i)}}{\hat{s}_j^{(i)}} \right)
        ,
    \end{split}
\end{equation}

where the softmax normalization is defined as follows:
$$s_j^{(i)} = \frac{\exp\left(G_{H_i}(c_j^{(i)})\right)}{\sum_{k=1}^{N_{\text{cam}}}\exp\left(G_{H_i}(c_k^{(i)})\right)} \quad \text{and} \quad \hat{s}_j^{(i)} = \frac{\exp\left(I_{H_i}(c_j^{(i)})\right)}{\sum_{k=1}^{N_{\text{cam}}}\exp\left(I_{H_i}(c_k^{(i)})\right)}.$$
%
%Since we want \method not only to select the right NBVs consistently but predict an accurate distribution of coverage gains in the volume for further path planning applications, and because our predicted volumetric gains are not equal to the true coverage gains but only proportional, we consider the predicted $I_H(c)$ as a distribution on $c$ and compare it to the ground truth coverage gains (computed with Monte Carlo integration on manifold following Eq.~\eqref{eqn:MC_integration_surface} \antoinermk{add a word about threshold $\epsilon$}) using the Kullback-Leibler divergence after applying softmax. In other words, we sample a dense set of $N_{cam}$ camera poses $\calC^{(i)} = (c_1^{(i)}, ..., c_{N_{cam}}^{(i)})$

We supervise on the final coverage value (\ie, the value of the Monte-Carlo integration) rather than per-point visibility gains, since it would be heavier to compute and would require further assumptions. Indeed, we would have to compute ground truth visibility gains for probabilistic points that may not be in $\chi$, which needs further hypothesis to handle properly. On the contrary, supervising on the integral value prevents us from handling directly this problem: it is up to the model to process information as a volumetric integral, identify which point could be in the spherical neighborhood in an implicit manner and learn a meaningful per-point visibility gain metric. We can stick to our volumetric framework and have no need to make further assumptions on geometry or directly extract surfaces from our predicted probability field, which could lead to inaccurate results.%\textcolor{red}{Add a sentence to explain how ground truth surface coverage is computed}

Overall, our model has 3,650,657 parameters: 2,257,769 for the occupancy probability prediction module, and 1,392,888 for the visibility gain prediction module. Each module is trained 20~hours on 4~GPUs Nvidia~Tesla~V100~SXM2~16~Go. %with a learning rate $\lambda = 10^{-4}$. 
We use a linear warmup strategy on the learning rate $\lambda$ to make training more stable: $\lambda$ linearly increases from 0 to its default value $\lambda=10^{-4}$ during the first 1,000 iterations. The learning rate is ultimately decreased to $10^{-5}$ after 60,000 iterations for further improvement.

All experiments were run on a single GPU Nvidia~GeForce~GTX~1080~Ti. Further details about the model's architecture are given in Section~\nameref{sec:implementation}.}

% \textcolor{red}{Maybe talk about the analytic study on cubes.}\\

\subsection{Implementation details}
\label{sec:implementation}

We implemented and trained our model with PyTorch~\cite{paszke_pytorch}. In particular, we used ray-casting renderers from PyTorch3D~\cite{ravi_pytorch3d} to generate and use depth maps as inputs to our model.

To compute all spherical mappings involved in our method, we use the orthonormal basis of spherical harmonics with rank lower or equal to 7, which makes a total of 64 harmonics. In this regard, both the camera history features $h_H$ and the predicted visibility gain functions $\phi$ are mapped as vectors with $64$ coordinates.

Code and data will be available on our dedicated webpage: \url{https://github.com/Anttwo/SCONE}.
%\textcolor{red}{Add, if possible, figures of the architecture with the dimension of inner features, the number of different scales for neighborhood features, etc.}

\section{Experiments}
\label{sec:supp_mat_3_experiments}

In this section, we give further details about our experiments: how the dataset is constructed, the hyperparameters involved in the prediction, and the evaluation metrics. We also provide additional results.

\subsection{Next Best View for Single Object Reconstruction}
\label{sec:supp_mat_single_object}

We first compare the performance of our model to the state of the art on a subset of the ShapeNet dataset~\cite{shapenet2015}.

\paragraph{Dataset.}
We follow the protocol of \cite{zeng-icirs20-pcnbv}: Using the train/validation/test split of \cite{yuan-corr18-pcn} for ShapeNet dataset~\cite{shapenet2015}, we sample 4,000 training meshes from 8 categories of objects (Airplane, Cabinet, Car, Chair, Lamp, Sofa, Table, Vessel), as well as 400 validation meshes and 400 test meshes from the same categories. Note that the full test set used in \cite{yuan-corr18-pcn} contains 1200 meshes. To make sure we provide a fair comparison with the state of the art, we sampled 10 different test subsets of 400 meshes among all 1200 meshes, ran the experiment 10 times and averaged the metrics. The results were really close from one subset to another, and our model systematically provided better coverage values than other methods in literature.%\textcolor{red}{(Add results with variance/error bar)}

Following \cite{zeng-icirs20-pcnbv}, we reconducted the same experiment with subsets of 400 test meshes from 8 categories unseen during training (Bed, Bench, Bookshelf, Bus, Guitar, Motorbike, Pistol, Skateboard). 

\paragraph{Prediction.}%\textcolor{red}{Input PC: 1024 points reprojected in 3D from each input DM (same as PC-NBV \cite{zeng-icirs20-pcnbv}). Since it is a small scale object, we do not use a grid of point clouds but a single cell containing all points. Moreover, $\chi_c = \chi$ for all camera poses in this particular setting, so at each round we predict coverage gain in all direction in a single forward pass.} 
All camera poses are sampled on a sphere around the object. For each mesh, we follow the protocol of \cite{zeng-icirs20-pcnbv} and capture a first depth map from a random camera pose, which is reprojected in 3D as a cloud of 1024 points. Then, we use our model to predict coverage gain for all camera poses, and select the NBV as the pose with the highest gain. We concatenate the partial point cloud captured from the new position to the previous one, and we iterate the process until we have 10 views of the object. Note that we reconstruct a small scale object which is entirely contained in the field of view of every camera, \ie, $\chi_c=\chi$ for all camera poses $c$. %Therefore, we do not partition the volume into a grid but a single cell containing all points
Therefore, at each round we compute coverage gains for all camera poses in a single forward pass thanks to the spherical mappings of visibility gain functions.

\paragraph{Evaluation metric.}
%\textcolor{red}{GT computation: 16,384 GT surface points (same as \cite{zeng-icirs20-pcnbv}, coverage computed in advance for each view. Epsilon chosen for coverage computation is the same as \cite{zeng-icirs20-pcnbv}.}
Once again, we follow the protocol of \cite{zeng-icirs20-pcnbv} to compute ground-truth coverage scores and evaluate our method. For each mesh, we uniformly sample a cloud $P_0$ of 16,384 points on the surface, which represent the ground truth surface. Then, the total coverage $C(P_H)$ of a partial point cloud $P_H$ obtained by merging depth maps captured from camera poses in $H$ is computed as follows:
\begin{equation}
    \label{eqn:gt_total_surface_coverage}
    C(P_H) = \frac{1}{|P_0|} \sum_{p_0\in P_0} U(\epsilon - \min_{p\in P_H}\|p_0-p\|_2),
\end{equation}
where $\epsilon$ is a distance threshold. $C(P_H)$ simply evaluates the number of points in $P_0$ that have at least a neighbor in $P_H$ closer than $\epsilon$. For the experiment, we used the same threshold as \cite{zeng-icirs20-pcnbv}, \ie, $\epsilon=0.00707m$.

\paragraph{Results.}
We provide additional results for this experiment. In particular, figure \ref{fig:shapenet_convergence_speed} illustrates the evolution of surface coverage during reconstruction for several methods.

\begin{figure}
  \centering
  \includegraphics[width=7cm]{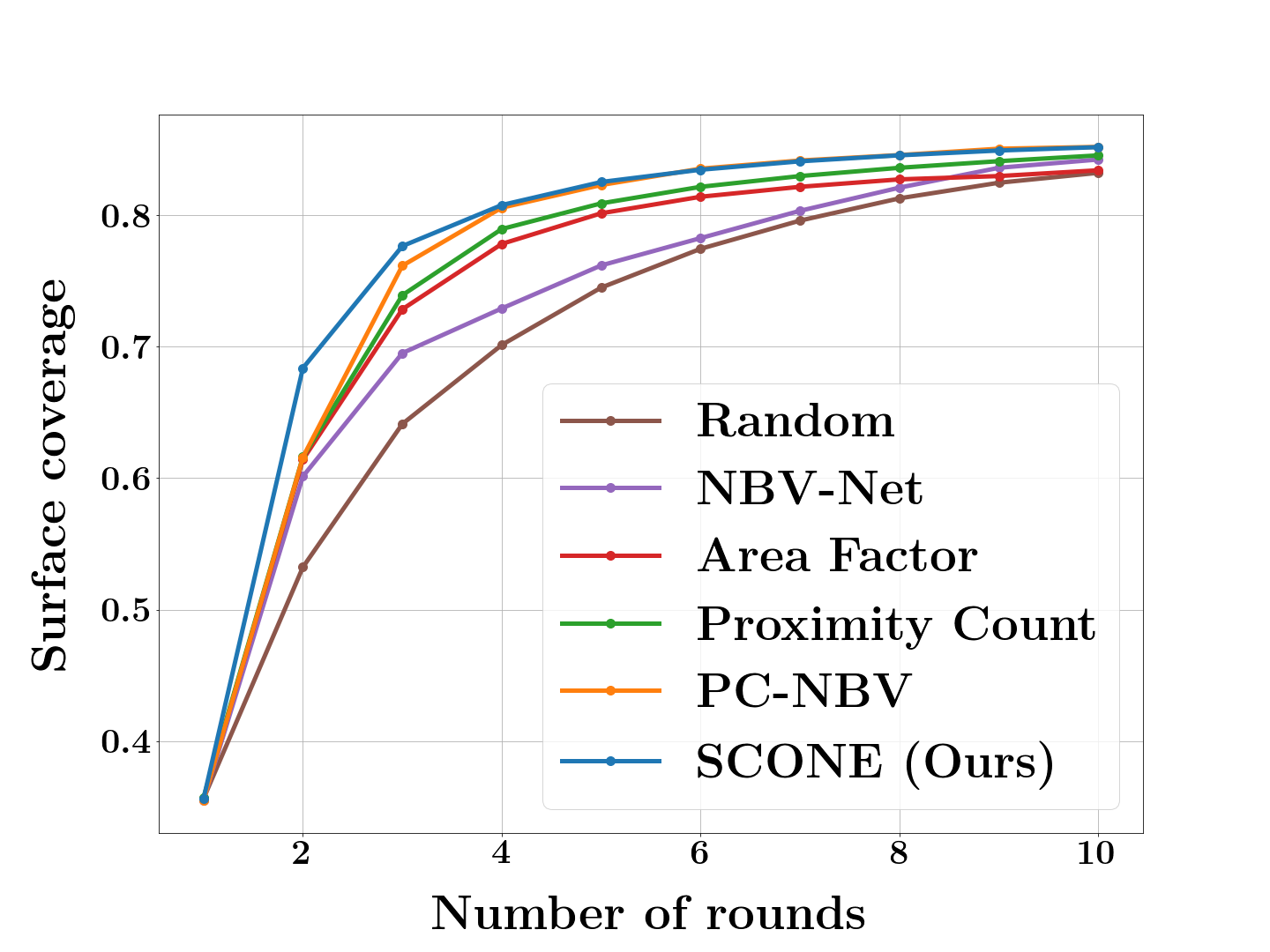}
  \caption{\label{fig:shapenet_convergence_speed} {\bf Convergence speed of the covered surface by several methods~(NBV-Net~\cite{Mendoza-2019}, Area Factor~\cite{vasquez-14-volumetric-nbv}, Proximity Count~\cite{delmerico-18-acomparison}, PC-NBV~\cite{zeng-icirs20-pcnbv}, ours) on a subset of ShapeNet dataset~\cite{shapenet2015}}. Our \method method has the highest reconstruction performance in terms of AUC, and performs significantly better when only a small number of views are available. For this experiment, we constrain the camera to stay on a sphere centered on the objects in order to compare with previous methods.}
\end{figure}

\subsection{Active View Planning in a 3D Scene}
\label{sec:supp_mat_view_planning_in_scene}

\paragraph{Dataset.}
%\textcolor{red}{Explain how data was processed: we downloaded 3D models, defined bounding boxes, and partitioned each bounding box in a grid of cells depending on their dimension.}
Since, to the best of our knowledge, we propose the first supervised Deep Learning method for free 6D motion of the camera, we created a dataset made of 13 large-scale scenes under the CC License for quantitative evaluation (3D models courtesy of Brian Trepanier, Andrea Spognetta, and 3D Interiors; all models were downloaded on the website Sketchfab). For each scene, we defined a bounding box delimiting the main structure to reconstruct. %: For instance, for the Eiffel Tower 3D scene, we defined a bounding box that delimits the tower.

To scale the geometric prediction module of \method to large 3D scenes, we partition the scene in 3D cells depending on the dimensions of the bounding box: the larger the bounding box, the more cells in the scene. Then, each point belonging to the partial point cloud $P_H$ gathered by the sensor (computed by reprojecting all depth maps in 3D) is stored in the corresponding cell. To predict the occupancy probability of each proxy point, we only use the points located in neighboring cells. Therefore, the growth of the global partial point cloud does not influence the computation time of the occupancy probability prediction: We avoid unnecessary computation but keep meaningful predictions since our geometric prediction model relies on local neighborhood features.

Note that we do not reproject every point of the depth maps to compute the partial point cloud $P_H$. Indeed, we introduce a distance threshold $\epsilon$ to avoid unnecessary large amounts of points in the cells. Let $p$ be a point belonging to a new depth map computed by the sensor; then, we identify the cell in which $p$ is located, and add $p$ to the point cloud $P_H$ if and only if for all points $p_0$ in the cell, $\|p-p_0\|_2 > \epsilon_\text{cloud}$. Therefore, the threshold $\epsilon_\text{cloud}$ is a parameter that influence the resolution of the reconstructed point cloud $P_H$.

\paragraph{Prediction.}
% \textcolor{red}{Give details about discretization of camera poses on a 5D grid, the planning strategy, and the distance factor that aims to adapt the model to large scenes (but precise that it is not needed for angular etc.). Also, give information about the inputs: we still a ray-tracing renderer that generates depth maps, and we reproject them into point clouds depending on the resolution of the grid (each cell has a maximum number of points it can contain, as well as a resolution threshold)}
To evaluate the scalability of our model to large environments as well as free camera motion in 3D space, we use once again a ray-casting renderer and follow the protocol described in section Active View Planning in a 3D Scene of the main paper. However, to predict surface coverage gain in a large 3D scene, we slightly modify the computation of per-point visibility gains. 

Indeed, note that the density of points gathered by a depth sensor like a LiDAR decreases with the distance to the surface, as well as the angle between the surface normal and the direction of observation. Since \method was trained on small scale objects with camera poses sampled on a sphere, our model learned to predict optimized angle for observation, but did not learn to predict optimized distance. Actually, the predicted visibility gain of proxy points sampled in space should reflect the variations in LiDAR density, and decrease inversely with the squared distance to the camera. 

To this end, given a camera pose $c$, we penalize the distance by multiplying the predicted visibility gain of a proxy point $x\in\chi_c$ by a factor $\frac{1}{\eta + \|x-c_{pos}\|_2^2}$. We apply the same strategy to the baseline \method-Entropy, and multiply the Shannon Entropy of a proxy point by the same factor.

\paragraph{Evaluation metric.}
%\textcolor{red}{GT coverage is computed in a similar way. The threshold $\epsilon$ matches the resolution of the cells.}
To compute ground-truth total surface coverage for evaluation, we use the same approach than the previous experiment detailed in subsection \ref{sec:supp_mat_single_object}. In particular, we sample 100,000 points on the surface to obtain a ground-truth cloud $P_0$. Moreover, we use $\epsilon_\text{cloud}$, the parameter that defines the resolution of the reconstructed point cloud $P_H$, to compute the total surface coverage following equation \ref{eqn:gt_total_surface_coverage}.

\begin{figure}%
    \centering
    \subfloat[\centering Average on all 13 scenes]{{\includegraphics[width=4.1cm, trim={0 0 0 0}, clip]{images/supp_mat/all_path_planning_curve.png} }}%
    \subfloat[\centering Dunnottar Castle]{{\includegraphics[width=4.1cm, trim={0 0 0 0}, clip]{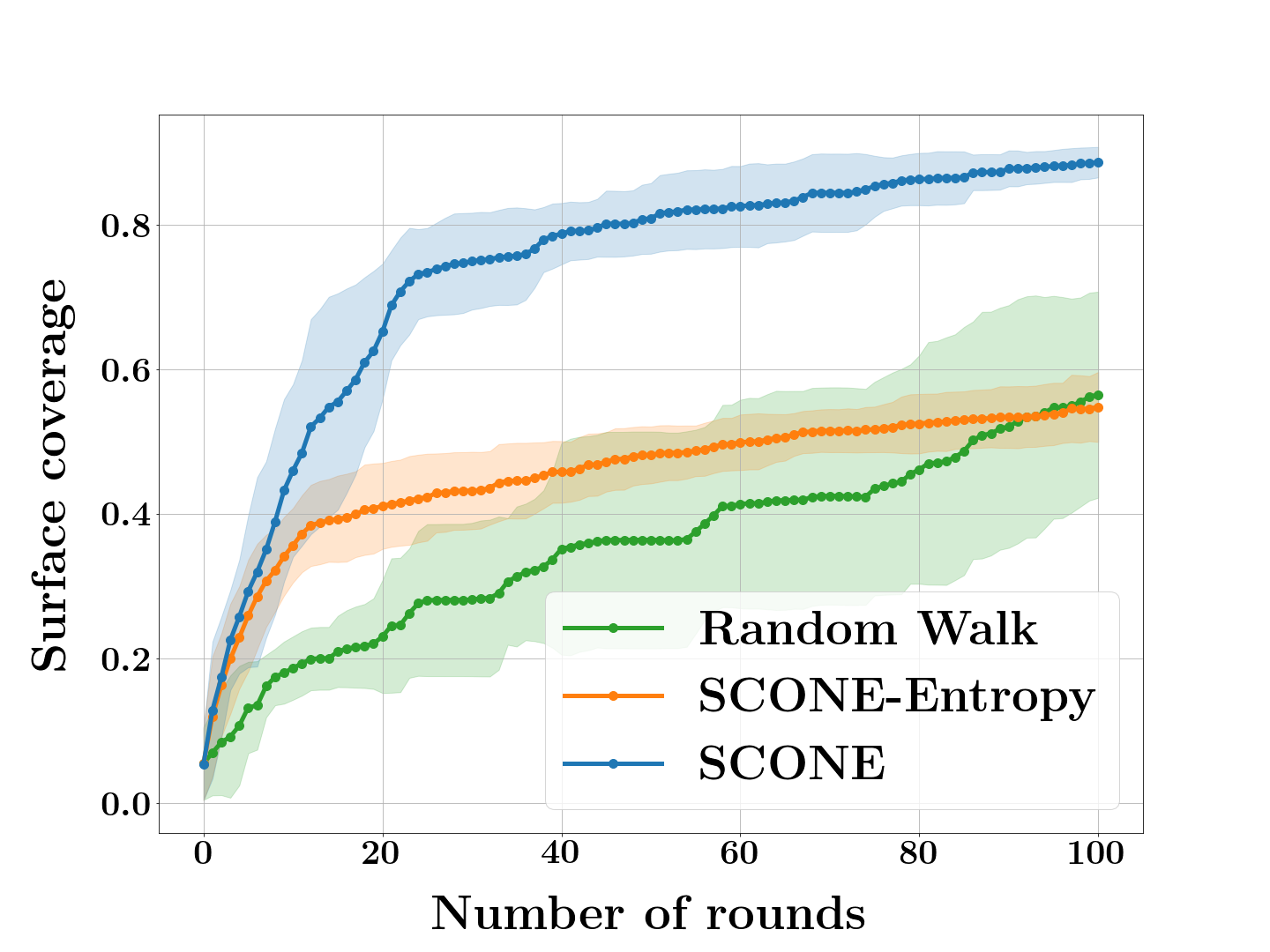} }}%
    \subfloat[\centering Manhattan Bridge]{{\includegraphics[width=4.1cm, trim={0 0 0 0}, clip]{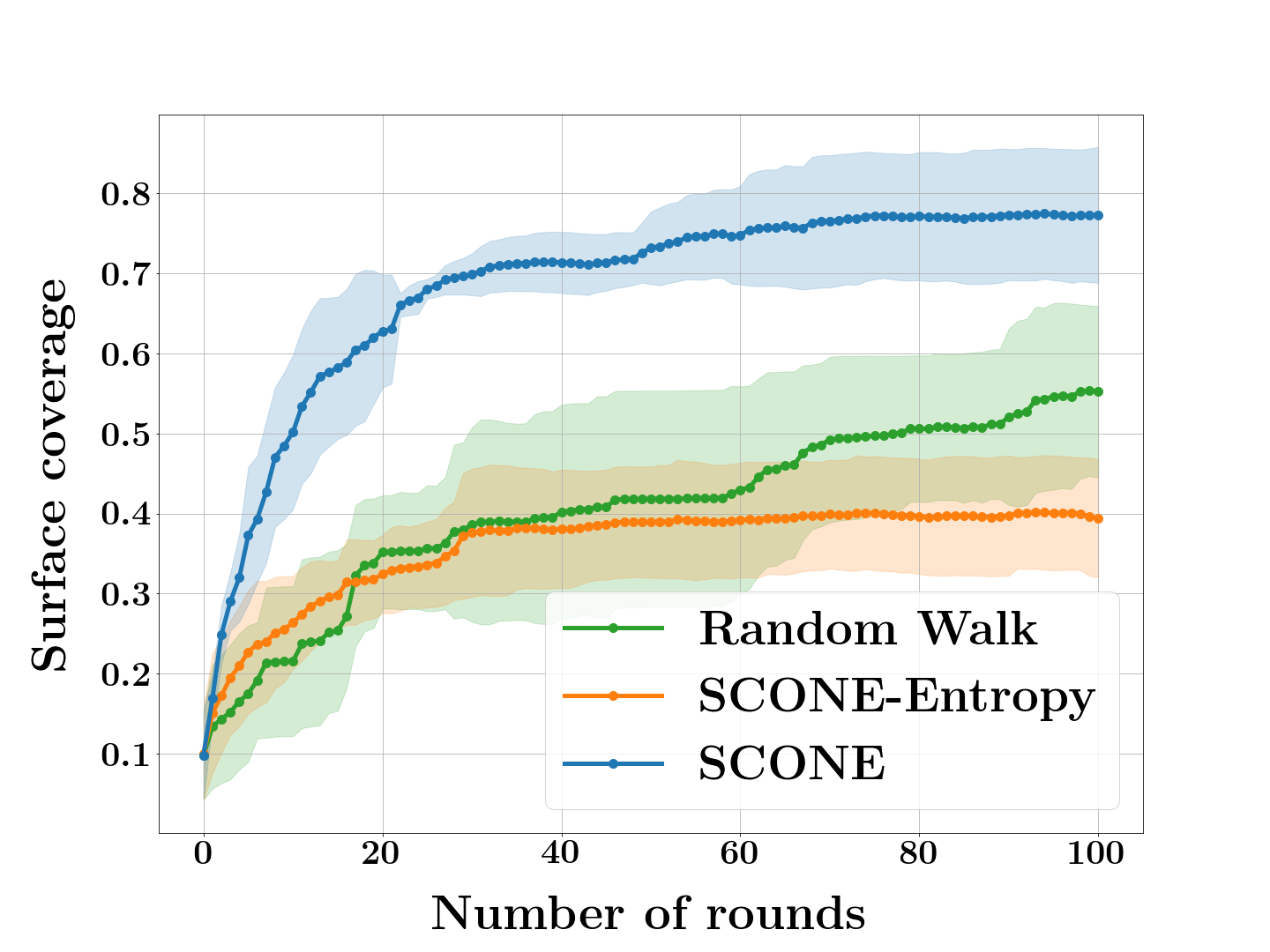} }}\\%
    
    \subfloat[\centering Alhambra Palace]{{\includegraphics[width=4.1cm, trim={0 0 0 0}, clip]{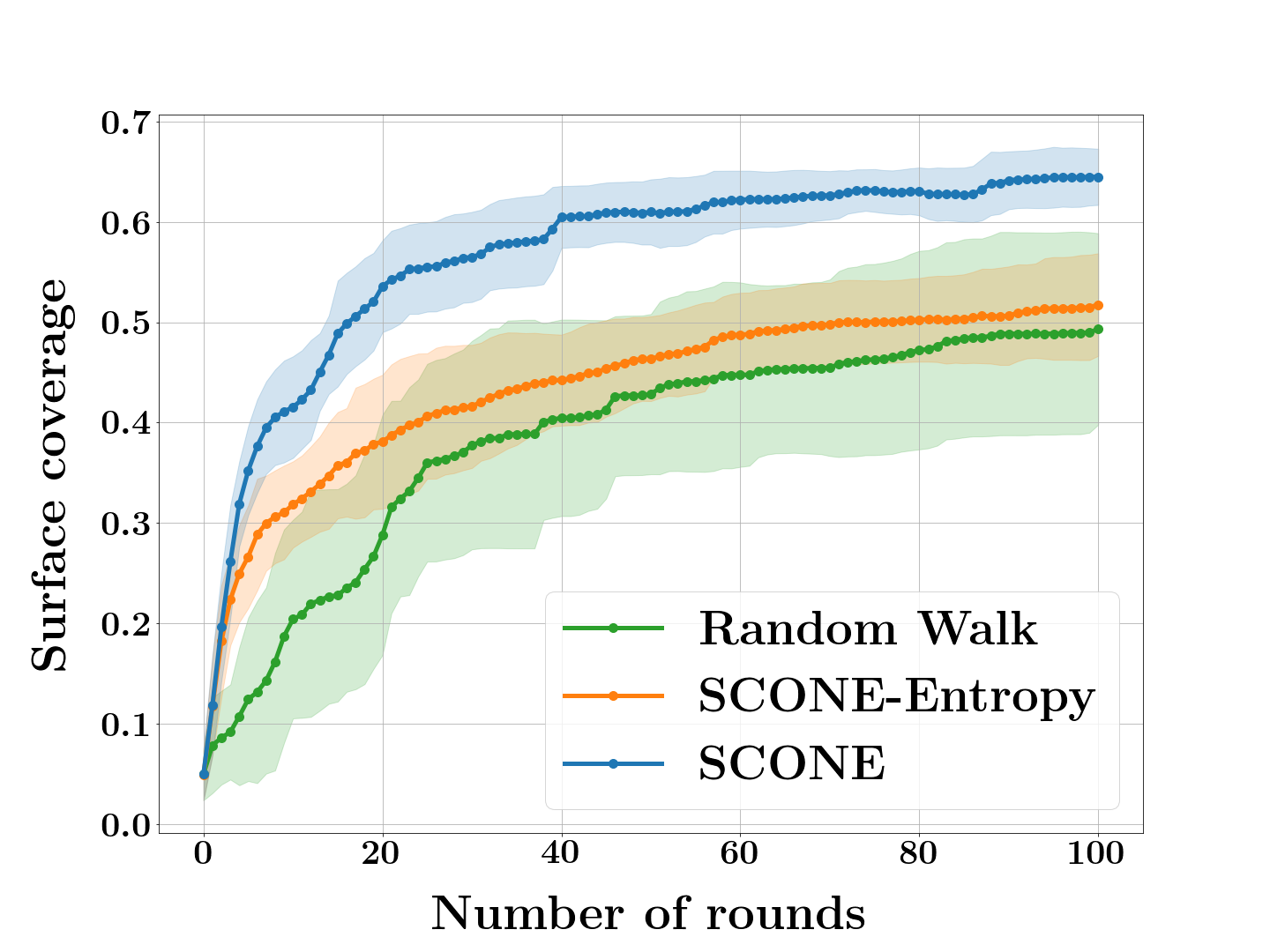} }}%
    \subfloat[\centering Leaning Tower, Pisa]{{\includegraphics[width=4.1cm, trim={0 0 0 0}, clip]{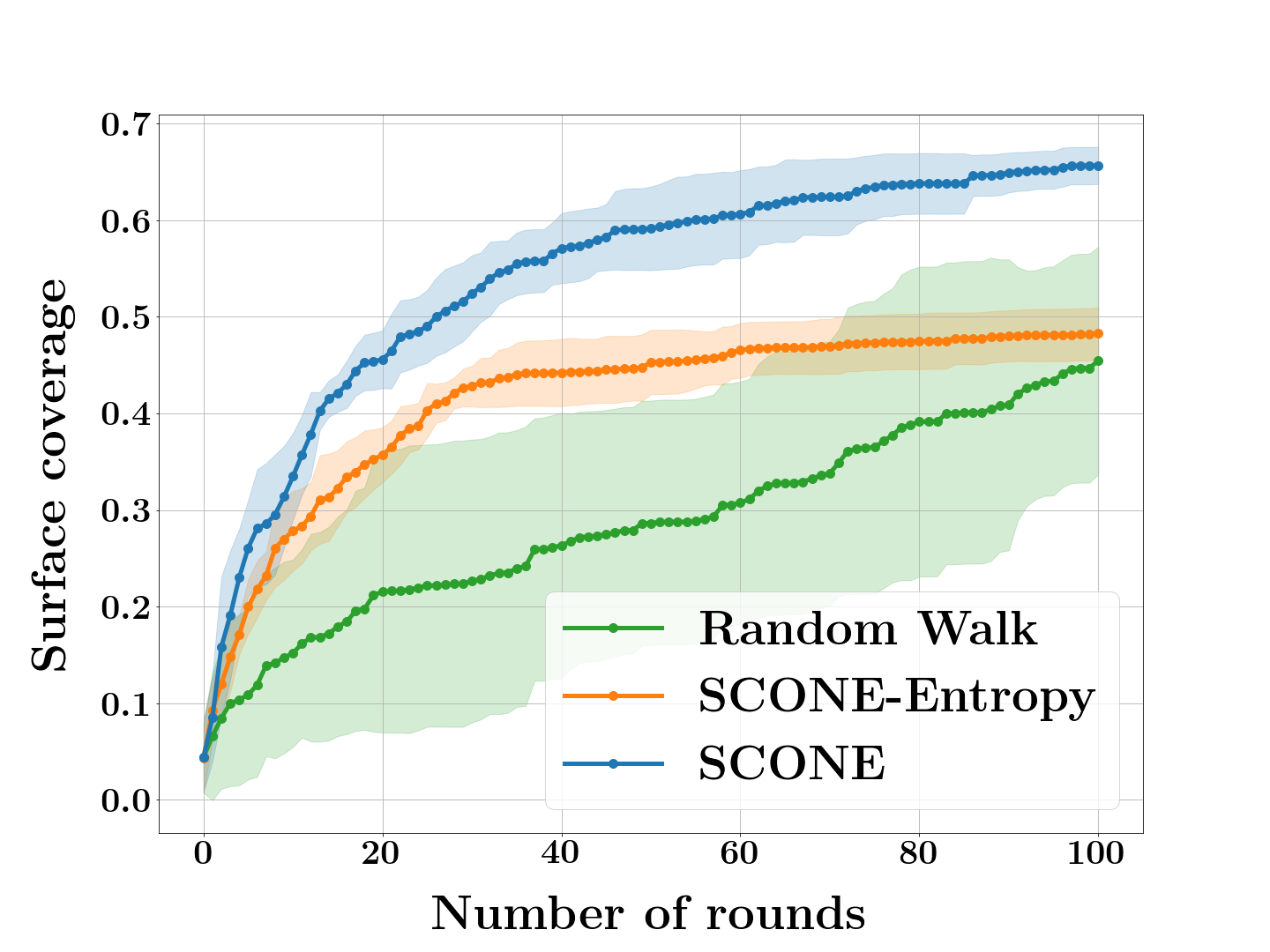} }}%
    \subfloat[\centering Neuschwanstein Castle]{{\includegraphics[width=4.1cm, trim={0 0 0 0}, clip]{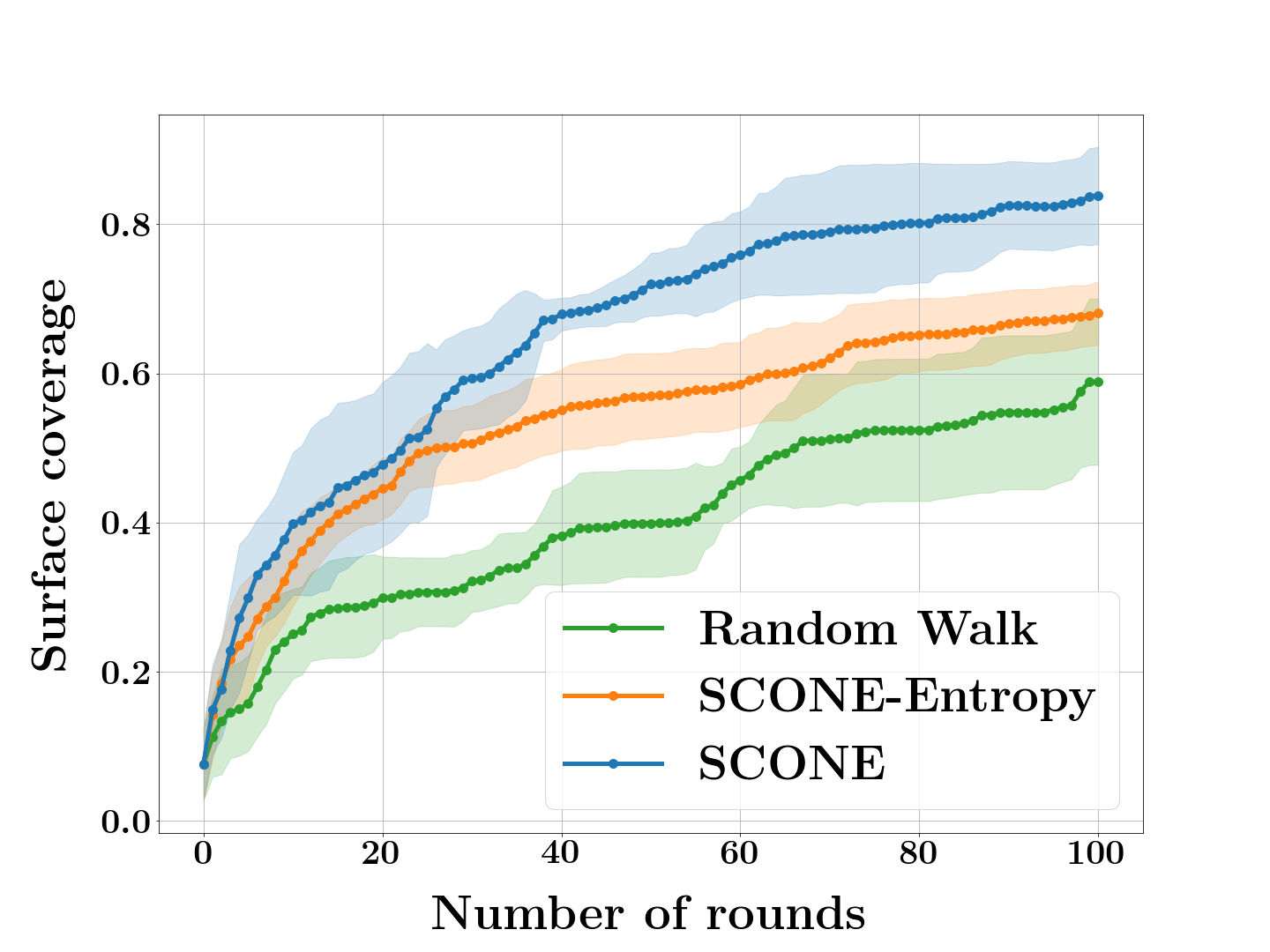} }}\\%
    
    \subfloat[\centering Colosseum]{{\includegraphics[width=4.1cm, trim={0 0 0 0}, clip]{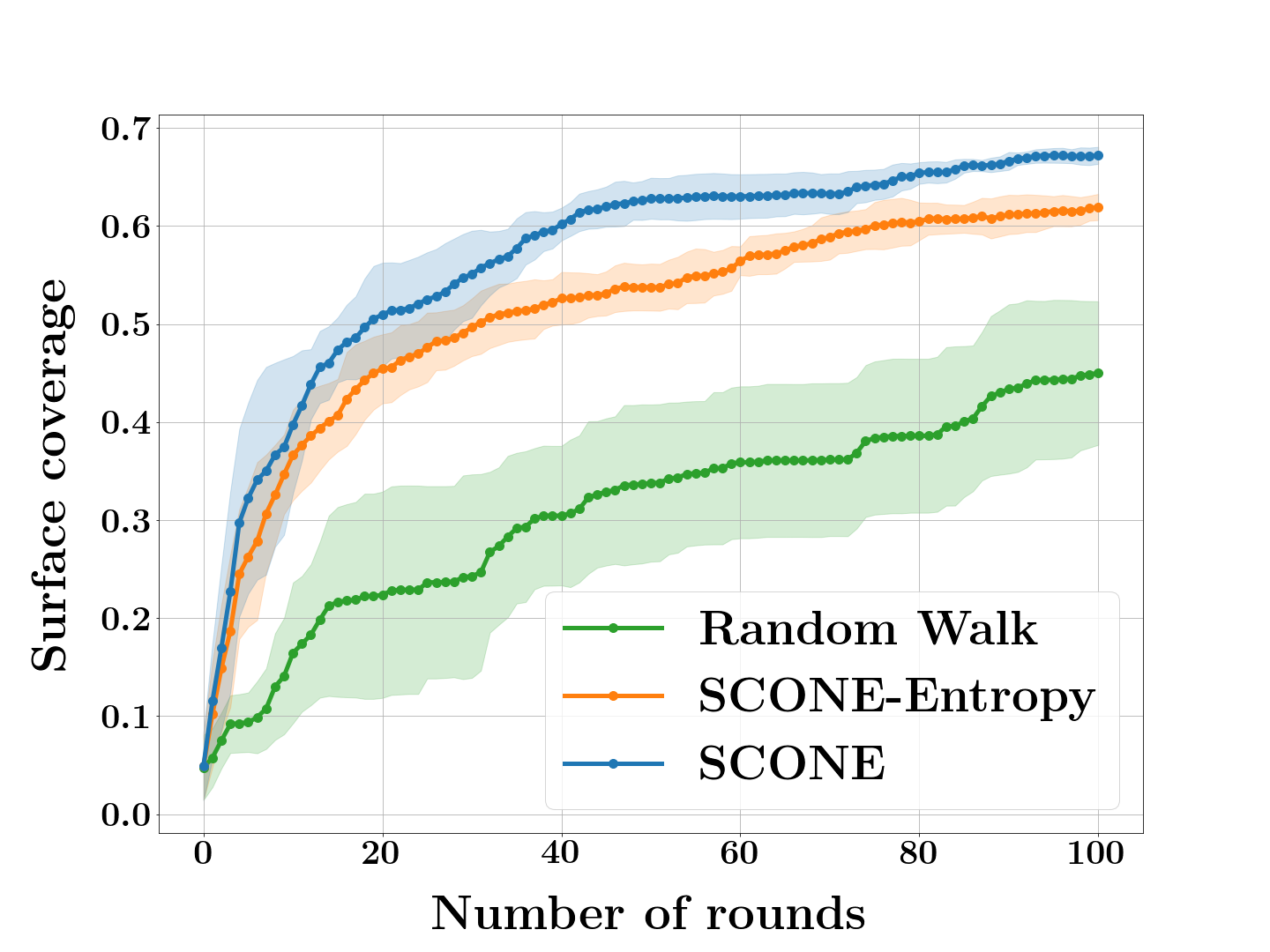} }}%
    \subfloat[\centering Eiffel Tower]{{\includegraphics[width=4.1cm, trim={0 0 0 0}, clip]{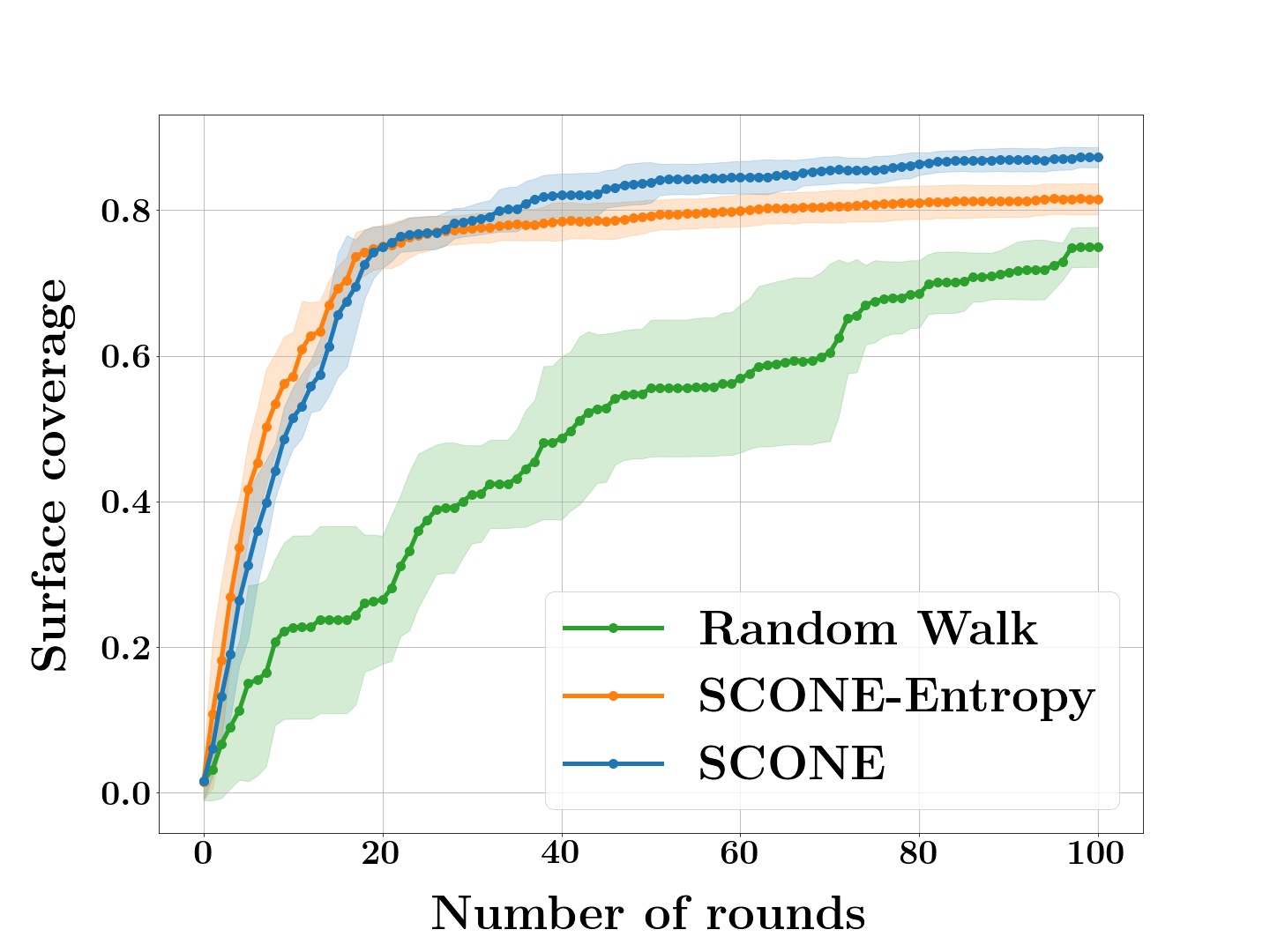} }}%
    \subfloat[\centering Fushimi Castle]{{\includegraphics[width=4.1cm, trim={0 0 0 0}, clip]{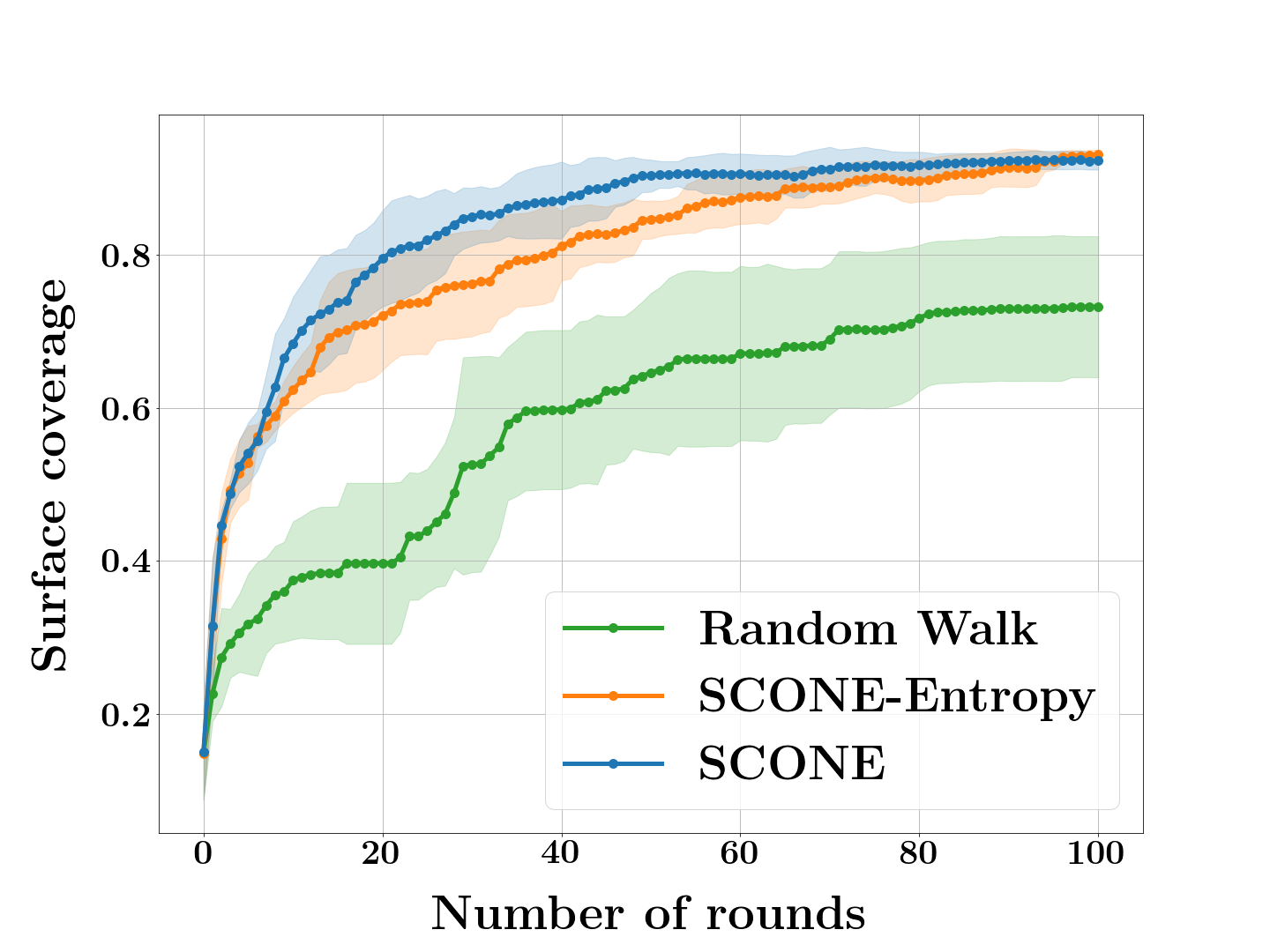} }}\\%
    
    \subfloat[\centering Pantheon]{{\includegraphics[width=4.1cm, trim={0 0 0 0}, clip]{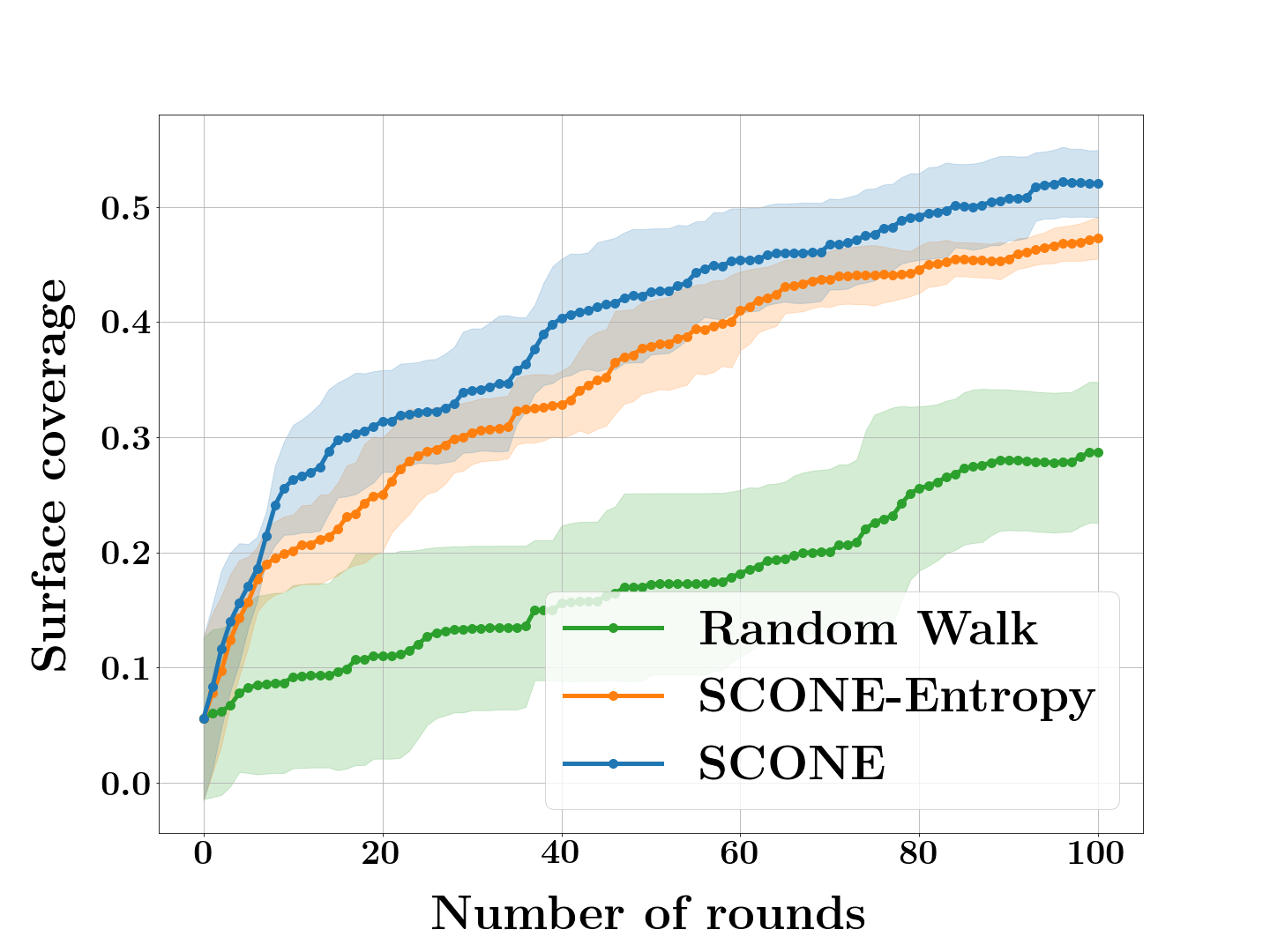} }}%
    \subfloat[\centering Bannerman Castle]{{\includegraphics[width=4.1cm, trim={0 0 0 0}, clip]{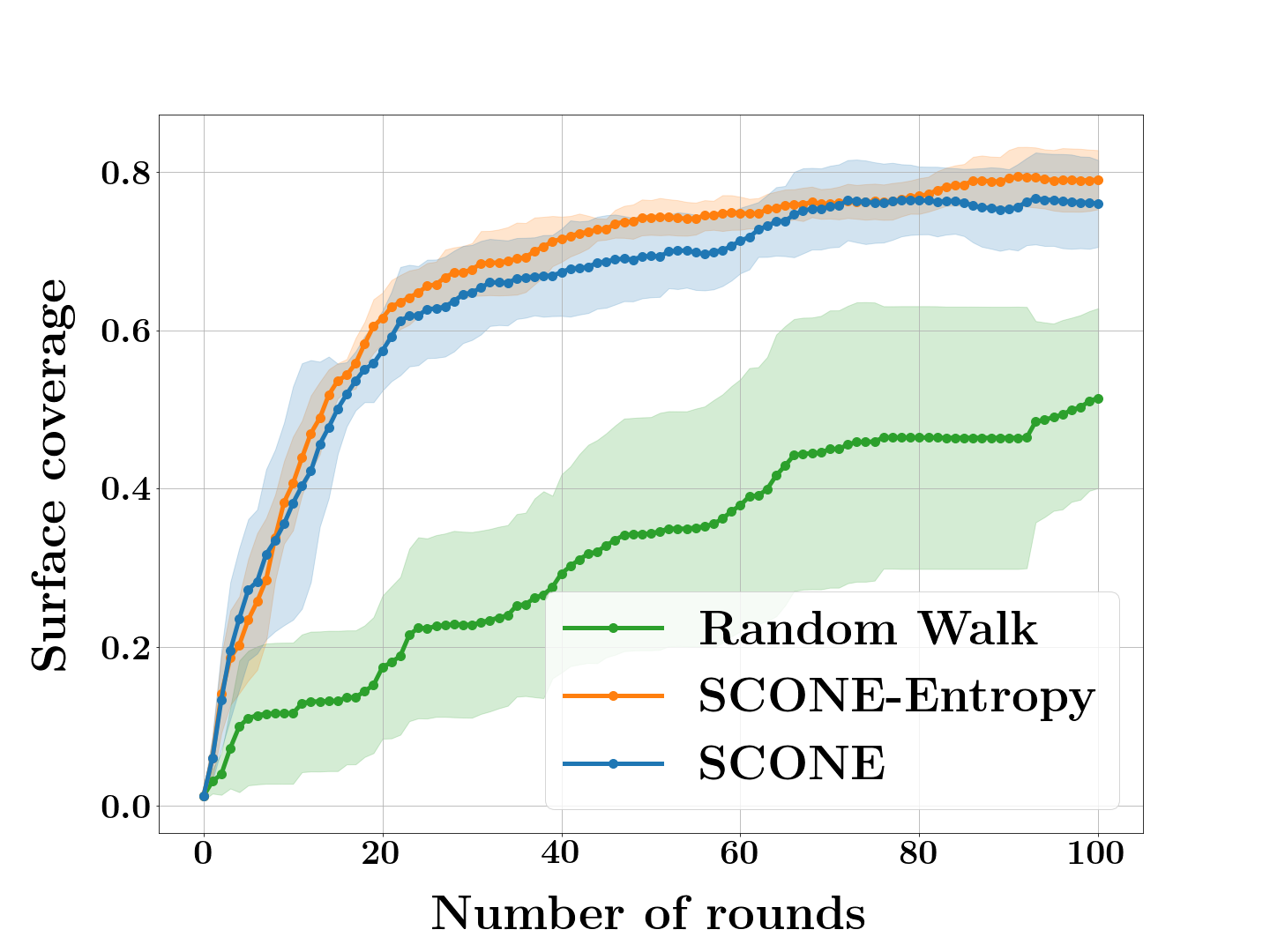} }}%
    \subfloat[\centering Christ the Redeemer]{{\includegraphics[width=4.1cm, trim={0 0 0 0}, clip]{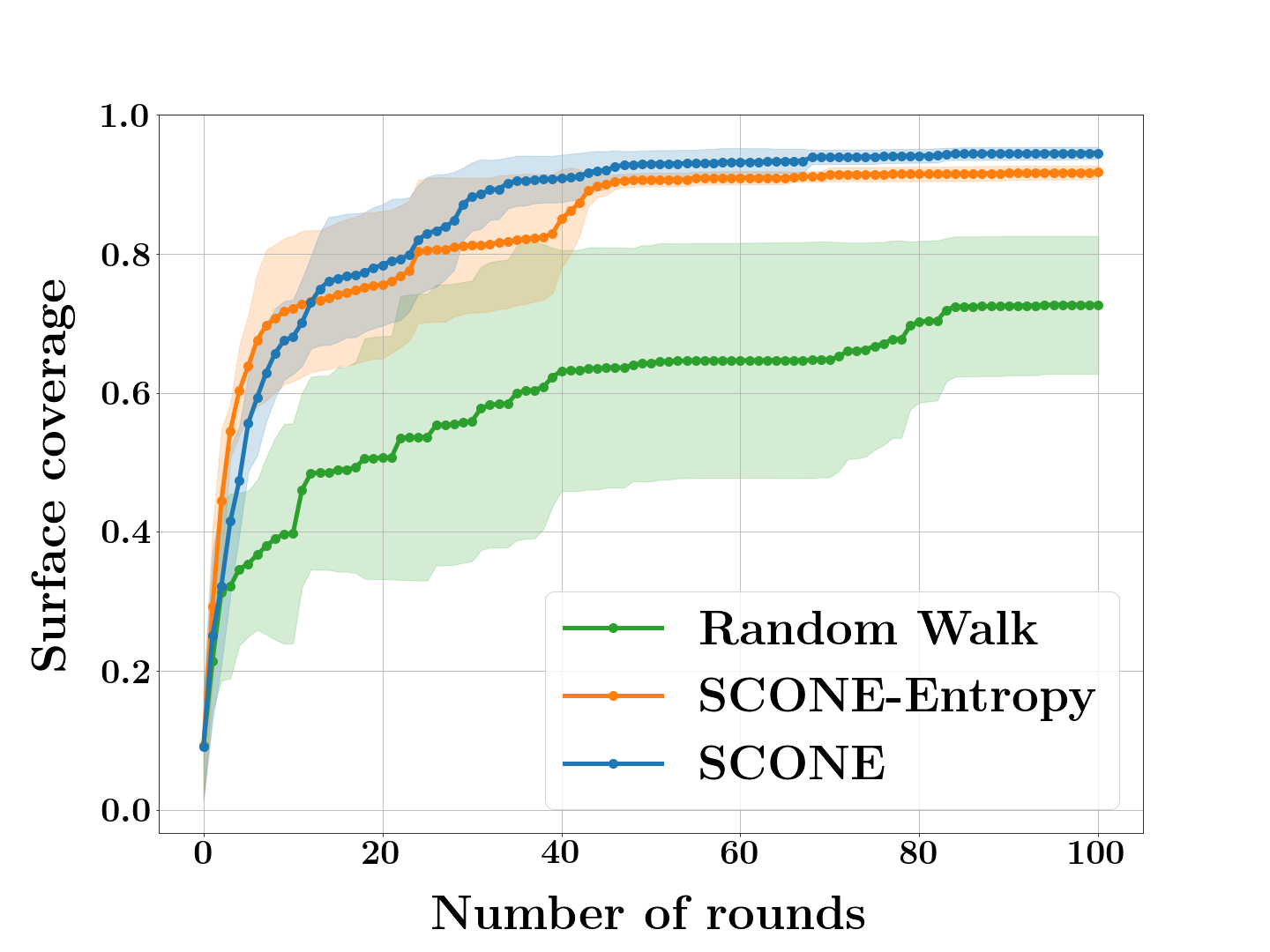} }}%

    \subfloat[\centering Statue of Liberty]{{\includegraphics[width=4.1cm, trim={0 0 0 0}, clip]{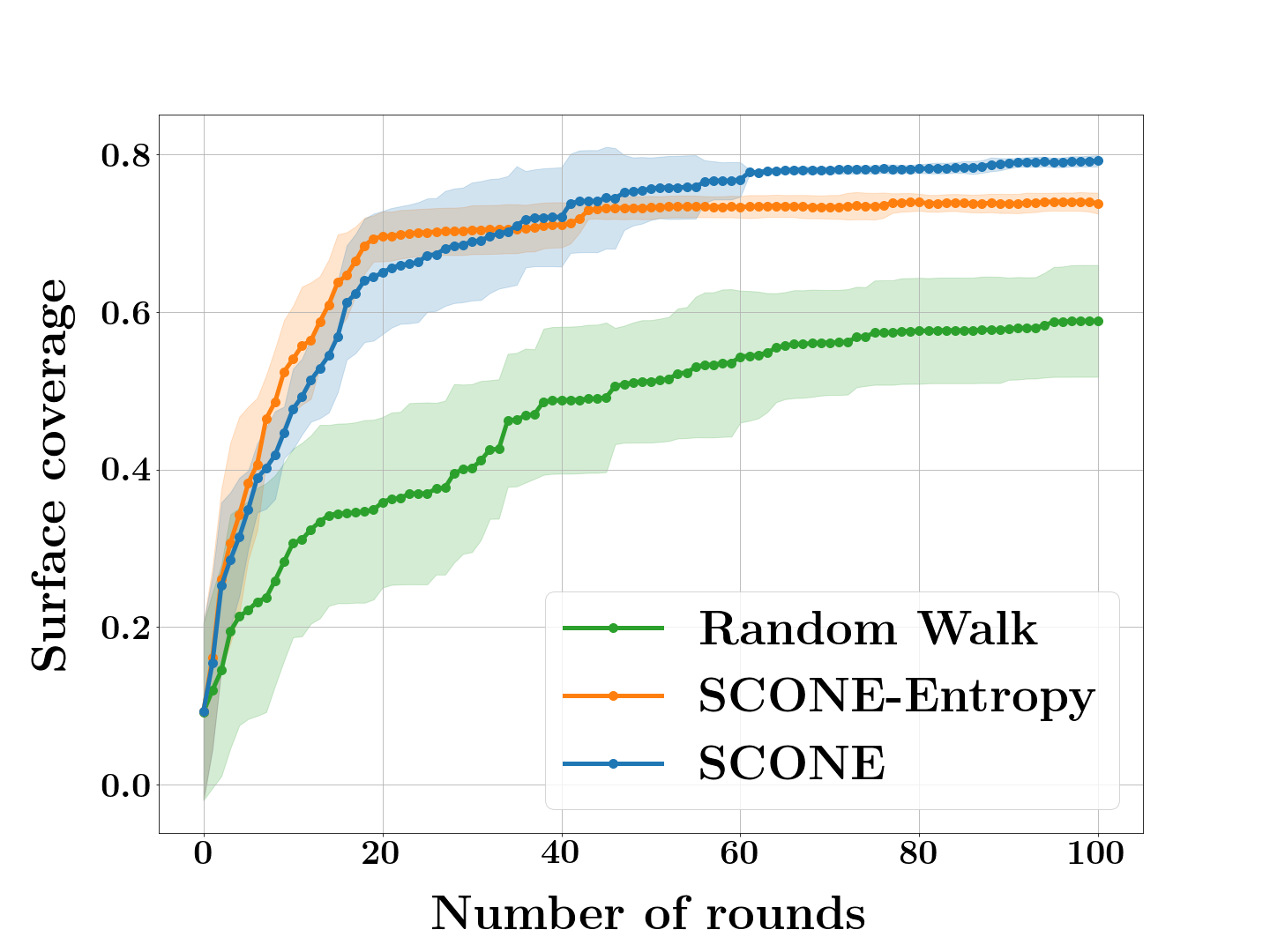} }}%
    %\subfloat[\centering Scottish National Portrait Gallery]{{\includegraphics[width=4.55cm, trim={0 0 0 0}, clip]{images/supp_mat/gallery_segmented.obj_path_planning_curve.png} }}%
    \subfloat[\centering Natural History Museum]{{\includegraphics[width=4.1cm, trim={0 0 0 0}, clip]{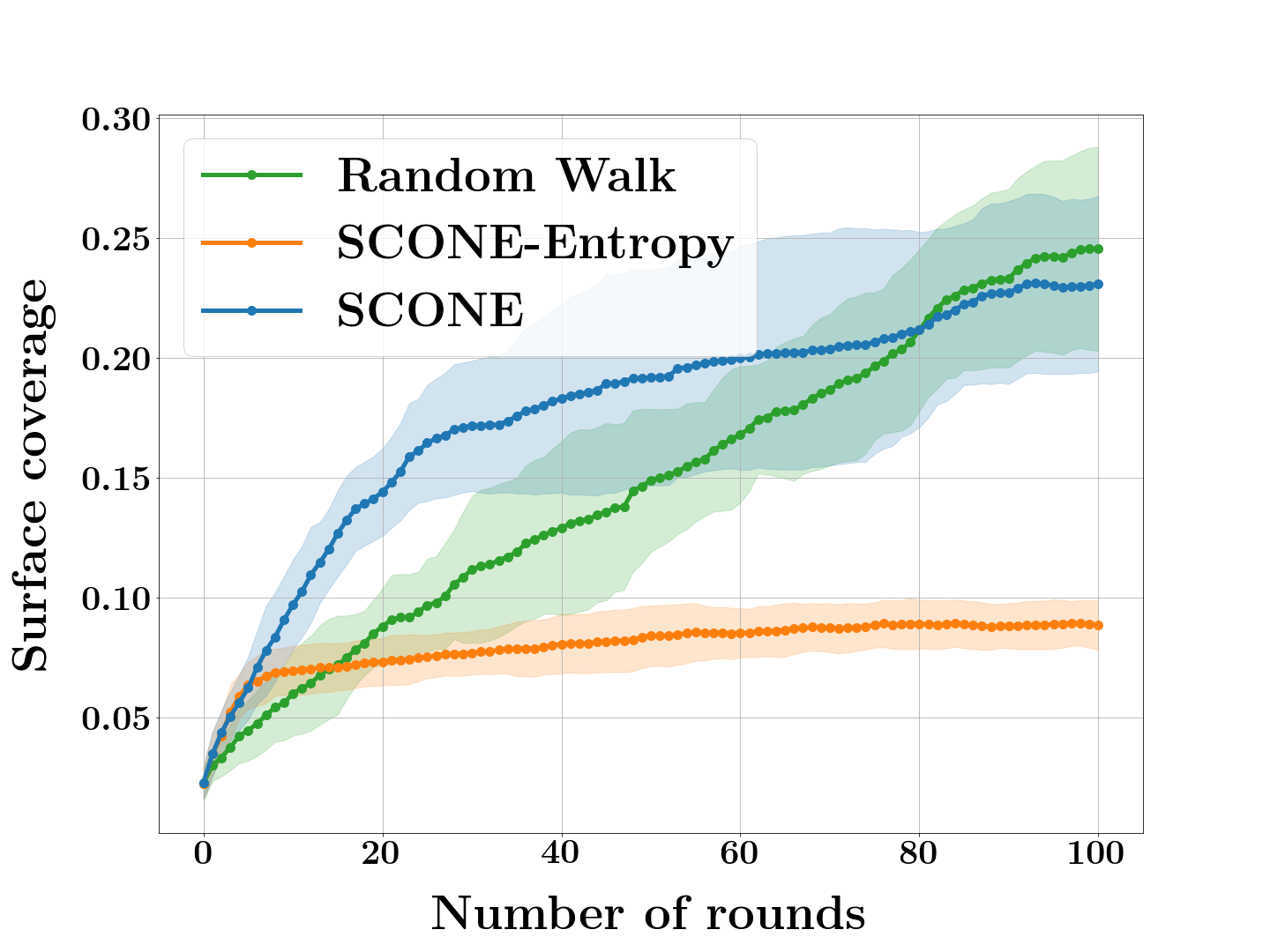} }}% To change with last scene
    
    \caption{\label{fig:scene_experiment_coverage}  {\bf Convergence speed of the covered surface in large 3D scenes by \method and our two baselines.} The first image shows the average on all scenes. For each scene, surface coverage is averaged on several trajectories starting from different camera poses. Standard deviations are shown on the figures. Despite being trained only on centered ShapeNet 3D models, the second module of \method is able to generalize to complex scenes and consistently reaches better coverage than the baselines.}
\end{figure}

\begin{table}
  \caption{\label{tab:scene_experiment_auc}  {\bf AUCs of surface coverage in large 3D scenes by \method and our two baselines} after averaging over multiple trajectories, with standard deviations. Despite being trained only on centered ShapeNet 3D models, the second module of \method is able to generalize to complex scenes and consistently reaches better AUC than the baselines.}
  \centering
  \scalebox{0.88}{ % previously at 0.96...
  \begin{tabular}{@{}lccc@{}}
    \toprule
    \multicolumn{1}{c}{} & \multicolumn{3}{c}{Method} \\
    \cmidrule(r){2-4}
     3D scene & Random Walk & \method-Entropy & \method\\
    \midrule
     Dunnottar Castle & 0.355 $\pm$ 0.106 & 0.456 $\pm$ 0.041 & \textbf{0.739} $\pm$ 0.050 \\
     Manhattan Bridge & 0.405 $\pm$ 0.089 & 0.361 $\pm$ 0.065 & \textbf{0.685} $\pm$ 0.034 \\
     Alhambra Palace & 0.384 $\pm$ 0.086 & 0.437 $\pm$ 0.047 & \textbf{0.567} $\pm$ 0.031 \\
     Leaning Tower & 0.286 $\pm$ 0.122 & 0.415 $\pm$ 0.023 & \textbf{0.542} $\pm$ 0.026 \\
     Neuschwanstein Castle & 0.403 $\pm$ 0.032 & 0.538 $\pm$ 0.040 & \textbf{0.653} $\pm$ 0.025 \\
     Colosseum & 0.308 $\pm$ 0.061 & 0.512 $\pm$ 0.024 & \textbf{0.571} $\pm$ 0.024 \\
     Eiffel Tower & 0.495 $\pm$ 0.062 & 0.741 $\pm$ 0.017 & \textbf{0.762} $\pm$ 0.020 \\
     Fushimi Castle & 0.584 $\pm$ 0.078 & 0.802 $\pm$ 0.022 & \textbf{0.841} $\pm$ 0.027 \\
     Pantheon & 0.175 $\pm$ 0.065 & 0.351 $\pm$ 0.020 & \textbf{0.396} $\pm$ 0.036 \\
     Bannerman Castle & 0.321 $\pm$ 0.121 & \textbf{0.667} $\pm$ 0.023 & 0.642 $\pm$ 0.047 \\
     Christ the Redeemer & 0.600 $\pm$ 0.146 & 0.839 $\pm$ 0.038 & \textbf{0.859} $\pm$ 0.022 \\
     Statue of Liberty & 0.469 $\pm$ 0.075 & 0.681 $\pm$ 0.018 & \textbf{0.693} $\pm$ 0.032 \\
     % Scottish National Portrait Gallery & 0.147 $\pm$ 0.064 & 0.134 $\pm$ 0.010 & \textbf{0.157} $\pm$ 0.028\\
     Natural History Museum & 0.147 $\pm$ 0.024 & 0.080 $\pm$ 0.010 & \textbf{0.177} $\pm$ 0.031 \\
     \midrule
     Mean & 0.380 & 0.529 & \textbf{0.625} \\
    % \bottomrule
  \end{tabular}
  }
\end{table}

\subsection{Ablation study}
In this subsection, we provide further analysis about both prediction modules of \method.

\paragraph{Occupancy probability.}
As we explained in the main paper, the lack of neighborhood features causes a huge loss in performance. On the contrary, using the spherical mappings $h_H(x)$ of camera history $H$ as an additional feature offers a marginal increase in performance.

In this appendix, we develop our analysis and provide not only the values of the MSE at the end of training but also IoU and training losses that support our conclusions in figure \ref{fig:ablation_occupancy}. In particular, we compute a Continuous IoU which extends the definition of IoU to a non-binary occupancy probability field. More exactly, we keep notations from subsection \ref{sec:training_occupancy}, and define the continuous IoU for the $i^{\text{th}}$ mesh of the test dataset as
\begin{equation}
    \text{IoU}_{\text{continuous}} = \frac{
    \sum_{j=1}^{N_X} \probfield(P_i;\ x_j^{(i)}) \cdot \occfield(x_j^{(i)})
    }{
    \sum_{j=1}^{N_X} \probfield(P_i;\ x_j^{(i)}) + \occfield(x_j^{(i)}) - 
    \probfield(P_i;\ x_j^{(i)}) \cdot \occfield(x_j^{(i)})
    }
\end{equation}
We also compute a more conventional IoU by thresholding occupancy probability: We define the set of predicted occupied points as the set of all points with a predicted occupancy probability above 0.5. Then, we compute the IoU with the set of ground truth occupied points.

\begin{figure}%
    \centering
    \begin{tabular}{@{}cc@{}}
  \hspace{-0.5cm}\scalebox{0.8}{ 
  \begin{tabular}{@{}lccc@{}}
    \toprule
    % \multicolumn{1}{c}{} & \multicolumn{5}{c}{3D scene} \\
    % \cmidrule(r){1-2}
    Architecture & Mean Squared Error & Continuous IoU & IoU \\
    % Method & Arena & Statue & Bridge & Castle & Monument & Mean \\
    \midrule
    Base Architecture & 0.0397 & 0.777 & 0.843 \\
    No Neighborhood Feature & 0.0816 & 0.611 & 0.702 \\
    With Camera History & \textbf{0.0386} & \textbf{0.782} & \textbf{0.844} \\
    \bottomrule
  \end{tabular}
 }
 &
 \raisebox{-.5\height}{\includegraphics[width=4.5cm]{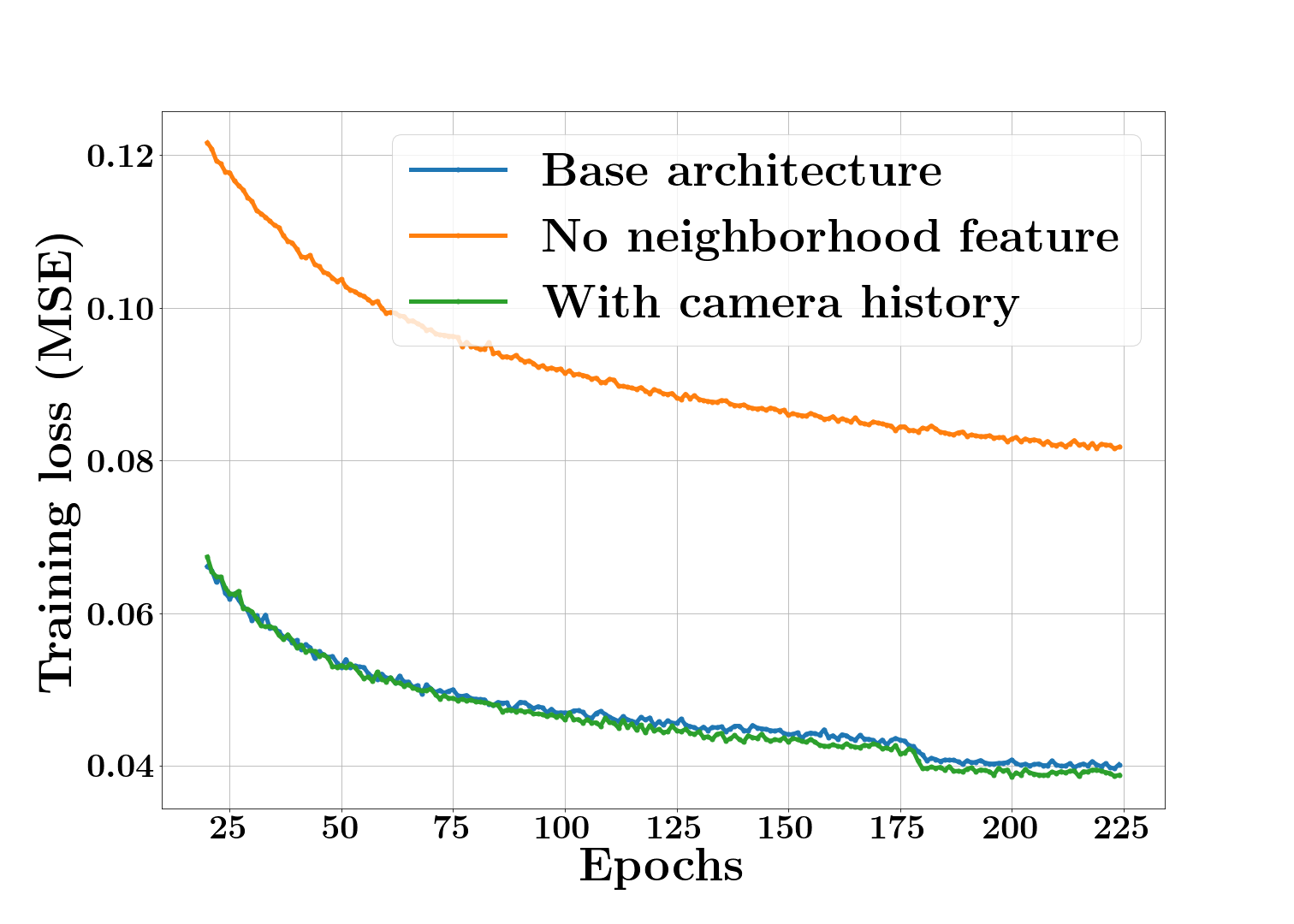}}\\[3mm]
(a) MSE and IoU after training & (b) Training loss (MSE) \\
    \end{tabular}
    \caption{\label{fig:ablation_occupancy} {\bf (a) Comparison of Mean Squared Error and IoU for variations of our occupancy probability prediction model after training. (b) Comparison of training losses (MSE) for variation of our model.} Without the multi-scale neighborhood features, the occupancy probability prediction module of \method suffers from a large decrease in performance. On the contrary, using camera history as an input feature only offers a marginal increase in performance.}%
\end{figure}

\paragraph{Visibility gain.}
We provide in figure \ref{fig:ablation_visibility} additional results supporting the observations we made in the main paper: the geometric prediction computed in a volumetric framework greatly increases performance for computing an accurate distribution of coverage gains for all camera poses in the scene, as the Kullback-Leibler divergence loss suggests. On the contrary, using spherical mappings $h_H$ of camera history as an additional input only offers a marginal increase in performance for computing the distribution of coverage gains. However, camera history features drastically improve the identification of the maximum of the distribution of coverage gains (\ie, the selection of a single NBV), as shown by the evolution of surface coverage during reconstruction.

\begin{figure}%
    \centering
    \subfloat[\centering Training loss ]{{\includegraphics[width=4.55cm, trim={0 0 3cm 0}, clip]{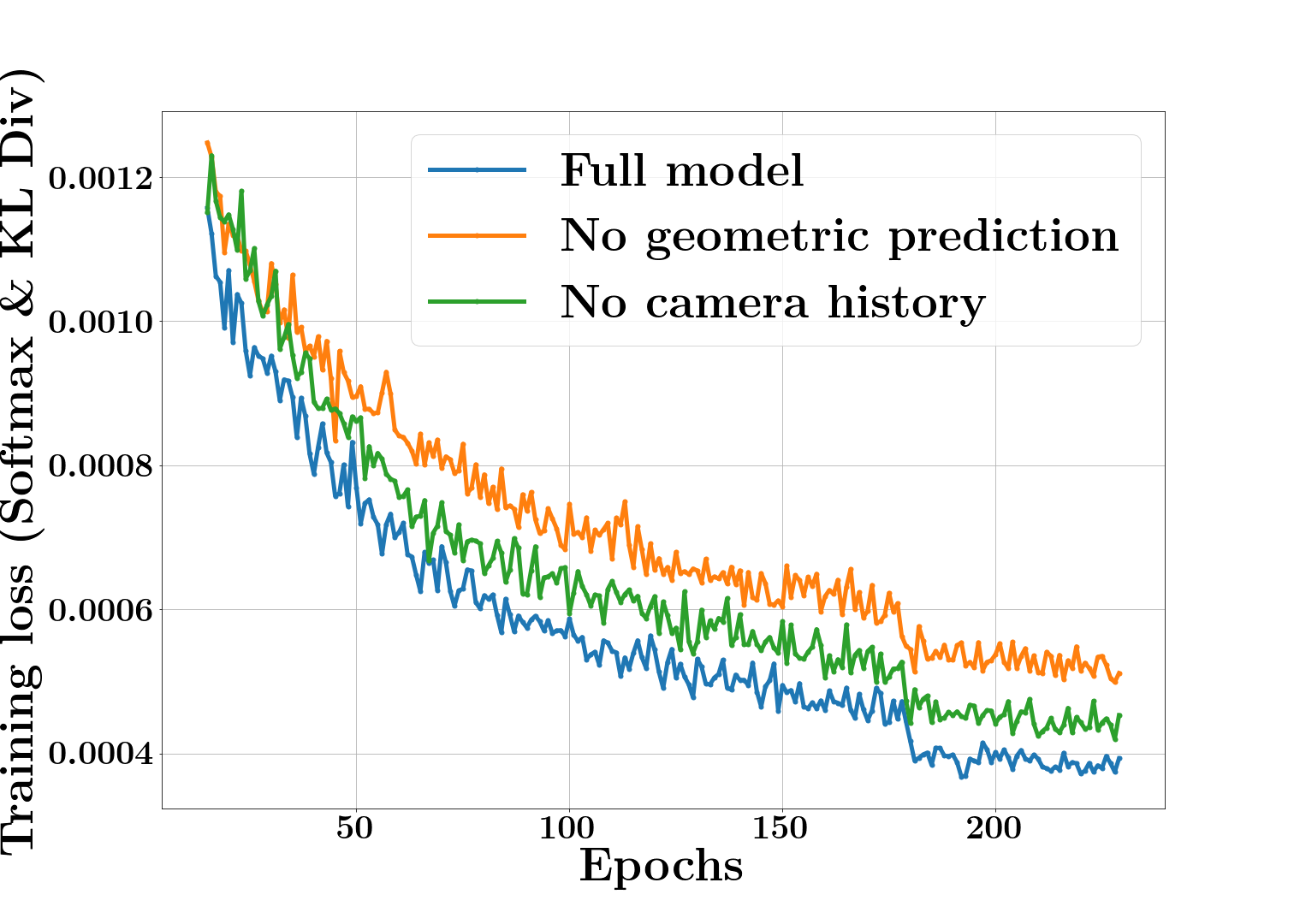} }}%
    \subfloat[\centering Validation loss ]{{\includegraphics[width=4.55cm, trim={0 0 3cm 0}, clip]{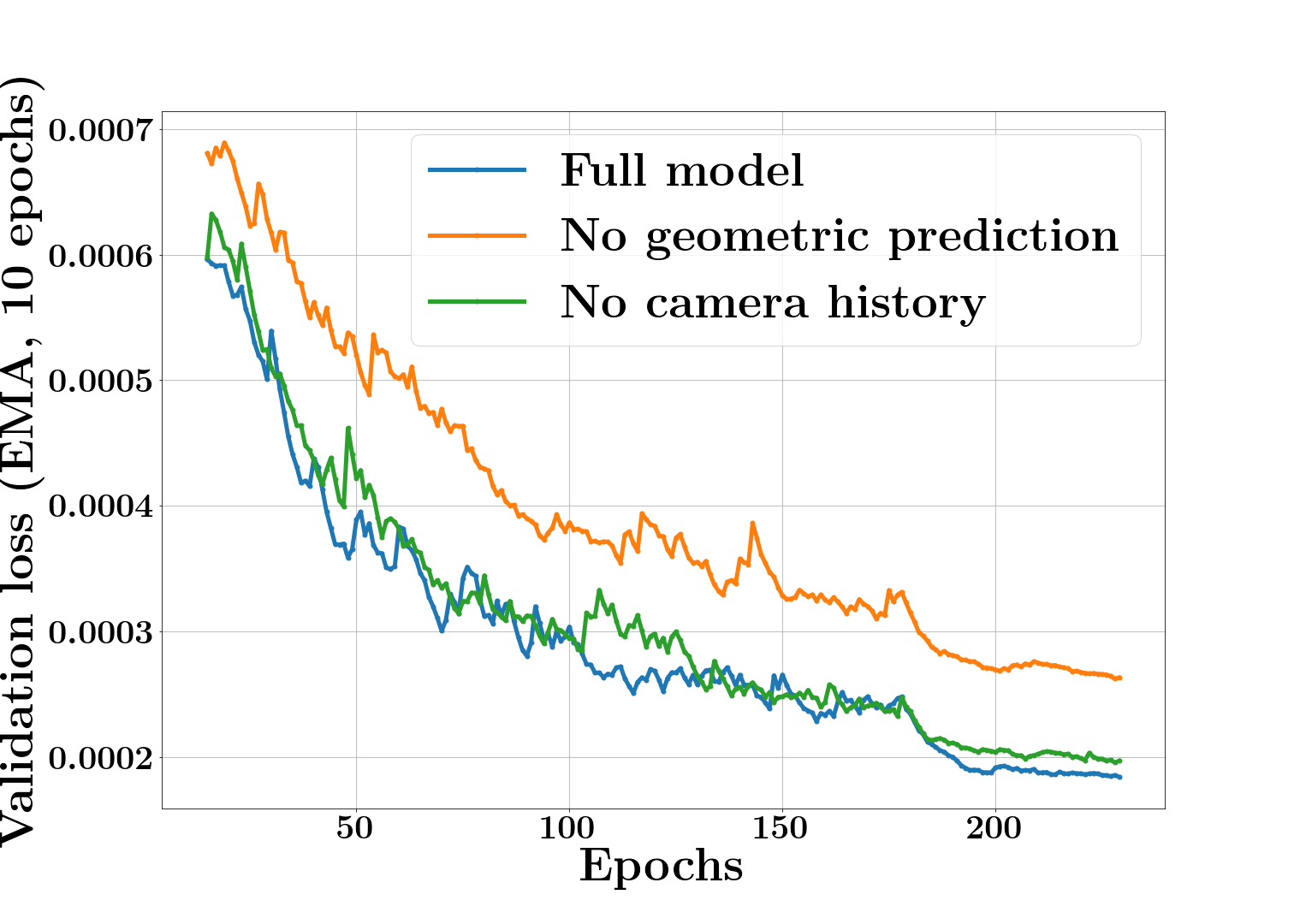} }}%
    \subfloat[\centering Surface coverage ]{{\includegraphics[width=4.55cm, trim={0 0 3cm 0}, clip]{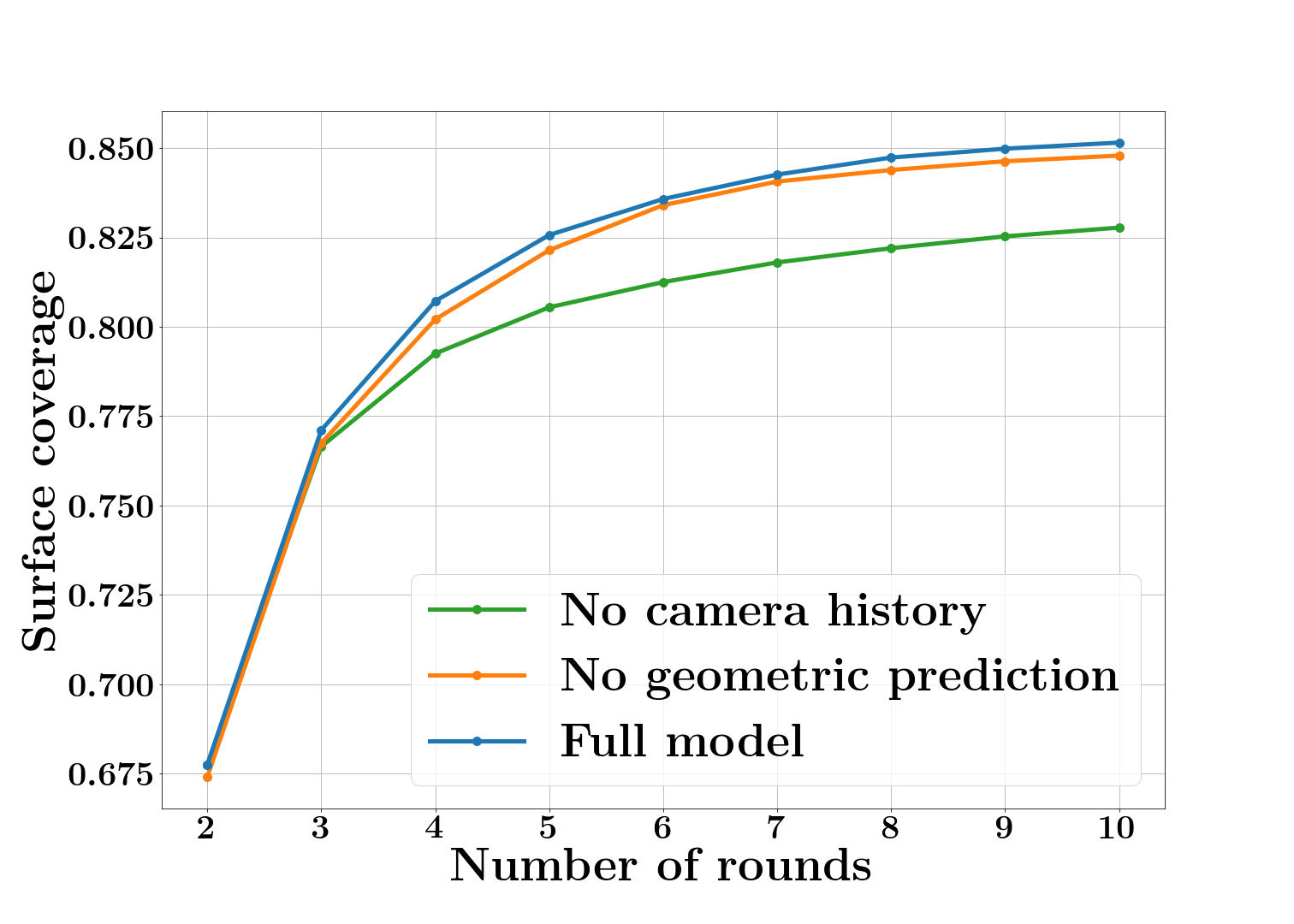} }}%
    \caption{\label{fig:ablation_visibility} {\bf Comparison of (a) training losses, (b) validation losses, and (c) surface coverage for variations of our visibility gain prediction model. The surface coverage is computed during a reconstruction process that follows the protocol presented in section \ref{sec:supp_mat_single_object}}. The validation loss is plotted with exponentially weighted moving average over 10 epochs. For surface coverage, the first round of reconstruction is not plotted since all curves start from the same point. Thanks to its volumetric approach, the full model predicts a better distribution of coverage gains on the whole space as it is indicated by the KL Divergence training loss, which is convenient for full path planning and trajectory computation in a 3D scene. Moreover, the full model does not suffer from a loss of performance in coverage when selecting a single NBV---\ie, identifying the maximum of the coverage gain distribution---compared to the version that uses directly the dense surface points, which is generally the case when working with volumetric approaches.}
\end{figure}

\end{document}